\documentclass{article}

\usepackage[main,final]{neurips_2026}

\usepackage[utf8]{inputenc}
\usepackage[T1]{fontenc}
\usepackage{natbib}
\usepackage{hyperref}
\usepackage{url}
\usepackage{booktabs, tabularx, array}
\usepackage{amsfonts}
\usepackage{nicefrac}
\usepackage{microtype}
\usepackage[dvipsnames]{xcolor}

\usepackage{mathrsfs}
\usepackage{mathtools}
\usepackage{amsmath}
\usepackage{amsthm}
\usepackage{amssymb}
\usepackage{subcaption}
\usepackage{upgreek}
\usepackage{enumitem}
\usepackage{multirow}
\usepackage{graphicx}
\usepackage{bbm}
\usepackage{aliascnt}

\providecommand{\corollaryname}{Corollary}
\providecommand{\definitionname}{Definition}
\providecommand{\examplename}{Example}
\providecommand{\lemmaname}{Lemma}
\providecommand{\propositionname}{Proposition}
\providecommand{\remarkname}{Remark}
\providecommand{\theoremname}{Theorem}

\theoremstyle{plain}

\newaliascnt{lem}{thm}
\newtheorem{lem}[lem]{\protect\lemmaname}
\aliascntresetthe{lem}
\newaliascnt{prop}{thm}
\newtheorem{prop}[prop]{\protect\propositionname}
\aliascntresetthe{prop}
\newaliascnt{cor}{thm}
\newtheorem{cor}[cor]{\protect\corollaryname}
\aliascntresetthe{cor}

\theoremstyle{definition}
\newaliascnt{definition}{thm}
\newtheorem{definition}[definition]{\protect\definitionname}
\aliascntresetthe{definition}
\newaliascnt{example}{thm}
\newtheorem{example}[example]{\protect\examplename}
\aliascntresetthe{example}

\theoremstyle{remark}
\newaliascnt{rem}{thm}

\aliascntresetthe{rem}

\usepackage[capitalize,noabbrev]{cleveref}
\crefname{cor}{Corollary}{Corollaries}
\Crefname{cor}{Corollary}{Corollaries}
\crefname{definition}{Definition}{Definitions}
\Crefname{definition}{Definition}{Definitions}
\crefname{lem}{Lemma}{Lemmas}
\Crefname{lem}{Lemma}{Lemmas}
\crefname{prop}{Proposition}{Propositions}
\Crefname{prop}{Proposition}{Propositions}
\crefname{rem}{Remark}{Remarks}
\Crefname{rem}{Remark}{Remarks}
\crefname{thm}{Theorem}{Theorems}
\Crefname{thm}{Theorem}{Theorems}

\DeclareMathOperator*{\argmin}{arg\,min}

\newcolumntype{Y}{>{\raggedright\arraybackslash}X}

\providecommand{\answerYes}[1][]{\textcolor{blue}{[Yes]#1}}
\providecommand{\answerNo}[1][]{\textcolor{orange}{[No]#1}}
\providecommand{\answerNA}[1][]{\textcolor{gray}{[N/A]#1}}

\hypersetup{
  hidelinks,
  pdftitle={Expected Batch Optimal Transport Plans and Consequences for Flow Matching},
  pdfauthor={Samuel Boïté, Julie Delon, Kimia Nadjahi},
  pdfsubject={NeurIPS 2026}
}

\title{Expected Batch Optimal Transport Plans and Consequences for Flow Matching}

\author{Samuel Boïté\thanks{Corresponding author: \href{mailto:samuel.boite@ens.fr}{\texttt{samuel.boite@ens.fr}}.} \\
  DMA, ENS Paris
  \And
  Julie Delon \\
  DMA, ENS Paris \\
  Université Paris Cité
  \And
  Kimia Nadjahi \\
  DI, ENS Paris, CNRS
}

\begin{document}

\maketitle

\begin{abstract}
Solving optimal transport (OT) on random minibatches is a common surrogate for exact OT in large-scale
learning. In flow matching (FM), this surrogate is used to obtain OT-like couplings that can straighten
probability paths and reduce numerical integration cost. Yet, the population-level coupling induced by repeated minibatch OT remains
only partially understood. We formalize this coupling as the expected batch OT
plan $\overline{\pi}_{k}$, obtained by averaging empirical OT plans over
independent minibatches of size $k$. We then establish its large-batch consistency and, in the semidiscrete
case relevant to generative modeling, derive rates for both the transport-cost bias and the convergence of
$\overline{\pi}_{k}$ to the OT plan. For FM, this yields a population coupling whose induced velocity
field is regular enough to define a unique flow from the source to the discrete target. We finally quantify how OT batch size
interacts with numerical integration in a tractable two-atom model and in synthetic and image experiments.
\end{abstract}

\section{Introduction}

Optimal transport (OT) provides a geometrically meaningful way to compare probability measures and
has become a useful principle for designing machine learning methods, with successful applications
in generative modeling, among others. In large-scale settings, exactly solving the OT problem is too
costly, both computationally and statistically. This has motivated several tractable OT-based
alternatives~\citep{peyre}. In particular, minibatch OT was introduced by \citet{fatras20a} to avoid
solving a single large OT problem between two empirical measures: instead, one subsamples smaller
batches from the two discrete measures, solves the corresponding OT problems, and averages the
resulting solutions. This strategy was then applied to generative adversarial network training, color
transfer, and domain adaptation~\citep{fatras21a,fatras}. 

More recently, minibatch OT has attracted renewed attention through flow matching (FM), a class of
generative models that learns a velocity field transporting a simple source distribution $\mu$ to a
target data distribution $\nu$~\citep{liu2022flow,lipman,albergo2024stochastic}. The FM objective
depends on a coupling between source and target, that is, a joint distribution with marginals $\mu$
and $\nu$, which determines both the probability path and the associated velocity field. Different
choices of coupling therefore lead to different tradeoffs between training and inference. The
independent coupling $\mu \otimes \nu$ is simple to implement but may lead to costly inference,
whereas the OT plan between $\mu$ and $\nu$ would yield straight trajectories and allow efficient
inference, but is intractable. Minibatch OT offers a practical compromise: it can produce straighter
flows than the independent coupling, while being more affordable than exact OT~\citep{pooladian,tong}. 

Many theoretical questions on minibatch OT remain open. In particular, the bias of minibatch OT relative to exact OT has not been precisely quantified,
either for the transport cost or plan. Then, regarding its use in FM, it is not clear
whether the velocity field induced by minibatch OT defines a well-posed flow, since a coupling is
not necessarily rectifiable~\citep[Proposition~13(iii)]{hertrich}. Finally, recent benchmarks suggest that the benefits
of minibatch OT-FM may only appear at much larger batch sizes~\citep{zhang}, while
semidiscrete OT can also improve FM performance~\citep{mousavi}. Since these
strategies make training more expensive, the practical tradeoff remains unclear: should one favor
larger transport batches during training, or a larger inference budget at sampling time?

\paragraph{Contributions.}
\begin{enumerate}[leftmargin=*]
\item In \cref{sec:expected-batch-ot-plan}, we study the expected batch optimal transport plan
      $\overline{\pi}_k$, obtained by averaging empirical OT plans over random batches of size $k$,
      complementing the results of~\citet{pooladian}.
      This plan converges to $\pi^{\star}$ when the OT plan is unique.
      In the semidiscrete setting relevant to FM, we further quantify this convergence: the cost bias $\mathbb{E}[W_{2}^{2}(\widehat{\mu}_{k},\,\widehat{\nu}_{k})] - W_{2}^{2}(\mu,\,\nu)$ decays as
      $O(k^{-1})$, yielding the plan-convergence rate
      $W_{2}^{2}(\overline{\pi}_{k},\,\pi^{\star})=O(k^{-1/2})$.
\item In \cref{sec:flow-matching-expected-batch-ot}, we show that when used in FM,
      the expected batch OT plan induces a well-posed flow: we prove that the associated velocity
      field is locally Lipschitz, generates a unique flow on $[0,1)$, and admits a terminal map
      transporting the Gaussian source $\mu$ to the discrete target $\nu$. To our knowledge, this
      is the first rectifiability result for a minibatch OT coupling in FM, and it
      provides a rigorous interpretation of this strategy beyond its use as a training heuristic.
      We give convergence rates for the terminal map's transport cost and cells.
\item We then quantify the tradeoff between OT batch size $k$ and inference budget (number of function
      evaluations, NFE) through theoretical results and a benchmark covering a wider range of batch
      sizes and NFE than in related work. In a tractable model where $\nu$ is supported on two atoms,
      we prove that increasing the NFE reduces the integration error much faster than increasing~$k$~(\cref{sec:numerical_integration}). This reveals an asymmetry between these two
      computational resources, which we also observe in synthetic Gaussian-to-discrete experiments
      and image benchmarks: increasing~$k$ improves performance
      mainly in the low-NFE regime, while the benefit vanishes at larger NFE. Altogether, these
      results clarify when larger OT batches can effectively trade training-time computation for
      inference-time savings.\footnote{Code is available at \url{https://github.com/samuelbx/expected-batch-ot}.}
\end{enumerate}

\section{Background on optimal transport and flow matching}

We recall the background on OT and FM used in our work. We refer to \citet{peyre} for a comprehensive introduction
to OT and to \citet{pierret} for a recent survey on FM.

\subsection{Reminders on optimal transport}

OT is a mathematical theory that studies how to transport one probability distribution to another
optimally, according to a given cost function. Let
$\mu,\,\nu\in\mathcal{P}_{2}(\mathbb{R}^{d})$ be probability measures with finite second moments.
Denote by $\Pi(\mu,\,\nu)$ the set of couplings between $\mu$ and
$\nu$, that is, the set of $\pi\in\mathcal{P}(\mathbb{R}^{d}\times\mathbb{R}^{d})$ whose marginals are $\mu$ and $\nu$. For the quadratic cost, the OT problem is
\begin{equation}
W_{2}^{2}(\mu,\,\nu)
=\inf_{\pi\in\Pi(\mu,\,\nu)}
\int_{\mathbb{R}^{d}\times\mathbb{R}^{d}}
\lVert x-y\rVert^{2}\,\mathrm{d}\pi(x,\,y).
\label{eq:quadratic-ot-problem}
\end{equation}
Any minimizer $\pi^{\star}$ of \eqref{eq:quadratic-ot-problem} is called an OT plan. When $\mu$ is absolutely continuous,
the OT plan is unique and takes the form $\pi^{\star}=(\mathrm{id},\,T^{\star})_\sharp\mu$, where $\sharp$
denotes the push-forward operator ($T_\sharp\mu$ is the distribution of $T(X)$ where $X\sim\mu$), and
where $T^{\star}:\mathbb{R}^{d}\to\mathbb{R}^{d}$ is called the OT map.
When $\mu = \frac{1}{k}\sum_{i=1}^{k}\updelta_{x_{i}}$ and $\nu = \frac{1}{k}\sum_{i=1}^{k}\updelta_{y_{i}}$ are
discrete measures supported on $k$ points of $\mathbb{R}^{d}$, \eqref{eq:quadratic-ot-problem} reduces to finding an
optimal permutation $\sigma^{\star}\in\mathfrak{S}_{k}$ minimizing the total cost
\[
W_{2}^{2}(\mu,\,\nu)=\min_{\sigma\in\mathfrak{S}_{k}}\frac{1}{k}\sum_{i=1}^{k}\lVert x_{i}-y_{\sigma(i)}\rVert^{2}.
\]

\subsection{Reminders on flow matching}

As in OT, the goal of FM is to transport one probability distribution to another. Let $\pi$
be a coupling between $\mu$ and $\nu$ and draw $(X_{0},\,X_{1})\sim\pi$. For $t\in[0,\,1]$, consider the
linear interpolation $X_{t}=(1-t)X_{0}+tX_{1}$ and denote its law by $\rho_{t}^{\pi}$. Then
$(\rho_{t}^{\pi})_{t\in[0,\,1]}$ is a probability path connecting $\rho^{\pi}_{0}=\mu$ and $\rho^{\pi}_{1}=\nu$.
We define the associated target velocity field by
\(
u_{t}^{\pi}(x)=\mathbb{E}[X_{1}-X_{0}\mid X_{t}=x].
\)
Under regularity assumptions, the pair $(\rho^{\pi},\,u^{\pi})$ satisfies the continuity
equation, and samples from $\rho_{t}^{\pi}$ can be generated by drawing $x(0)\sim\mu$ and
integrating the ODE 
\(
\dot{x}(s)=u_{s}^{\pi}(x(s)),
\)
on $s\in[0,\,t]$.

In practice, $u_{t}^{\pi}$ is generally not available in closed form, so one usually trains a
time-dependent model $v_{t}^{\theta}$ by regressing onto the conditional target $X_{1}-X_{0}$ under
$(X_{0},\,X_{1})\sim\pi$ and $t\sim\operatorname{Unif}[0,\,1]$. Once the velocity has been trained, the
ODE is integrated with a discrete scheme such as Euler or Runge--Kutta. The number
of velocity-function evaluations used by the scheme is called the NFE.

The choice of the coupling $\pi$ is central in FM, as it determines the entire probability path
$(\rho_{t}^{\pi})_{t\in[0,\,1]}$ and corresponding target velocity field $u_{t}^{\pi}$.
Choosing the OT plan $\pi^{\star}$, though generally intractable, would in principle yield straight paths and allow
transport in a single integration step. Indeed, when $\mu$ is absolutely continuous and $T^{\star}$
is the OT map between $\mu$ and $\nu$, the initial velocity is given by
$u_{0}^{\pi^{\star}}=T^{\star}-\mathrm{id}$, so the OT map is recovered exactly by one Euler step
from $t=0$ to $t=1$.

\section{Expected batch optimal transport plan: consistency and rates}
\label{sec:expected-batch-ot-plan}

We formalize the expected batch OT plan and study its convergence to the OT plan,
assuming it is unique.
Given $\mu,\,\nu\in\mathcal{P}_{2}(\mathbb{R}^{d})$, minibatch OT consists in solving
OT on random batches of size $k$ drawn from $\mu$ and $\nu$. This is especially useful for
applications where solving the exact OT problem is intractable,
and where we can use the resulting matched pairs during training.
For $k\in\mathbb{N}^{*}$, let $\mathbf{x}_{k}=(x_{1},\ldots,x_{k})$ and $\mathbf{y}_{k}=(y_{1},\ldots,y_{k})$ be the source and target batches, respectively, with $x_i,\,y_i \in \mathbb{R}^d$. We denote the induced
empirical measures by $\widehat{\mu}_{\mathbf{x}_{k}}=\frac{1}{k}\sum_{i=1}^{k}\updelta_{x_{i}}$
and $\widehat{\nu}_{\mathbf{y}_{k}}=\frac{1}{k}\sum_{j=1}^{k}\updelta_{y_{j}}$.
Given random batches $(\mathbf{X}_{k},\,\mathbf{Y}_{k})\sim\mu^{\otimes k}\otimes\nu^{\otimes k}$, we write $\widehat{\mu}_{k}=\widehat{\mu}_{\mathbf{X}_{k}}$ and $\widehat{\nu}_{k}=\widehat{\nu}_{\mathbf{Y}_{k}}$.

\subsection{Definition of the expected batch optimal transport plan}\label{sec:expected-batch-ot-definition}

Minibatch OT induces a coupling $\overline{\pi}_{k}$ between $\mu$ and $\nu$ by averaging random empirical OT solutions.
To define this coupling unambiguously, we assume that for each $k\in\mathbb{N}^{*}$, we have access to a
measurable solver whose output $(\mathbf{x}_{k},\,\mathbf{y}_{k})\mapsto\widehat{\pi}_{\mathbf{x}_{k},\,\mathbf{y}_{k}}$
is supported on an optimal permutation between $\widehat{\mu}_{\mathbf{x}_{k}}$ and $\widehat{\nu}_{\mathbf{y}_{k}}$.
For random batches, we write \(\widehat{\pi}_{k}=\widehat{\pi}_{\mathbf{X}_{k},\,\mathbf{Y}_{k}}\).
The expected batch OT plan is then obtained by averaging these solver outputs over random batches.

\begin{definition}[Expected batch OT plan]
\label{def:expected-batch-ot-plan}
Fix the measurable solver introduced above.
For $k\in\mathbb{N}^{*}$, the expected batch optimal transport plan $\overline{\pi}_{k}$
is the unique probability measure on $\mathbb{R}^{d}\times\mathbb{R}^{d}$ such that,
for every bounded Borel function $g:\mathbb{R}^{d}\times\mathbb{R}^{d}\to\mathbb{R}$,
\begin{equation*}
\begin{aligned}
\int_{\mathbb{R}^{d}\times\mathbb{R}^{d}}g(x,\,y)\,
\mathrm{d}\overline{\pi}_{k}(x,\,y)
&=\int_{(\mathbb{R}^{d})^{k}\times(\mathbb{R}^{d})^{k}}
\left(\int_{\mathbb{R}^{d}\times\mathbb{R}^{d}}g\,
\mathrm{d}\widehat{\pi}_{\mathbf{x}_{k},\,\mathbf{y}_{k}}\right)
\,\mathrm{d}(\mu^{\otimes k}\otimes\nu^{\otimes k})
(\mathbf{x}_{k},\,\mathbf{y}_{k}) \,.
\end{aligned}
\end{equation*}
\end{definition}

Hence, $\overline{\pi}_k$ is the law of the pair obtained as follows: first draw
$(\mathbf{X}_{k},\,\mathbf{Y}_{k})\sim\mu^{\otimes k}\otimes\nu^{\otimes k}$, then solve OT between $\widehat{\mu}_{k}$ and $\widehat{\nu}_{k}$, and finally select one of the $k$ matched pairs uniformly at random. 
This construction specializes, in the fully discrete case, to the
averaged minibatch transport matrix of \citet{fatras}. It also matches the
marginalized minibatch coupling $q^{(k)}$ used in batch OT-FM by
\citet{pooladian}.

We show in~\cref{app:expected-batch-ot-construction} that this construction indeed yields a well-defined probability measure, that $\overline{\pi}_{k}\in\Pi(\mu,\,\nu)$,
and $\overline{\pi}_{1}=\mu\otimes\nu$. When one marginal is absolutely continuous,
solver choices agree almost surely under random batches, and hence induce the same $\overline{\pi}_{k}$.

\subsection{Asymptotic consistency and rates for the cost}\label{subsec:cost-rates}

We first study the expected batch OT plan through the transport cost it induces.
Since $\overline{\pi}_k \in \Pi(\mu,\,\nu)$, this cost is greater than the optimal value $W_{2}^{2}(\mu,\,\nu)$.
To quantify the excess cost, we first show that the cost induced by $\overline{\pi}_k$ equals the expected value of $W_{2}^{2}(\widehat{\mu}_k,\,\widehat{\nu}_k)$ over random batches.

\begin{prop}\label{prop:expected-batch-ot-cost-identity}
  For any $k\in\mathbb{N}^{*}$, $\int_{\mathbb{R}^{d}\times\mathbb{R}^{d}}\lVert x-y\rVert^{2}\,\mathrm{d}\overline{\pi}_{k}(x,\,y)
= \mathbb{E}[W_{2}^{2}(\widehat{\mu}_{k},\,\widehat{\nu}_{k})]$.
\end{prop}

The identity above connects our analysis with classical questions in statistical OT, such as sample
complexity of empirical Wasserstein costs. In contrast
with most sample-complexity results, which control $\mathbb{E}[|W_{2}^{2}(\widehat{\mu}_{k},\,\widehat{\nu}_{k})-W_{2}^{2}(\mu,\,\nu)|]$, we study directly the signed bias $\mathbb{E}[W_{2}^{2}(\widehat{\mu}_{k},\,\widehat{\nu}_{k})]-W_{2}^{2}(\mu,\,\nu)$. This distinction yields a faster $O(k^{-1})$ rate in the nondegenerate, semidiscrete setting.

\begin{prop}[Asymptotic consistency and rates for expected batch OT cost]
\label{prop:expected-batch-ot-cost-convergence}
Let $\mu,\,\nu \in \mathcal{P}_2(\mathbb{R}^d)$. \\[-6mm]
\begin{enumerate}[leftmargin=*]
\item The sequence $(\mathbb{E}[W_{2}^{2}(\widehat{\mu}_{k},\,\widehat{\nu}_{k})])_{k\in\mathbb{N}^*}$
    is nonincreasing.\\ If $\mu,\,\nu\in\mathcal{P}_{q}(\mathbb{R}^{d})$ for some $q > 2$, then $\lim_{k\to\infty} \mathbb{E}[W_{2}^{2}(\widehat{\mu}_{k},\,\widehat{\nu}_{k})] = W_{2}^{2}(\mu,\,\nu)$.
\item If $\mu=\mathcal{N}(0,\,I_d)$ and $\nu$ has finite support, then $\mathbb{E}[W_{2}^{2}(\widehat{\mu}_{k},\,\widehat{\nu}_{k})]-W_{2}^{2}(\mu,\,\nu)=O(k^{-1})$.
\end{enumerate}
\end{prop}

Thanks to the semidiscrete structure, our rates avoid the exponential dependence on $d$ that appears in the fully discrete case~\citep{fournier2015}. 
One can show an $O(k^{-1/2})$ rate for $\mathbb{E}[|W_{2}^{2}(\widehat{\mu}_{k},\,\widehat{\nu}_{k})-W_{2}^{2}(\mu,\,\nu)|]$ using existing results from~\citep{hundrieser,staudt-unbounded},
but the sharper $O(k^{-1})$ rate exploits the signed expectation.
This result holds beyond Gaussian $\mu$ under the more general Assumptions A or B in \citet{genans}.

\subsection{Convergence and rates for the plan}\label{subsec:plan-rates}

We now establish the convergence of $(\overline{\pi}_{k})_{k\in\mathbb{N}^{*}}$ in the space $(\mathcal{P}(\mathbb{R}^{d}\times\mathbb{R}^{d}),\, W_{2})$.
From now on, we assume that $\mu$ is absolutely continuous with respect to Lebesgue measure, so that the OT plan $\pi^\star$ between $\mu$ and $\nu$ is unique.

Before stating our main result, we provide an illustration in \cref{fig:expected-batch-ot-plan-1d}, which suggests
that $\overline{\pi}_{k}$ approaches $\pi^{\star}$ as $k$ increases. We also
compare $\overline{\pi}_{k}$ with an entropically regularized plan. Qualitatively, the minibatch strategy induces
a form of regularization distinct from entropic regularization. A precise characterization of this
effect is outside the scope of the article and left for future work.

\begin{figure}[t]
\centering
\begin{subfigure}{0.235\textwidth}
\centering
\includegraphics[width=\linewidth]{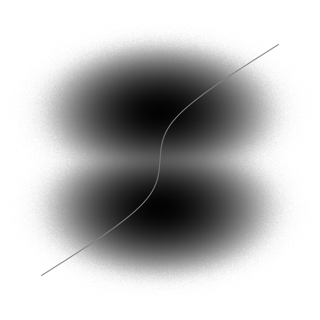}
\caption{$k=1$ (independent)}
\end{subfigure}\hfill
\begin{subfigure}{0.235\textwidth}
\centering
\includegraphics[width=\linewidth]{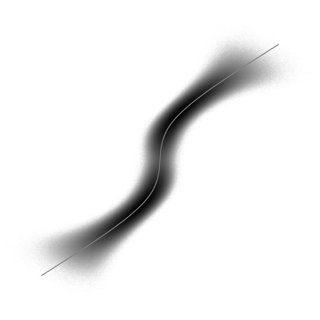}
\caption{$k=100$}
\end{subfigure}\hfill
\begin{subfigure}{0.235\textwidth} 
\centering
\includegraphics[width=\linewidth]{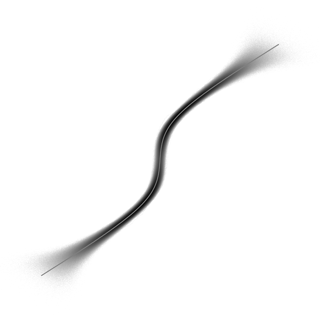}
\caption{$k=1\,000$}
\end{subfigure}\hfill
\begin{subfigure}{0.235\textwidth}
\centering
\includegraphics[width=\linewidth]{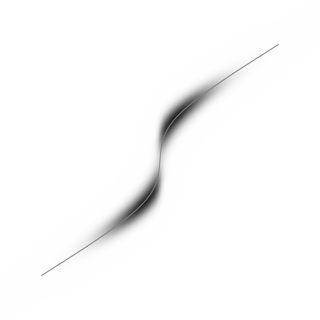}
\caption{$\varepsilon=0.15$ (entropic)}
\end{subfigure}
\caption{Transport plans between $\mathcal{N}(0,1)$ and a mixture of two symmetric Gaussians. As $k$ grows, $\overline{\pi}_{k}$ approaches the OT plan $\pi^{\star}$ (shown as a line) (a--c). (d) shows the entropically regularized plan.}
\label{fig:expected-batch-ot-plan-1d}
\end{figure}

The next proposition proves the convergence to $\pi^{\star}$ as well as quantitative bounds on
$W_{2}^{2}(\overline{\pi}_{k},\,\pi^{\star})$ in the semidiscrete setting, which is relevant for flow matching. The detailed proof is
in~\cref{app:expected-batch-ot-plan-convergence}.

\begin{prop}[Convergence and rates for the expected batch OT plan]
\label{prop:expected-batch-ot-plan-convergence}
Let $\mu,\,\nu\in \mathcal{P}_{2}(\mathbb{R}^{d})$ with $\mu$ absolutely continuous with respect to the Lebesgue measure. Then, $(\overline{\pi}_{k})_{k\in \mathbb{N}^{*}}$ converges weakly, and therefore in \(W_{2}\), to $\pi^{\star}$: for any bounded continuous function $g:\mathbb{R}^{d}\times\mathbb{R}^{d}\to\mathbb{R}$, $\lim_{k\to +\infty} \int g\,\mathrm{d} \overline{\pi}_{k} = \int g\,\mathrm{d}\pi^{\star}$.
Moreover,\\[-6mm]
\begin{enumerate}[leftmargin=*]
\item For any $k\geqslant1$, $W_{2}^{2}(\overline{\pi}_{k},\,\pi^{\star})\geqslant\frac{1}{2}\big((\mathbb{E}[W_{2}^{2}(\widehat{\mu}_{k},\,\widehat{\nu}_{k})])^{\frac{1}{2}}-W_{2}(\mu,\,\nu)\big)^{2}$.
\item If $\mu=\mathcal{N}(0,\,I_{d})$ and $\nu$ has finite support, then there exists $A_\nu>0$ such that
\[
  W_{2}^{2}(\overline{\pi}_{k},\,\pi^{\star})\leqslant A_\nu(\mathbb{E}[W_{2}^{2}(\widehat{\mu}_{k},\,\widehat{\nu}_{k})]-W_{2}^{2}(\mu,\,\nu))^{1/2}=O(k^{-1/2}).
\]
\end{enumerate}
\end{prop}

The above result shows that if the expected batch OT cost is far above $W_{2}^{2}(\mu,\,\nu)$, then $\overline{\pi}_{k}$
is necessarily far from $\pi^{\star}$ in Wasserstein distance. However, the converse fails in general.
More precisely, we show that for any $\pi \in \Pi(\mu,\,\nu)$, there is no universal modulus that controls
$W_{2}^{2}(\pi,\,\pi^{\star})$ in terms of the cost bias~(\cref{app:no-universal-modulus}). Plan convergence
thus requires additional hypotheses on $\mu,\,\nu$.
We use the \(W_2\) distance to compare plans because it is consistent with the
geometry of the underlying OT problem. Moreover, all couplings in \(\Pi(\mu,\,\nu)\) have the same second moment,
so the weak convergence upgrades to \(W_2\) convergence, and it is natural to consider the corresponding rate.

\subsection{Numerical illustration of convergence rates} \label{subsec:num_rates}

We illustrate the convergence rates of the expected batch OT cost and plan in Gaussian-to-discrete settings. In \cref{fig:expected-batch-ot-cost-plan-rates}, we consider $\mu=\mathcal{N}(0,\,I_d)$ and $\nu$ supported on $M$ distinct atoms, drawn uniformly from $[0,\,1]^d$ and then fixed, for $d\in\{10,\,1000\}$ and $M\in\{5,\,10,\,50\}$.

We evaluate the cost bias and observe a decay in $O(k^{-1})$, illustrating \cref{prop:expected-batch-ot-cost-convergence}. We also observe an $O(k^{-1/2})$ decay for a computable upper bound on $W_{2}^{2}(\overline{\pi}_{k},\,\pi^{\star})$. The finite-sample regime appears slower than the asymptotic regime. Directly estimating $W_{2}^{2}(\overline{\pi}_{k},\,\pi^{\star})$ is both computationally and statistically challenging, since it requires discretizing both $\overline{\pi}_{k}$ and $\pi^{\star}$ and solving a discrete OT problem on $\mathbb{R}^d\times\mathbb{R}^d$. Instead, we estimate $\mathbb{E}_{(X,Y)\sim\overline{\pi}_{k}}\bigl[\lVert Y-T^{\star}(X)\rVert^2\bigr]$, where $T^{\star}$ is the OT map from $\mu$ to $\nu$. This quantity upper bounds $W_{2}^{2}(\overline{\pi}_{k},\,\pi^{\star})$: if $(X,Y)\sim\overline{\pi}_{k}$, then $((X,Y),(X,T^{\star}(X)))$ is a coupling of $\overline{\pi}_{k}$ and $\pi^{\star}$.

\begin{figure}[t]
\centering
\includegraphics[width=\linewidth]{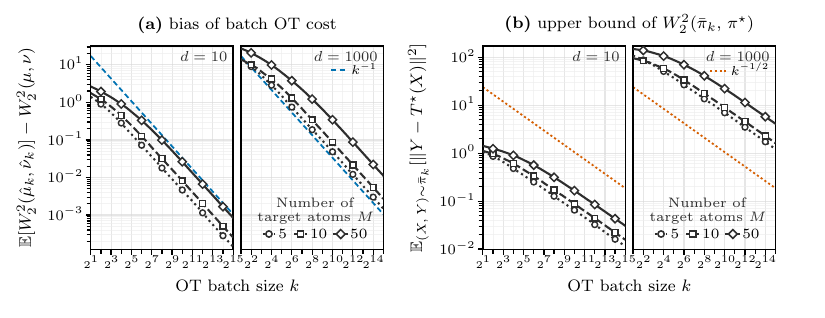}
\caption{Numerical rates in Gaussian-to-discrete experiments.
The expected batch OT cost bias decays as \(O(k^{-1})\) (left).
The quantity \(\mathbb{E}_{(X,Y)\sim\overline{\pi}_{k}}[\lVert Y-T^{\star}(X)\rVert^{2}]\), 
which upper bounds $W_{2}^{2}(\overline{\pi}_{k},\,\pi^{\star})$, decays
as \(O(k^{-1/2})\) (right). Standard errors are below
plotting resolution.}
\label{fig:expected-batch-ot-cost-plan-rates}
\end{figure}

\section{Flow matching under expected batch optimal transport}
\label{sec:flow-matching-expected-batch-ot}

In this section, we study the effect of minibatch OT on the dynamics of flow matching. We focus on the Gaussian-to-discrete setting,
where the source $\mu$ is a standard Gaussian, and the target $\nu$ is supported on pairwise distinct atoms $v_{1},\ldots,v_{M}$.
We first show that the velocity field of FM, when derived from the expected batch OT plan
$\overline{\pi}_{k}$, is regular enough to define a well-posed ODE flow on $[0,\,1)$. We then study the
extent to which increasing the batch size $k$ can serve as a substitute for finer numerical integration,
thereby complementing the empirical studies of~\citet{pooladian,tong,zhang}.
The theoretical statements in this section concern the exact target velocity induced by the coupling.
The image experiments then separately study learned neural approximations.

\subsection{Posterior representation of the velocity field}

In this Gaussian-to-discrete setting, the velocity field admits a closed-form expression, as studied
for instance by \citet[Proposition~1]{bertrand}.
We prove the following result in \cref{app:conditional-mean-posterior-formulas}.

\begin{lem}
\label{lem:gaussian-discrete-velocity-posterior}
Assume $\mu=\mathcal{N}(0,\,I_d)$ and
$\nu=\sum_{j=1}^{M}q_j\updelta_{v_j}$, with $q_j>0$ and pairwise
distinct atoms $v_j\in\mathbb{R}^d$. Let $\pi\in\Pi(\mu,\,\nu)$ and
for $t\in[0,1)$, let $X_t = (1-t)X_0 + tX_1$,
$(X_0,X_1)\sim\pi$. Then,
\begin{equation}\label{eq:gaussian-discrete-velocity-posterior-mean}
u_{t}^{\pi}(x)=\frac{m_{t}^{\pi}(x)-x}{1-t},\ 
m_{t}^{\pi}(x)=\mathbb{E}_\pi[X_{1}\mid X_{t}=x]
=\sum_{j=1}^{M}v_{j}\mathbb{P}_{(X_{0},\,X_{1})\sim\pi}(X_{1}=v_{j}\mid X_{t}=x).
\end{equation}
Moreover, if $f_{{\pi}}$ denotes the density of
${\pi}$ with respect to the product of the Lebesgue measure on $\mathbb{R}^{d}$
and the counting measure on $\{v_{1},\ldots,v_{M}\}$, then, for $\rho_t^\pi$-a.e. \(x\)
(equivalently, whenever the denominator below is positive),
\begin{equation}
  \label{eq:gaussian-discrete-posterior-probabilities}
  \mathbb{P}_{(X_0,X_1)\sim\pi}(X_1=v_j\mid X_t=x)
  =
  \left.
  f_{\pi}\!\left(\frac{x-tv_j}{1-t},v_j\right)
  \middle/
  \sum_{\ell=1}^{M}
  f_{\pi}\!\left(\frac{x-tv_\ell}{1-t},v_\ell\right)
  \right. .
\end{equation}
\end{lem}

The velocity field points from $x$ to the posterior mean $m_{t}^{\pi}(x)$, which always lies
in the convex hull of the target atoms.
Observe that if we write $\overline{a}^{\pi}_{j}(x_{0})=\mathbb{P}_{(X_{0},\,X_{1})\sim{\pi}}(X_{1}=v_{j}\mid X_{0}=x_{0})$, then $f_{{\pi}}(x_{0},\,v_{j})=\varphi(x_{0})\overline{a}^{\pi}_{j}(x_{0})$,
where $\varphi$ is the standard Gaussian density (\cref{app:expected-batch-ot-rectifiability}).
Thus the velocity field is also completely determined by the assignment probabilities $\overline{a}^{\pi}_{j}$ between $X_{0}$ and $X_{1}$. In the specific case where $\pi = \mu\otimes\nu$ and $\nu$ is uniform, the assignment probabilities in~\eqref{eq:gaussian-discrete-posterior-probabilities} form a softmax with logits $-\frac{\lVert x-tv_{j}\rVert^{2}}{2(1-t)^{2}}$, $j\in\{1,\ldots,M\}$.

\subsection{Smoothness and rectifiability of the expected batch OT plan}

When using batch OT for FM, a natural question is whether the velocity field generated
by $\overline{\pi}_{k}$ defines a well-posed flow.
The following result, proven in \cref{app:expected-batch-ot-rectifiability}, confirms that this is the case.

\begin{prop}[Rectifiability of the expected batch OT plan]
\label{prop:expected-batch-ot-rectifiability}
Fix $k\in\mathbb{N}^{*}$ and consider the Gaussian-to-discrete setting.
The expected batch OT plan $\overline{\pi}_{k}$
is rectifiable in the following sense. \\[-6mm]
\begin{enumerate}[leftmargin=*]
\item \emph{(Local Lipschitzness of $u_{t}^{\overline{\pi}_{k}}$).} For any compact $K\subseteq\mathbb{R}^d$, the map $(t,\,x)\in[0,\,1)\times K\mapsto
u_{t}^{\overline{\pi}_{k}}(x)$ is jointly continuous and for any $\tau\in(0,\,1)$, it is locally Lipschitz in $x$ uniformly in $t\in[0,\,\tau]$. 
\item \emph{(Pre-terminal flow).} For any $x_{0}\in\mathbb{R}^{d}$, the ODE
$\dot{x}(t)=u_{t}^{\overline{\pi}_{k}}(x(t))$, $x(0)=x_{0}$, admits a unique
solution on $[0,\,1)$, denoted by $t\mapsto\phi_{t}^{\overline{\pi}_{k}}(x_{0})$.
 Moreover, for any $t\in[0,\,1)$, the
map $x_{0}\mapsto\phi_{t}^{\overline{\pi}_{k}}(x_{0})$ is continuous and 
$(\phi_{t}^{\overline{\pi}_{k}})_{\sharp}\mu=\rho_{t}^{\overline{\pi}_{k}}$.
\item \emph{(Terminal flow).} There exists a Borel map
$\phi^{\overline{\pi}_{k}}:\mathbb{R}^{d}\to\{v_{1},\ldots,v_{M}\}$ such
that $\phi^{\overline{\pi}_{k}}(x_{0})=\lim_{t\uparrow1}\phi_{t}^{\overline{\pi}_{k}}(x_{0})$
for $\mu$-a.e. $x_{0}$, and $(\phi^{\overline{\pi}_{k}})_{\sharp}\mu=\nu$.
\end{enumerate}
\end{prop}

The terminal flow map $\phi^{\overline{\pi}_{k}}$ partitions $\mathbb{R}^{d}$ into regions $\Gamma_{k,\,j}=(\phi^{\overline{\pi}_{k}})^{-1}(\{v_{j}\})$, $j\in\{1,\ldots,M\}$, according to the final atom reached when starting from $x_{0}\in\mathbb{R}^{d}$.
As $k$ grows, these regions approach the Laguerre cells $C_{j}^{\star}=(T^{\star})^{-1}(\{v_{j}\})$ of the semidiscrete OT map~\citep[Section~5]{peyre} between $\mu$ and $\nu$ (see~\cref{fig:expected-batch-ot-flow-cells}).
The next result, shown in~\cref{app:terminal-flow-map-rates}, transfers the rates of~\cref{subsec:cost-rates,subsec:plan-rates}
to the terminal flow map \(\phi^{\overline\pi_k}\).
Here \(A\triangle B=(A\setminus B)\cup(B\setminus A)\) denotes the symmetric difference.

\begin{cor}[Rates for the terminal flow map]\label{cor:terminal-flow-map-rates}
In the Gaussian-to-discrete setting, \\[-6mm]
\begin{enumerate}[leftmargin=*]
\item \emph{(Transport cost)}
\(
0\leqslant\int_{\mathbb{R}^d} \lVert \phi^{\overline\pi_k}(x)-x\rVert^2\, \mathrm{d}\mu(x) - W_2^2(\mu,\,\nu)=O(k^{-1}).
\)
\item \emph{(FM cells)} \(\mu(\{x\in\mathbb{R}^d,\;\phi^{\overline\pi_k}(x)\neq T^\star(x)\})=O(k^{-1/2})\),
 therefore $\mu(\Gamma_{k,\,j}\triangle C_j^\star)=O(k^{-1/2})$.
\end{enumerate}
\end{cor}
The transport cost of $\phi^{\overline\pi_k}$ can be nonmonotonic in $k$, as highlighted in the numerical counterexample of \cref{app:four-point-nonmonotonicity}.
Moreover, the convergence of the FM cells is in \(\mu\)-measure and is not uniform, as shown in \cref{app:four-point-nonuniformity} using the four-point example by \citet{pierret-tessellations}.

\begin{figure}[t]
\centering
\begin{subfigure}{0.235\textwidth}
\centering
\includegraphics[width=\linewidth]{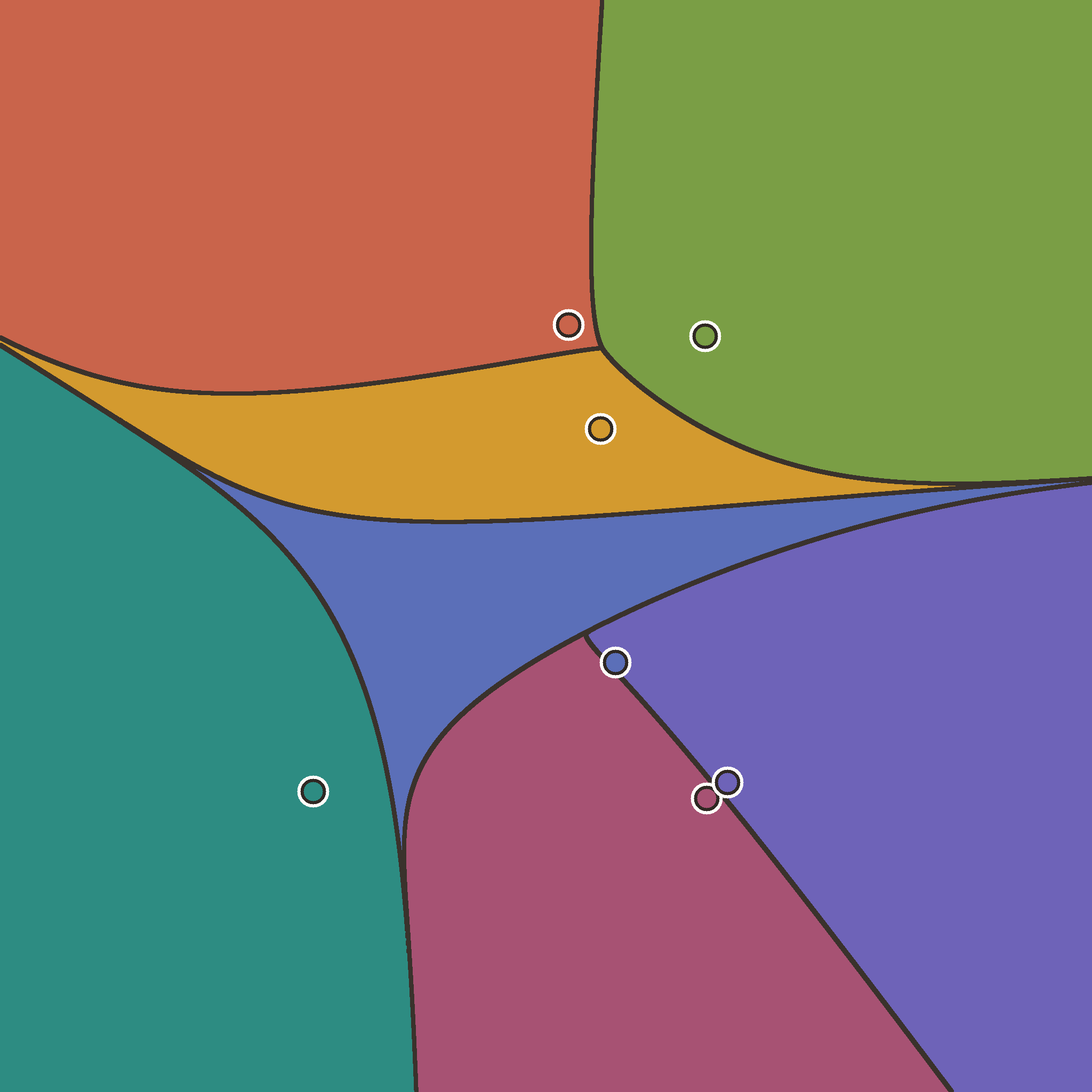}
\caption{$k=1$}
\end{subfigure}\hfill
\begin{subfigure}{0.235\textwidth}
\centering
\includegraphics[width=\linewidth]{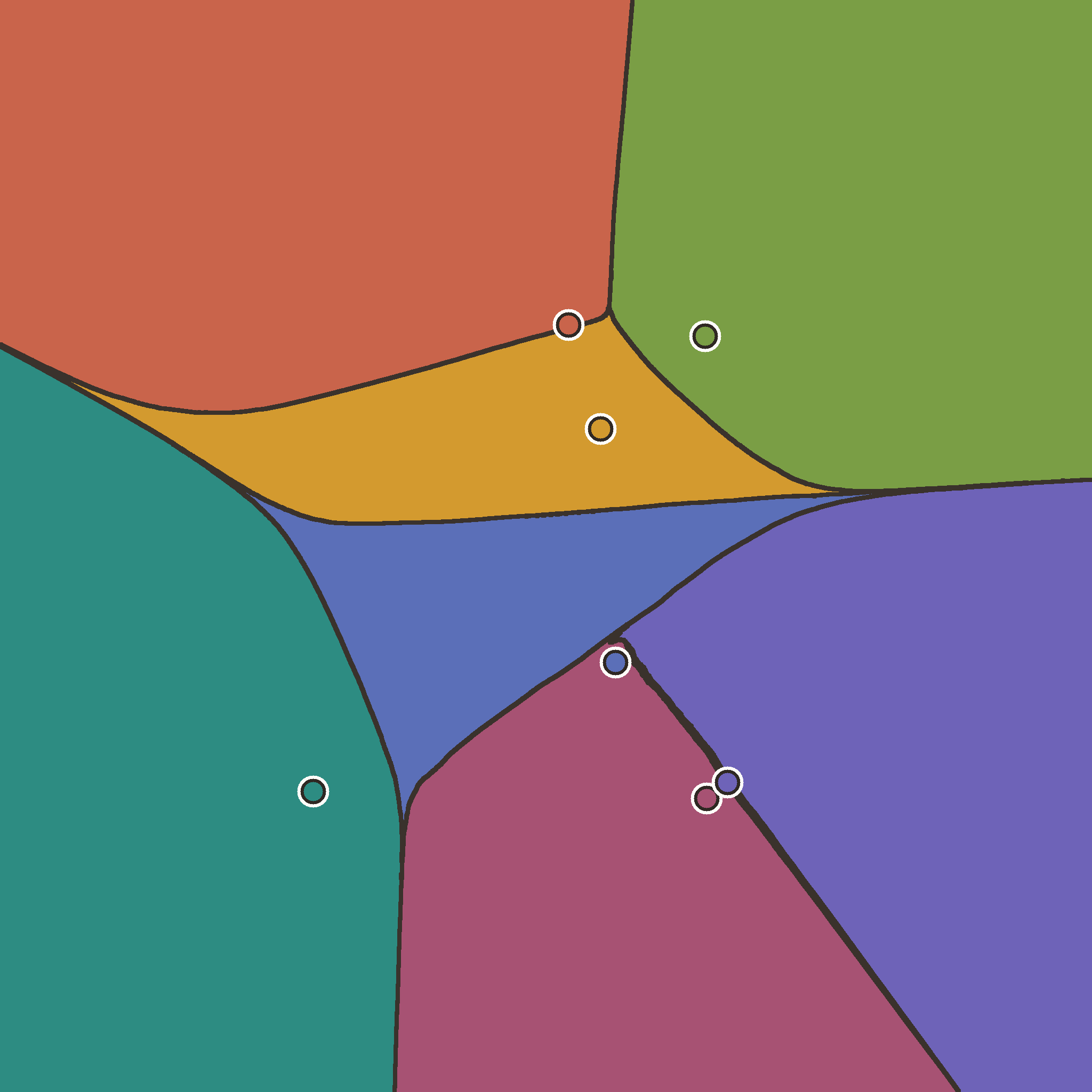}
\caption{$k=25$}
\end{subfigure}\hfill
\begin{subfigure}{0.235\textwidth}
\centering
\includegraphics[width=\linewidth]{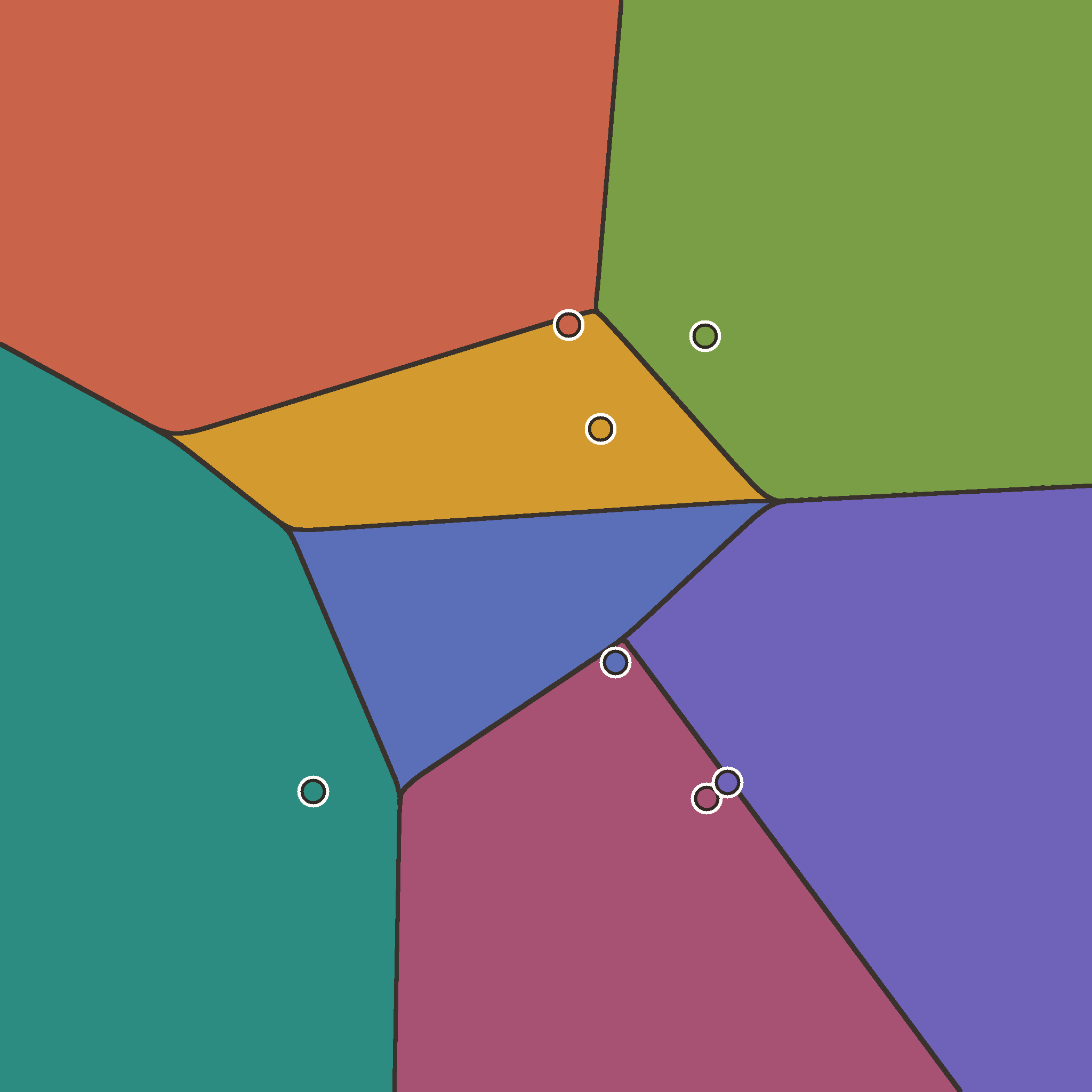}
\caption{$k=1\,000$}
\end{subfigure}\hfill
\begin{subfigure}{0.235\textwidth}
\centering
\includegraphics[width=\linewidth]{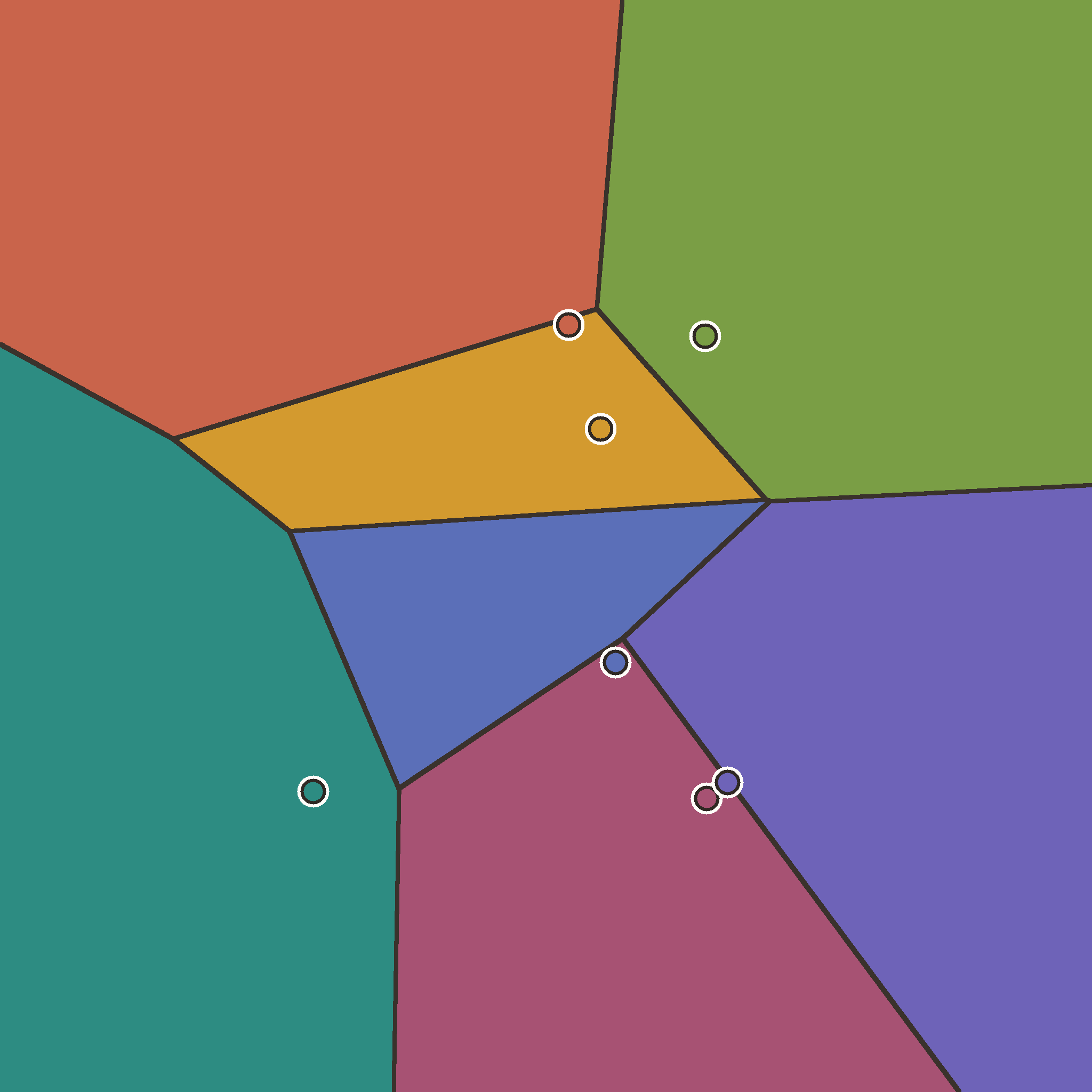}
\caption{OT plan $\pi^{\star}$}
\end{subfigure}
\caption{The terminal map $\phi^{\overline{\pi}_{k}}$ induced by the flow partitions
$\mathbb{R}^{2}$ into cells that approach the semidiscrete OT Laguerre
partition. The case $k=1$ corresponds to the independent coupling. In this example, $\mu=\mathcal{N}(0,\,I_{2})$ and $\nu$ is uniformly supported on $7$ fixed atoms, shown as colored points.
}
\label{fig:expected-batch-ot-flow-cells}
\end{figure}

\subsection{Posterior concentration}\label{sec:posterior-concentration}

The concentration in time of the posterior weights
$\mathbb{P}_{(X_{0},X_{1}) \sim \pi}(X_{1}=v_{j}\mid X_{t}=x)$ determines how quickly the posterior mean $m_{t}^{\pi}(x)$ converges to a single specific atom, according to~\eqref{eq:gaussian-discrete-velocity-posterior-mean}. \cref{fig:posterior-concentration} reports
\(
\mathbb{E}\left[\max_{1\leqslant i\leqslant M} \mathbb{P}_{(X_{0},X_{1}) \sim \overline{\pi}_{k}}(X_{1}=v_{i}\mid X_{t})\right]
\)
when $k$ increases (left panel) for a fixed dimension $d=20$, or when $d$ increases (right panel) for the independent coupling $k=1$. Both panels show stronger
concentration over time as $k$ or $d$ increases. For $k=1$, the increase in concentration with $d$ was already observed by~\cite{bertrand}; we turn this observation into a quantitative theoretical statement in \cref{app:posterior-concentration-euler}. The increase with $k$ is harder to prove, since the OT batch plan is not available in closed form, but we hope to provide a theoretical analysis in future work.

\begin{figure}[t]
\centering
\includegraphics[width=\linewidth]{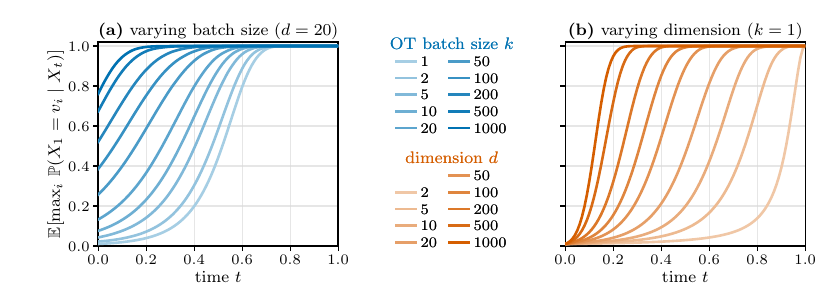}
\caption{Posterior concentration increases with OT batch size $k$ (left), and with ambient dimension $d$ when $k=1$ (right) in a Gaussian-to-discrete setting with
$M=100$ atoms.
Standard errors over sampled trajectories are below plotting
resolution. The batch-size experiment uses a fixed cloud of $100$ atoms drawn uniformly in $[-1,\,1]^{20}$; the dimension experiment uses one such cloud for each $d$.}
\label{fig:posterior-concentration}
\end{figure}

\subsection{Numerical integration} \label{sec:numerical_integration}

We now turn to the interaction between the OT batch size $k$ and the numerical integration budget,
measured by the number of function evaluations (NFE).
For any $\pi\in\Pi(\mu,\,\nu)$, define the explicit Euler iterates with $n\in\mathbb{N}^{*}$ uniform time steps $t_{j}=j/n$ by
\begin{equation}
f_{0,\,n}^{\pi}=\operatorname{id}\quad\text{and}\quad
f_{j+1,\,n}^{\pi}=(\operatorname{id}+\tfrac{1}{n}u_{t_{j}}^{\pi})\circ f_{j,\,n}^{\pi}
\quad\text{for}\;j\in\{0,\ldots,n-1\}.
\label{eq:euler-recursion}
\end{equation}
We denote the final Euler map by $f_{n}^{\pi}=f_{n,\,n}^{\pi}$.
For $n,\,k\in\mathbb{N}^{*}$, we quantify the expected Euler integration error of the target velocity $u^{\overline{\pi}_{k}}$ by
\(
\mathscr{E}_{n,\,k}
=\mathbb{E}_{X_{0}\sim\mu}\left[\lVert f_{n}^{\overline{\pi}_{k}}(X_{0})-\phi^{\overline{\pi}_{k}}(X_{0})\rVert\right].
\)

Can the OT batch size efficiently compensate for fewer function evaluations? To isolate the effect of the two parameters on $\mathscr{E}_{n,\,k}$, we consider a tractable model and compare the asymptotics along the two extreme regimes $k=1$ and $n=1$. This simple setting already reveals a striking difference in how $\mathscr{E}_{n,\,k}$ may decay with respect to $n$ versus $k$.

\begin{prop}[Asymptotic comparison between OT batch size and NFE]
\label{prop:binary-integration-asymptotics}
Let $\mu=\mathcal{N}(0,\,1)$ and $\nu=\operatorname{Unif}\{-1,\,1\}$.
Then $\log\mathscr{E}_{n,\,1}\sim-\frac{2\pi}{\sqrt{3}}n^{\frac{2}{3}}$ as $n\to\infty$
and $\log\mathscr{E}_{1,\,k}\sim-\frac{1}{2}\log k$ as $k\to\infty$.
\end{prop}
Increasing the number of Euler steps thus causes the integration error to decay as
$\exp(-c\,n^{2/3})$, whereas increasing the OT batch size only yields a
$k^{-1/2}$-type decay.
In this specific setting, the influence of the NFE is considerably more pronounced, and this asymmetry remains clear even outside the asymptotic regime.
\cref{fig:binary-error-contours} displays the level sets of $\mathscr{E}_{n,\,k}$ in the
$(n,\,k)$ plane. It shows for example that to compensate for a reduction in NFE from $25$ to $10$, one needs to increase OT batch size $k$ from $1$ to $90$, well beyond the two-atom support size.

\begin{figure}[t]
\centering
\includegraphics[width=\linewidth]{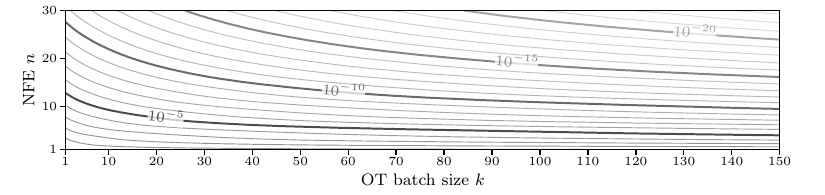}
\caption{Level sets of the Euler integration error $\mathscr{E}_{n,\,k}$ in the two-atom model of \cref{prop:binary-integration-asymptotics}.}
\label{fig:binary-error-contours}
\end{figure}

We observe a similar qualitative asymmetry in higher-dimensional semidiscrete examples with larger finite support. Although we were not able to prove it formally, we conjecture that in this general case, $\mathscr{E}_{n,\,k}$ decays as $1/n$ and $1/\sqrt{k}$.
As a result, halving the expected integration error would require either doubling the NFE, and hence roughly doubling inference time, or multiplying the OT batch size by $4$, thereby multiplying the OT solve time during training by $4^3=64$. This highlights the difficulty of using larger OT batches as a surrogate for additional inference steps.
\cref{fig:nfe-vs-batch-size} shows the (Monte Carlo) average integration error $\hat{\mathscr{E}}_{n,\,k}$ as a function of $n$ and $k$, for $\nu=\operatorname{Unif}\{v_{1},\ldots,v_{100}\}$ with $\{v_{1},\ldots,v_{100}\}\subseteq[-1,\,1]^{20}$. For each pair $(n,\,k)$, we estimate the numerical integration error of the expected batch OT flow
by comparing a coarse $n$-step Euler discretization with a much finer reference discretization.
The error also appears to be non-monotonic in $k$ for small values,
reflecting the fact that the monotonicity of the expected transport cost~(\Cref{prop:expected-batch-ot-cost-convergence}) does not necessarily imply monotonicity of the induced numerical error.

\begin{figure}[t]
\centering
\includegraphics[width=\linewidth]{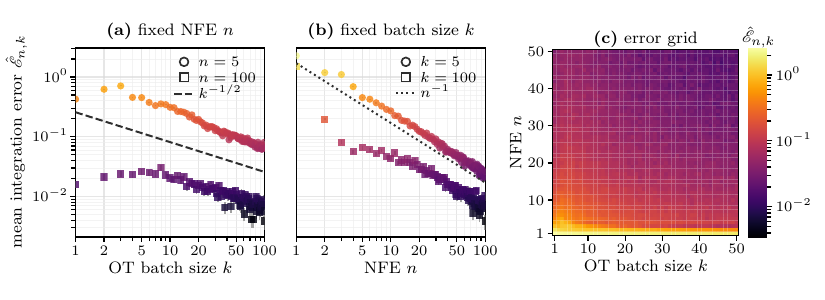}
\caption{In a Gaussian-to-discrete setting with \(100\) atoms in dimension \(20\),
the numerical integration error is consistent with a \(1/n\) scaling in the number of Euler steps
\(n\) and a \(1/\sqrt{k}\) scaling in the OT batch size \(k\).
Error bars show \(\pm1\) standard error over sampled initial conditions.}
\label{fig:nfe-vs-batch-size}
\end{figure}

\subsection{Image dataset experiments}

We now evaluate expected batch OT on two image datasets: CIFAR-10 \citep{cifar} and SVHN \citep{svhn}. These experiments probe the learned neural velocity field, rather than the exact target velocity studied above.

For each dataset, we define $\nu$ as the uniform empirical measure on the flattened images after normalization to $[-1,\,1]^{3072}$.
We keep the optimization minibatch size fixed and vary only the OT batch size
$k\in\{1,\,2,\,4,\ldots,8192\}$.
For each value of $k$, we train a separate flow model using exact OT couplings computed on batches of size $k$.
We then generate samples with explicit Euler using a prescribed NFE, and evaluate them with the Fréchet Inception Distance (FID, \citet{fid}). We stop at $k=8192$ because larger exact OT batches exceeded our computational budget.
For each dataset and each $k$, we report FID over six training seeds.

\cref{fig:image-cost-fid-vs-batch-size} (left) shows that
$\mathbb{E}[W_{2}^{2}(\widehat{\mu}_{k},\,\widehat{\nu}_{k})]$ decreases steadily on both datasets, with no visible saturation up to $k=8192$.
By \cref{prop:expected-batch-ot-plan-convergence}(1), the expected batch OT plan should not be regarded as effectively converged.
\cref{fig:image-cost-fid-vs-batch-size} (right) shows that increasing $k$ improves FID mainly when the integration scheme is coarse ($\mathrm{NFE} = 10$). At high NFE, the effect is weaker and can reverse: on SVHN, large $k$ degrades the FID at $\mathrm{NFE}=50$ and $100$. This effect can also be seen qualitatively on generated samples: see~\cref{fig:cifar-nfe10,fig:cifar-nfe100,fig:svhn-nfe5,fig:svhn-nfe100} in~\cref{app:experimental-details}.

We recall that the effect of $k$ and NFE occurs at different stages: increasing $k$ affects
training time, while NFE determines the cost of sampling at inference. The improvement at low NFE in
\cref{fig:image-cost-fid-vs-batch-size} therefore suggests a possible
training-inference tradeoff: larger OT batches can improve low-NFE sampling,
at the price of additional training computation. On our hardware, exact OT
batches up to $k=1024$ add about $3\%$ training overhead, whereas
$k=8192$ increases training time by $55\%$
(\cref{tab:training-inference-wall-time}; see \cref{sec:comp_tradeoff} for
details).

\begin{figure}[t]
\centering
\includegraphics[width=\linewidth]{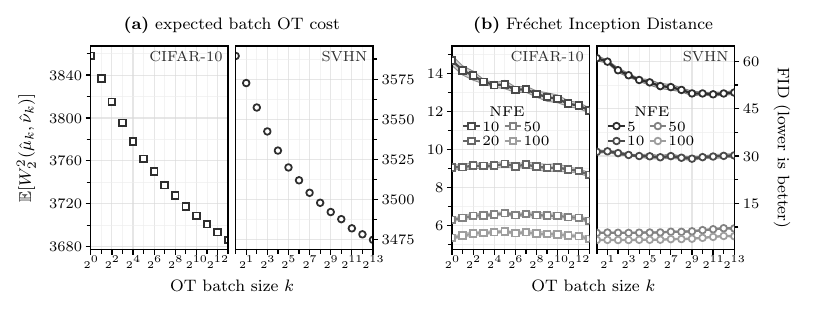}
\caption{\textbf{Image experiment}. The expected batch OT cost decreases steadily up to \(k=8192\) (left).
Increasing \(k\) improves FID mainly in the low-NFE regime, while the effect weakens or reverses at large NFE (right).
Error bands show \(\pm1\) empirical standard deviation over training seeds.}
\label{fig:image-cost-fid-vs-batch-size}
\end{figure}

Our experiment complements previous empirical studies.
\citet{pooladian,tong} fix the training batch size $B$ and compare independent pairing with minibatch OT only in the case $k=B$; thus, they do not explore a grid of OT batch sizes $k$.
\citet{zhang} study much larger batch couplings with entropic OT, instead of exact batch OT.
Moreover, these works do not track the expected batch OT cost
$\mathbb{E}[W_{2}^{2}(\widehat{\mu}_{k},\,\widehat{\nu}_{k})]$ as a function of $k$.
A comparison between exact batch OT and entropic or sliced alternatives is left for future work.

\section{Conclusion}

We studied the expected batch OT plan as a population-level object underlying minibatch OT for flow matching, complementing~\citet{fatras,pooladian}. In the semidiscrete setting, we established quantitative convergence rates and proved that the induced velocity field is regular enough to define a well-posed flow. Taken together, these results put minibatch OT on firmer theoretical ground and sharpen its practical interpretation: increasing the OT batch size can meaningfully improve low-NFE generation, but it does not eliminate the need for accurate inference-time integration and the gains must be weighed against additional training-time OT cost.

\paragraph{Limitations and future directions.}
\begin{itemize}[leftmargin=*]
  \item Although the expected batch OT plan can be formulated in full generality, most of our quantitative results are obtained in the semidiscrete setting, which is the one directly relevant to generative modeling. Extending comparable convergence rates and structural properties to more general source-target configurations would further clarify the scope of the construction.
  \item Our theoretical analysis concerns the true velocity field induced by the coupling, whereas practical FM relies on an approximation learned by a neural network. Understanding how our results transfer to learned velocities is therefore a central next step, even if analyzing the exact target dynamics is a necessary starting point.
  \item We obtained a theoretical comparison between OT batch size and numerical integration in a tractable, one-dimensional two-atom model using the Euler method. Deriving general bounds for the behavior of $\mathscr{E}_{n,\,k}$ as a function of both $n$ and $k$, and for other numerical schemes, appears much more challenging and remains open.
  \item Expected batch OT is only one possible coupling surrogate between independent pairing and exact OT. Similar questions about rectifiability and interaction with numerical integration could be addressed for other couplings used in practice, such as entropic and sliced variants.
\end{itemize}

\section*{Acknowledgements}
This research was supported, in part, by the Agence nationale de la recherche (ANR) under the France 2030 programme, through the PEPR IA projects PDE-AI (ANR-23-PEIA-0004) and MacLeOD (ANR-25-PEIA-0005).
The authors declare no competing interests.

\clearpage
\bibliographystyle{apalike}
\bibliography{references}

\begin{thebibliography}{}

\bibitem[Albergo et~al., 2024]{albergo2024stochastic}
Albergo, M.~S., Goldstein, M., Boffi, N.~M., Ranganath, R., and Vanden-Eijnden,
  E. (2024).
\newblock Stochastic interpolants with data-dependent couplings.
\newblock In {\em {ICML}}, pages 921--937.

\bibitem[Ambrosio et~al., 2005]{ambrosio}
Ambrosio, L., Gigli, N., and Savar{\'e}, G. (2005).
\newblock {\em Gradient Flows: In Metric Spaces and in the Space of Probability
  Measures}.
\newblock Birkh{\"a}user Basel, 1 edition.

\bibitem[Bertrand et~al., 2025]{bertrand}
Bertrand, Q., Gagneux, A., Massias, M., and Emonet, R. (2025).
\newblock On the closed-form of flow matching: Generalization does not arise
  from target stochasticity.
\newblock In {\em {NeurIPS}}.

\bibitem[Bobkov and Ledoux, 2019]{bobkovledoux}
Bobkov, S. and Ledoux, M. (2019).
\newblock One-dimensional empirical measures, order statistics, and
  {Kantorovich} transport distances.
\newblock {\em Memoirs of the American Mathematical Society}, 261(1259):1--126.

\bibitem[Boucheron and Massart, 2011]{boucheron}
Boucheron, S. and Massart, P. (2011).
\newblock A high-dimensional {Wilks} phenomenon.
\newblock {\em Probability Theory and Related Fields}, 150(3--4):405--433.

\bibitem[Chen et~al., 2018]{torchdiffeq}
Chen, R. T.~Q., Rubanova, Y., Bettencourt, J., and Duvenaud, D.~K. (2018).
\newblock Neural ordinary differential equations.
\newblock In {\em {NeurIPS}}.

\bibitem[Dodson et~al., 2026]{dodson2026two}
Dodson, N., Gao, X., Wang, Q., Wang, Y., and Wan, Z. (2026).
\newblock Two calm ends and the wild middle: A geometric picture of
  memorization in diffusion models.
\newblock In {\em {ICML}}, pages 25643--25694.

\bibitem[Fatras et~al., 2021a]{fatras21a}
Fatras, K., S{\'e}journ{\'e}, T., Courty, N., and Flamary, R. (2021a).
\newblock Unbalanced minibatch optimal transport; applications to domain
  adaptation.
\newblock In {\em {ICML}}, pages 3186--3197.

\bibitem[Fatras et~al., 2020]{fatras20a}
Fatras, K., Zine, Y., Flamary, R., Gribonval, R., and Courty, N. (2020).
\newblock Learning with minibatch {Wasserstein}: asymptotic and gradient
  properties.
\newblock In {\em {AISTATS}}, pages 2131--2141.

\bibitem[Fatras et~al., 2021b]{fatras}
Fatras, K., Zine, Y., Majewski, S., Flamary, R., Gribonval, R., and Courty, N.
  (2021b).
\newblock Minibatch optimal transport distances; analysis and applications.
\newblock arXiv:2101.01792v1.

\bibitem[Flamary et~al., 2021]{pot}
Flamary, R., Courty, N., Gramfort, A., et~al. (2021).
\newblock {POT}: {Python Optimal Transport}.
\newblock {\em Journal of Machine Learning Research}, 22(78):1--8.

\bibitem[Fournier and Guillin, 2015]{fournier2015}
Fournier, N. and Guillin, A. (2015).
\newblock On the rate of convergence in {Wasserstein} distance of the empirical
  measure.
\newblock {\em Probability Theory and Related Fields}, 162(3--4):707--738.

\bibitem[Genans et~al., 2025]{genans}
Genans, F., Godichon-Baggioni, A., Vialard, F.-X., and Wintenberger, O. (2025).
\newblock Stochastic optimization in semi-discrete optimal transport:
  Convergence analysis and minimax rate.
\newblock In {\em {NeurIPS}}.

\bibitem[Harris et~al., 2020]{numpy}
Harris, C.~R., Millman, K.~J., van~der Walt, S.~J., et~al. (2020).
\newblock Array programming with {NumPy}.
\newblock {\em Nature}, 585(7825):357--362.

\bibitem[Hertrich et~al., 2025]{hertrich}
Hertrich, J., Chambolle, A., and Delon, J. (2025).
\newblock On the relation between rectified flows and optimal transport.
\newblock In {\em {NeurIPS}}.

\bibitem[Heusel et~al., 2017]{fid}
Heusel, M., Ramsauer, H., Unterthiner, T., Nessler, B., and Hochreiter, S.
  (2017).
\newblock {GANs} trained by a two time-scale update rule converge to a local
  {Nash} equilibrium.
\newblock In {\em {NeurIPS}}.

\bibitem[Hundrieser et~al., 2024]{hundrieser}
Hundrieser, S., Staudt, T., and Munk, A. (2024).
\newblock Empirical optimal transport between different measures adapts to
  lower complexity.
\newblock {\em Annales de l'Institut Henri Poincaré, Probabilités et
  Statistiques}, 60(2):824--846.

\bibitem[Klatt et~al., 2022]{klatt2022}
Klatt, M., Munk, A., and Zemel, Y. (2022).
\newblock Limit laws for empirical optimal solutions in random linear programs.
\newblock {\em Annals of Operations Research}, 315(1):251--278.

\bibitem[Krizhevsky, 2009]{cifar}
Krizhevsky, A. (2009).
\newblock Learning multiple layers of features from tiny images.
\newblock Technical report, University of Toronto.

\bibitem[Lipman et~al., 2023]{lipman}
Lipman, Y., Chen, R. T.~Q., Ben-Hamu, H., Nickel, M., and Le, M. (2023).
\newblock Flow matching for generative modeling.
\newblock In {\em {ICLR}}.

\bibitem[Liu et~al., 2023]{liu2022flow}
Liu, X., Gong, C., and Liu, Q. (2023).
\newblock Flow straight and fast: Learning to generate and transfer data with
  rectified flow.
\newblock In {\em {ICLR}}.

\bibitem[Mousavi-Hosseini et~al., 2026]{mousavi}
Mousavi-Hosseini, A., Zhang, S., Klein, M., and Cuturi, M. (2026).
\newblock Flow matching with semidiscrete couplings.
\newblock In {\em {ICLR}}.

\bibitem[Netzer et~al., 2011]{svhn}
Netzer, Y., Wang, T., Coates, A., Bissacco, A., Wu, B., and Ng, A.~Y. (2011).
\newblock Reading digits in natural images with unsupervised feature learning.
\newblock In {\em {NIPS} Workshop on Deep Learning and Unsupervised Feature
  Learning}.

\bibitem[Parmar et~al., 2022]{cleanfid}
Parmar, G., Zhang, R., and Zhu, J.-Y. (2022).
\newblock On aliased resizing and surprising subtleties in {GAN} evaluation.
\newblock In {\em {CVPR}}, pages 11400--11410.

\bibitem[Paszke et~al., 2019]{torchvision}
Paszke, A., Gross, S., Massa, F., et~al. (2019).
\newblock {PyTorch}: An imperative style, high-performance deep learning
  library.
\newblock In {\em {NeurIPS}}.

\bibitem[Peyr{\'e} and Cuturi, 2019]{peyre}
Peyr{\'e}, G. and Cuturi, M. (2019).
\newblock Computational optimal transport: With applications to data science.
\newblock {\em Foundations and Trends in Machine Learning}, 11(5--6):355--607.

\bibitem[Pierret et~al., 2026a]{pierret-tessellations}
Pierret, E., Hertrich, J., Hurault, S., and Delon, J. (2026a).
\newblock Tessellations of semi-discrete flow matching.
\newblock In {\em {NeurIPS}}.
\newblock To appear. arXiv:2605.07513v1.

\bibitem[Pierret et~al., 2026b]{pierret}
Pierret, E., Tosel, V., Delon, J., and Newson, A. (2026b).
\newblock Flow matching for applied mathematicians.
\newblock HAL: hal-05538982v2.

\bibitem[Poli et~al., 2021]{torchdyn}
Poli, M., Massaroli, S., Yamashita, A., Asama, H., Park, J., and Ermon, S.
  (2021).
\newblock {TorchDyn}: Implicit models and neural numerical methods in
  {PyTorch}.
\newblock Software package.

\bibitem[Pooladian et~al., 2023a]{pooladian}
Pooladian, A.-A., Ben-Hamu, H., Domingo-Enrich, C., Amos, B., Lipman, Y., and
  Chen, R. T.~Q. (2023a).
\newblock Multisample flow matching: Straightening flows with minibatch
  couplings.
\newblock In {\em {ICML}}, pages 28100--28127.

\bibitem[Pooladian et~al., 2023b]{minimax}
Pooladian, A.-A., Divol, V., and Niles-Weed, J. (2023b).
\newblock Minimax estimation of discontinuous optimal transport maps: The
  semi-discrete case.
\newblock In {\em {ICML}}, pages 28128--28150.

\bibitem[Santambrogio, 2015]{santambrogio}
Santambrogio, F. (2015).
\newblock {\em Optimal Transport for Applied Mathematicians: Calculus of
  Variations, {PDEs}, and Modeling}.
\newblock Birkh{\"a}user Cham, 1 edition.

\bibitem[Staudt and Hundrieser, 2025]{staudt-unbounded}
Staudt, T. and Hundrieser, S. (2025).
\newblock Convergence of empirical optimal transport in unbounded settings.
\newblock {\em Bernoulli}, 31(3):1929--1954.

\bibitem[Tong et~al., 2024]{tong}
Tong, A., Fatras, K., Malkin, N., et~al. (2024).
\newblock Improving and generalizing flow-based generative models with
  minibatch optimal transport.
\newblock {\em Transactions on Machine Learning Research}.

\bibitem[{TorchVision maintainers and contributors},
  2024]{torchvision-software}
{TorchVision maintainers and contributors} (2024).
\newblock {TorchVision}: {PyTorch}'s computer vision library.
\newblock Software release 0.20.0.

\bibitem[van~der Vaart and Wellner, 1996]{vanderVaart1996}
van~der Vaart, A.~W. and Wellner, J.~A. (1996).
\newblock {\em Weak Convergence and Empirical Processes: With Applications to
  Statistics}.
\newblock Springer, 1 edition.

\bibitem[Virtanen et~al., 2020]{scipy}
Virtanen, P., Gommers, R., Oliphant, T.~E., et~al. (2020).
\newblock {{SciPy} 1.0: Fundamental Algorithms for Scientific Computing in
  Python}.
\newblock {\em Nature Methods}, 17(3):261--272.

\bibitem[Wan et~al., 2025]{wan}
Wan, Z., Wang, Q., Mishne, G., and Wang, Y. (2025).
\newblock Elucidating flow matching {ODE} dynamics via data geometry and
  denoisers.
\newblock In {\em {ICML}}, pages 62020--62083.

\bibitem[Zhang et~al., 2026]{zhang}
Zhang, S.~Y., Mousavi-Hosseini, A., Klein, M., and Cuturi, M. (2026).
\newblock On fitting flow models with large {Sinkhorn} couplings.
\newblock {\em Transactions on Machine Learning Research}.

\end{thebibliography}

\newpage
\appendix

\section[Proofs for Section 3]{Proofs for \cref{sec:expected-batch-ot-plan}}\label{app:expected-batch-ot-proofs}

This appendix follows the structure of \cref{sec:expected-batch-ot-plan}. 
We first justify the construction of \(\overline{\pi}_{k}\) and prove its elementary properties (\cref{app:expected-batch-ot-construction}). 
We then analyze its transport cost through the empirical OT cost \(\mathbb{E}[W_{2}^{2}(\widehat{\mu}_{k},\,\widehat{\nu}_{k})]\) (\cref{app:expected-batch-ot-cost-convergence,app:expected-batch-ot-cost-rates}). 
The last part proves convergence of \(\overline{\pi}_{k}\) as a plan and gives the cost-to-plan estimates used in \cref{prop:expected-batch-ot-plan-convergence} (\cref{app:expected-batch-ot-plan-convergence}).

Throughout, if $\mathcal{Y} \subseteq \mathbb{R}^d$ is a finite set containing at least two points, we define
\[
\operatorname{sep}(\mathcal{Y}) =
\min_{\substack{y,\,y'\in\mathcal{Y}\\ y\neq y'}}
\lVert y-y'\rVert\,,
\quad \text{and} \quad
\mathrm{diam}(\mathcal{Y})=\max_{y,y'\in\mathcal{Y}}\lVert y-y'\rVert. 
\]

\subsection[Construction and elementary properties]{Construction and elementary properties (\cref{sec:expected-batch-ot-definition})}\label{app:expected-batch-ot-construction}

We detail the construction of the expected batch OT plan
\(\overline{\pi}_{k}\) and record its elementary properties.

\subsubsection{Empirical OT solver}\label{app:empirical-ot-measurable-solver}

An empirical OT problem may have several solutions, and the solver
used in the implementation of batch OT typically selects one. To model
this, we fix a measurable empirical OT solver, and define the corresponding expected
batch OT plan as the average of these empirical OT plans over random batches.
We then show that the solver-dependence disappears when $\mu$ or $\nu$
is absolutely continuous.

For $k\in\mathbb{N}^{*}$, write $E_{k}=(\mathbb{R}^{d})^{k}\times(\mathbb{R}^{d})^{k}$,
$F=\mathbb{R}^{d}\times\mathbb{R}^{d}$ and $\alpha_{k}=\mu^{\otimes k}\otimes\nu^{\otimes k}$.
For $(\mathbf{x}_{k},\,\mathbf{y}_{k})\in E_{k}$, recall that $\widehat{\mu}_{\mathbf{x}_{k}}=\frac{1}{k}\sum_{i=1}^{k}\updelta_{x_{i}}$
and $\widehat{\nu}_{\mathbf{y}_{k}}=\frac{1}{k}\sum_{j=1}^{k}\updelta_{y_{j}}$.
We consider a measurable solver  
\[
\widehat{\sigma}_k \colon
\begin{array}[t]{c@{\;}c@{\;}c}
  E_k
    & \longrightarrow
    & \mathfrak{S}_k, \\
  (\mathbf{x}_k,\,\mathbf{y}_k)
    & \longmapsto
    & \widehat{\sigma}_{\mathbf{x}_k,\,\mathbf{y}_k},
\end{array}
\]
such that, for every \((\mathbf{x}_{k},\,\mathbf{y}_{k})\in E_{k}\),
\[
\widehat{\sigma}_{\mathbf{x}_{k},\,\mathbf{y}_{k}}
\in
\operatorname*{argmin}_{\sigma\in\mathfrak{S}_{k}}
C_{\sigma}(\mathbf{x}_{k},\,\mathbf{y}_{k}),
\qquad
C_{\sigma}(\mathbf{x}_{k},\,\mathbf{y}_{k}) = \frac{1}{k}\sum_{i=1}^{k}
\lVert x_{i}-y_{\sigma(i)}\rVert^{2} .
\]
Let $\widehat{\pi}_{\mathbf{x}_{k},\,\mathbf{y}_{k}}=\frac{1}{k}\sum_{i=1}^{k}\updelta_{(x_{i},\,y_{\widehat{\sigma}_{\mathbf{x}_{k},\,\mathbf{y}_{k}}(i)})}$
be the corresponding OT plan between $\widehat{\mu}_{\mathbf{x}_{k}}$ and $\widehat{\nu}_{\mathbf{y}_{k}}$.

\paragraph*{Existence of a measurable solver.}

Fix a total order $\preceq_{k}$ on the finite
set $\mathfrak{S}_{k}$, and choose the $\preceq_{k}$-smallest minimizer
$\widehat{\sigma}_{k}^{\preceq}$. This map is Borel because each $C_\sigma$ is continuous, and for every
$\sigma\in\mathfrak{S}_{k}$,
\[
\{\widehat{\sigma}_{k}^{\preceq}=\sigma\}=\bigcap_{\sigma'\prec_{k}\sigma}\{C_{\sigma}<C_{\sigma'}\}\cap\bigcap_{\sigma'\succeq_{k}\sigma}\{C_{\sigma}\leqslant C_{\sigma'}\}.
\]

\paragraph*{Uniqueness for absolutely continuous measures.}
Generic uniqueness of empirical OT plans when both random supports are
absolutely continuous is proved in
\citet[Lemma~6.5 and Proposition~6.6]{klatt2022}. The same genericity argument gives the
one-sided version needed here. Suppose that \(\mu\) is absolutely continuous
and condition on \(\mathbf y_k=(y_1,\ldots,y_k)\). For two permutations
\(\sigma,\tau\in\mathfrak S_k\), a tie in quadratic cost is equivalent to
\[
  C_{\sigma}(\mathbf{x}_{k},\,\mathbf{y}_{k})-C_{\tau}(\mathbf{x}_{k},\,\mathbf{y}_{k})
  =
  -\frac{2}{k}\sum_{i=1}^k
  \langle x_i,\,y_{\sigma(i)}-y_{\tau(i)}\rangle
  = 0 .
\]
If the two permutations induce distinct empirical plans, this is a nontrivial
hyperplane in \((\mathbb R^d)^k\). It therefore has
\(\mu^{\otimes k}\)-measure zero. A finite union over pairs
\((\sigma,\,\tau)\) shows that, almost surely, all optimal permutations induce
the same empirical plan. Since every coupling between the labelled uniform
empirical measures is a convex combination of permutation plans, the empirical
OT plan is almost surely unique as a probability measure. The case where
\(\nu\) is absolutely continuous is symmetric.

\subsubsection{Well-posedness of the expected batch OT plan}

We can now define the expected batch OT plan using the measurable OT solver
introduced above.

\begin{proof}[Justification for \cref{def:expected-batch-ot-plan}]
The expected batch OT plan is defined by
averaging the solver outputs $\widehat{\pi}_{\mathbf{x}_{k},\,\mathbf{y}_{k}}$
over random batches. We show that this is well-posed.
For $z=(\mathbf{x}_{k},\,\mathbf{y}_{k})\in E_{k}$, let $\eta_{z}=\widehat{\pi}_{\mathbf{x}_{k},\,\mathbf{y}_{k}}\in\mathcal{P}(F)$.
We first show that $(\mathbf{x}_{k},\,\mathbf{y}_{k})\mapsto\widehat{\pi}_{\mathbf{x}_{k},\,\mathbf{y}_{k}}$
is a Borel family of probability measures. For Borel $A\subseteq F$,
\[
\widehat{\pi}_{\mathbf{x}_{k},\,\mathbf{y}_{k}}(A)=\frac{1}{k}\sum_{i=1}^{k}\sum_{\sigma\in\mathfrak{S}_{k}}\mathbbm{1}_{A}\bigl(x_{i},\,y_{\sigma(i)}\bigr)\,\mathbbm{1}_{\{\sigma=\widehat{\sigma}_{\mathbf{x}_{k},\,\mathbf{y}_{k}}\}}.
\]
Since $\widehat{\sigma}_{k}$ is Borel by assumption, the right-hand side
is a Borel function of $(\mathbf{x}_{k},\,\mathbf{y}_{k})$. Therefore,
by \citep[Section~5.3]{ambrosio}, there exists a unique probability
measure $\widetilde{\eta}_{k}\in\mathcal{P}(E_{k}\times F)$ such
that, for every bounded Borel map $f:E_{k}\times F\to\mathbb{R}$,
\[
\int_{E_{k}\times F}f(z,\,z')\,\mathrm{d}\widetilde{\eta}_{k}(z,\,z')=\int_{E_{k}}\left(\int_{F}f(z,\,z')\,\mathrm{d}\eta_{z}(z')\right)\,\mathrm{d}(\mu^{\otimes k}\otimes\nu^{\otimes k})(z).
\]
Let $\operatorname{pr}^{F}:E_{k}\times F\to F$ be the projection
onto the second factor, and define $\overline{\pi}_{k}=\operatorname{pr}^{F}_{\sharp}\widetilde{\eta}_{k}$.
For bounded Borel $g:F\to\mathbb{R}$, applying the change-of-variables
formula for push-forward measures and then the identity above, we obtain
\begin{align*}
\int_{F}g(z')\,\mathrm{d}\overline{\pi}_{k}(z') & =\int_{E_{k}\times F}g(z')\,\mathrm{d}\widetilde{\eta}_{k}(z,\,z')\\
 & =\int_{E_{k}}\left(\int_{F}g(z')\,\mathrm{d}\eta_{z}(z')\right)\,\mathrm{d}(\mu^{\otimes k}\otimes\nu^{\otimes k})(z)\\
 & =\int_{(\mathbb{R}^{d})^{k}\times(\mathbb{R}^{d})^{k}}\left(\int_{\mathbb{R}^{d}\times\mathbb{R}^{d}}g\,\mathrm{d}\widehat{\pi}_{\mathbf{x}_{k},\,\mathbf{y}_{k}}\right)\,\mathrm{d}(\mu^{\otimes k}\otimes\nu^{\otimes k})(\mathbf{x}_{k},\,\mathbf{y}_{k}).
\end{align*}

We now show that $\overline{\pi}_k$ is the law of a random pair produced by the minibatch OT procedure.
If we draw $(\mathbf{X}_{k},\,\mathbf{Y}_{k})\sim\mu^{\otimes k}\otimes\nu^{\otimes k}$
and let $I\sim\operatorname{Unif}\{1,\ldots,k\}$ be independent of
$(\mathbf{X}_{k},\,\mathbf{Y}_{k})$, then $(X_{I},\,Y_{\widehat{\sigma}_{\mathbf{X}_{k},\,\mathbf{Y}_{k}}(I)})\sim\overline{\pi}_{k}$.
Indeed, for every bounded Borel function $g:\mathbb{R}^{d}\times\mathbb{R}^{d}\to\mathbb{R}$,
conditioning on $I$ gives
\begin{align*}
\mathbb{E}[g(X_{I},\,Y_{\widehat{\sigma}_{\mathbf{X}_{k},\,\mathbf{Y}_{k}}(I)})] & =\mathbb{E}\left[\frac{1}{k}\sum_{i=1}^{k}g\left(X_{i},\,Y_{\widehat{\sigma}_{\mathbf{X}_{k},\,\mathbf{Y}_{k}}(i)}\right)\right]\\
 & =\mathbb{E}\left[\int_{\mathbb{R}^{d}\times\mathbb{R}^{d}}g(x,\,y)\,\mathrm{d}\widehat{\pi}_{\mathbf{X}_{k},\,\mathbf{Y}_{k}}(x,\,y)\right]\\
 & =\int_{\mathbb{R}^{d}\times\mathbb{R}^{d}}g(x,\,y)\,\mathrm{d}\overline{\pi}_{k}(x,\,y).
\end{align*}
\end{proof}

\subsubsection{Elementary properties of the expected batch OT plan}\label{app:expected-batch-ot-elementary-properties}

We show that $\overline{\pi}_{k}$ is an admissible coupling between $\mu$ and $\nu$ \citep[Lemma~4.1]{pooladian}.
Indeed, let $f:\mathbb{R}^{d}\to\mathbb{R}$ be bounded and Borel,
and let $\operatorname{pr}^{1}:\mathbb{R}^{d}\times\mathbb{R}^{d}\to\mathbb{R}^{d}$
be defined by $\operatorname{pr}^{1}(x,\,y)=x$. Then
\begin{align}
\int_{\mathbb{R}^{d}}f(x)\,\mathrm{d}(\operatorname{pr}^{1}_{\sharp}\overline{\pi}_{k})(x) & =\int_{\mathbb{R}^{d}\times\mathbb{R}^{d}}f(x)\,\mathrm{d}\overline{\pi}_{k}(x,\,y)\nonumber \\
 & =\int_{(\mathbb{R}^{d})^{k}\times(\mathbb{R}^{d})^{k}}\left(\int_{\mathbb{R}^{d}\times\mathbb{R}^{d}}f(x)\,\mathrm{d}\widehat{\pi}_{\mathbf{x}_{k},\,\mathbf{y}_{k}}(x,\,y)\right)\,\mathrm{d}\alpha_{k}(\mathbf{x}_{k},\,\mathbf{y}_{k})\nonumber \\
 & =\int_{(\mathbb{R}^{d})^{k}\times(\mathbb{R}^{d})^{k}}\left(\int_{\mathbb{R}^{d}}f(x)\,\mathrm{d}\widehat{\mu}_{\mathbf{x}_{k}}(x)\right)\,\mathrm{d}\alpha_{k}(\mathbf{x}_{k},\,\mathbf{y}_{k})\label{eq:expected-batch-ot-first-marginal}\\
 & =\int_{(\mathbb{R}^{d})^{k}}\frac{1}{k}\sum_{i=1}^{k}f(x_{i})\,\mathrm{d}\mu^{\otimes k}(\mathbf{x}_{k})\nonumber \\
 & =\int_{\mathbb{R}^{d}}f(x)\,\mathrm{d}\mu(x).\nonumber
\end{align}
Equation~\eqref{eq:expected-batch-ot-first-marginal} follows from the fact that $\widehat{\pi}_{\mathbf{x}_{k},\,\mathbf{y}_{k}}$
is an OT plan between $\widehat{\mu}_{\mathbf{x}_{k}}$ and $\widehat{\nu}_{\mathbf{y}_{k}}$,
thus its first marginal is $\widehat{\mu}_{\mathbf{x}_{k}}$. We conclude that $\operatorname{pr}^{1}_{\sharp}\overline{\pi}_{k}=\mu$.
The same argument can be adapted to show that $\operatorname{pr}^{2}_{\sharp}\overline{\pi}_{k}=\nu$,
with $\operatorname{pr}^{2}(x,\,y)=y$. Therefore, $\overline{\pi}_{k}\in\Pi(\mu,\,\nu)$.

We now show that $\overline{\pi}_{1} = \mu\otimes\nu$.
By definition, $\widehat{\pi}_{\mathbf{x}_{1},\,\mathbf{y}_{1}}=\updelta_{(x_{1},\,y_{1})}$.
For every Borel set $A\subseteq\mathbb{R}^{d}\times\mathbb{R}^{d}$,
using $g=\mathbbm{1}_{A}$ in \cref{def:expected-batch-ot-plan},
\begin{align*}
\overline{\pi}_{1}(A) & =\int_{\mathbb{R}^{d}\times\mathbb{R}^{d}}\mathbbm{1}_{A}(x,\,y)\,\mathrm{d}\overline{\pi}_{1}(x,\,y)\\
 & =\int_{\mathbb{R}^{d}\times\mathbb{R}^{d}}\left(\int_{\mathbb{R}^{d}\times\mathbb{R}^{d}}\mathbbm{1}_{A}\,\mathrm{d}\widehat{\pi}_{\mathbf{x}_{1},\,\mathbf{y}_{1}}\right)\,\mathrm{d}(\mu\otimes\nu)(\mathbf{x}_{1},\,\mathbf{y}_{1})\\
 & =\int_{\mathbb{R}^{d}\times\mathbb{R}^{d}}\widehat{\pi}_{\mathbf{x}_{1},\,\mathbf{y}_{1}}(A)\,\mathrm{d}(\mu\otimes\nu)(x_{1},\,y_{1})\\
 & =\int_{\mathbb{R}^{d}\times\mathbb{R}^{d}}\mathbbm{1}_{A}(x_{1},\,y_{1})\,\mathrm{d}(\mu\otimes\nu)(x_{1},\,y_{1})\\
 & =(\mu\otimes\nu)(A).
\end{align*}
Therefore, $\overline{\pi}_{1}=\mu\otimes\nu$.

\subsection[Properties of the expected batch OT cost]{Properties of the expected batch OT cost (\cref{prop:expected-batch-ot-cost-identity} and \cref{prop:expected-batch-ot-cost-convergence}(1))}
\label{app:expected-batch-ot-cost-convergence}

We prove the equivalent formulation of the expected batch OT cost (\cref{prop:expected-batch-ot-cost-identity}).
We then establish that this cost is nonincreasing in $k$ and converges to the population OT cost (\cref{prop:expected-batch-ot-cost-convergence}(1)).

\begin{proof}[Proof of \cref{prop:expected-batch-ot-cost-identity}]
Let $k\in\mathbb{N}^*$. For $n\in\mathbb N^*$, define the truncated cost
function $g_{n}:(x,y)\mapsto\min(\lVert x-y\rVert^{2},\,n)$ for $(x,y) \in \mathbb{R}^d \times \mathbb{R}^d$. Applying \cref{def:expected-batch-ot-plan} with $g_n$ gives,
\begin{equation}
\begin{aligned}
\int_{\mathbb{R}^{d}\times\mathbb{R}^{d}}g_{n}\,\mathrm{d}\overline{\pi}_{k}
&=\int_{(\mathbb{R}^{d})^{k}\times(\mathbb{R}^{d})^{k}}
\left(\int_{\mathbb{R}^{d}\times\mathbb{R}^{d}}g_{n}\,
\mathrm{d}\widehat{\pi}_{\mathbf{x}_{k},\,\mathbf{y}_{k}}\right)
\,\mathrm{d}(\mu^{\otimes k}\otimes\nu^{\otimes k})
(\mathbf{x}_{k},\,\mathbf{y}_{k}).
\end{aligned}
\label{eq:cost-identity-before-monotone-convergence}
\end{equation}
Since $g_{n+1}\geqslant g_{n}\geqslant0$ and $\lim_{n\to+\infty}g_{n}(x,\,y)=\lVert x-y\rVert^{2}$,
by the monotone convergence theorem,
\begin{equation}
\lim_{n \to +\infty} \int_{\mathbb{R}^{d}\times\mathbb{R}^{d}}g_{n}\,\mathrm{d}\overline{\pi}_{k}
= 
\int_{\mathbb{R}^{d}\times\mathbb{R}^{d}}
\lVert x-y\rVert^{2}\,\mathrm{d}\overline{\pi}_{k}.
\label{eq:cost-identity-monotone-convergence-left}
\end{equation}
Moreover, for any $(\mathbf{x}_{k},\,\mathbf{y}_{k}) \in
(\mathbb{R}^{d})^{k}\times(\mathbb{R}^{d})^{k}$,
\begin{align*}
\lim_{n \to +\infty} \int_{\mathbb{R}^{d}\times\mathbb{R}^{d}}g_{n}\,
\mathrm{d}\widehat{\pi}_{\mathbf{x}_{k},\,\mathbf{y}_{k}}
&= \lim_{n \to +\infty} \frac{1}{k}\sum_{i=1}^{k}
g_{n}(x_{i},\,y_{\widehat{\sigma}_{\mathbf{x}_{k},\,\mathbf{y}_{k}}(i)}) \\
&= \frac{1}{k}\sum_{i=1}^{k}
\lVert x_{i}-y_{\widehat{\sigma}_{\mathbf{x}_{k},\,\mathbf{y}_{k}}(i)}\rVert^{2}
=W_{2}^{2}(\widehat{\mu}_{\mathbf{x}_{k}},\,\widehat{\nu}_{\mathbf{y}_{k}}).
\end{align*}
Since $(\int_{\mathbb{R}^{d}\times\mathbb{R}^{d}}g_{n}\,
\mathrm{d}\widehat{\pi}_{\mathbf{x}_{k},\,\mathbf{y}_{k}})_{n \in \mathbb{N}^*}$ is nondecreasing, we apply
the monotone convergence theorem once more on the right-hand side of \eqref{eq:cost-identity-before-monotone-convergence}.
Combining with \eqref{eq:cost-identity-monotone-convergence-left}, we then obtain
\begin{align*}
\int_{\mathbb{R}^{d}\times\mathbb{R}^{d}}\lVert x-y\rVert^{2}\,
\mathrm{d}\overline{\pi}_{k}(x,\,y)
&=\int_{(\mathbb{R}^{d})^{k}\times(\mathbb{R}^{d})^{k}}
W_{2}^{2}(\widehat{\mu}_{\mathbf{x}_{k}},\,\widehat{\nu}_{\mathbf{y}_{k}})
\,\mathrm{d}(\mu^{\otimes k}\otimes\nu^{\otimes k})
(\mathbf{x}_{k},\,\mathbf{y}_{k})\\
&=\mathbb{E}_{(\mathbf{X}_{k},\,\mathbf{Y}_{k})
\sim\mu^{\otimes k}\otimes\nu^{\otimes k}}
[W_{2}^{2}(\widehat{\mu}_{k},\,\widehat{\nu}_{k})].
\end{align*}
\end{proof}

\begin{proof}[Proof of \cref{prop:expected-batch-ot-cost-convergence}(1)]
\emph{Monotonicity.}
We show that
\(
\mathbb{E}[W_{2}^{2}(\widehat{\mu}_{k+1},\,\widehat{\nu}_{k+1})]\leqslant\mathbb{E}[W_{2}^{2}(\widehat{\mu}_{k},\,\widehat{\nu}_{k})]
\)
for any $k\in\mathbb{N}^{*}$.
This result was first stated in \citet[Theorem~D.8]{pooladian} with an incomplete proof.

Let independent batches $\mathbf{X}_{k+1}=(X_1,\ldots,X_{k+1})\sim\mu^{\otimes{k+1}}$
and $\mathbf{Y}_{k+1}=(Y_1,\ldots,Y_{k+1})\sim\nu^{\otimes{k+1}}$. For each $s\in\{1,\ldots,k+1\}$,
define the leave-one-out batches $\mathbf{X}_{-s}=(X_i)_{i\ne s}$ and 
$\mathbf{Y}_{-s}=(Y_j)_{j\ne s}$. Denote the corresponding empirical measures by 
$\widehat{\mu}_{-s}=\widehat{\mu}_{\mathbf{X}_{-s}}$ and $\widehat{\nu}_{-s}=\widehat{\nu}_{\mathbf{Y}_{-s}}$,
and let $\pi_{-s}$ be an OT plan between $\widehat{\mu}_{-s}$
and $\widehat{\nu}_{-s}$.
As $\overline{\pi}=\frac{1}{k+1}\sum_{s=1}^{k+1}\pi_{-s}$ averages batch OT plans between discrete measures $\widehat{\mu}_{k+1}$ and
$\widehat{\nu}_{k+1}$, we have $\overline{\pi}\in\Pi(\widehat{\mu}_{k+1},\,\widehat{\nu}_{k+1})$ by
\citet[Lemma~9]{fatras}, but we detail the first marginal for completeness: by linearity of the push-forward operator,
\[
\operatorname{pr}^1_\sharp\overline{\pi}=\frac{1}{k+1}\sum_{s=1}^{k+1} \operatorname{pr}^1_\sharp \pi_{-s}=\frac{1}{k+1}\sum_{s=1}^{k+1} \widehat{\mu}_{-s}=\widehat{\mu}_{k+1}.
\]

Therefore, by definition of the Wasserstein distance,
\begin{equation}
\begin{aligned}
W_{2}^{2}(\widehat{\mu}_{k+1},\,\widehat{\nu}_{k+1})
&\leqslant\int_{\mathbb{R}^{d}\times\mathbb{R}^{d}}\lVert x-y\rVert^{2}\,\mathrm{d}\overline{\pi}(x,y)\\
&=\frac{1}{k+1}\sum_{s=1}^{k+1}\int_{\mathbb{R}^{d}\times\mathbb{R}^{d}}\lVert x-y\rVert^{2}\,\mathrm{d}\pi_{-s}(x,y)\\
&=\frac{1}{k+1}\sum_{s=1}^{k+1}W_{2}^{2}(\widehat{\mu}_{-s},\,\widehat{\nu}_{-s}).
\end{aligned}
\label{eq:leave-one-out-cost-monotonicity}
\end{equation}
Finally, for each $s$, $(\mathbf{X}_{-s},\,\mathbf{Y}_{-s})\sim \mu^{\otimes k}\otimes\nu^{\otimes k}$,
hence $\mathbb{E}[W_{2}^{2}(\widehat{\mu}_{-s},\,\widehat{\nu}_{-s})]=\mathbb{E}[W_{2}^{2}(\widehat{\mu}_{k},\,\widehat{\nu}_{k})]$,
and taking expectations in \eqref{eq:leave-one-out-cost-monotonicity} gives
\(
\mathbb{E}[W_{2}^{2}(\widehat{\mu}_{k+1},\,\widehat{\nu}_{k+1})]\leqslant\mathbb{E}[W_{2}^{2}(\widehat{\mu}_{k},\,\widehat{\nu}_{k})].
\)

\emph{Consistency.} Let $\mu,\,\nu\in\mathcal{P}_{q}(\mathbb{R}^{d})$ with $q>2$. By \cref{app:expected-batch-ot-elementary-properties} and the definition of the Wasserstein distance,
$\mathbb{E}[W_{2}^{2}(\widehat{\mu}_{k},\,\widehat{\nu}_{k})]\geqslant W_{2}^{2}(\mu,\,\nu)$. On the other hand, by the triangle inequality,
\[
W_{2}(\widehat{\mu}_{k},\,\widehat{\nu}_{k}) \leqslant W_{2}(\widehat{\mu}_{k},\,\mu)+W_{2}(\mu,\,\nu)+W_{2}(\widehat{\nu}_{k},\,\nu).
\]
Therefore, denoting $a_k=W_{2}(\widehat{\mu}_{k},\,\mu)+W_{2}(\widehat{\nu}_{k},\,\nu)$,
\begin{equation*}
W_{2}^{2}(\widehat{\mu}_{k},\,\widehat{\nu}_{k})-W_{2}^{2}(\mu,\,\nu) \leqslant 2W_{2}(\mu,\,\nu)a_k+a_k^2.
\end{equation*}
By admissibility of $\overline\pi_k$ in~\cref{app:expected-batch-ot-elementary-properties} and applying~\cref{prop:expected-batch-ot-cost-identity},
$W_2^2(\mu,\,\nu)\leqslant \mathbb{E}[W_{2}^{2}(\widehat{\mu}_{k},\,\widehat{\nu}_{k})]$, thus
\begin{equation*}
0\leqslant 
\mathbb{E}[W_{2}^{2}(\widehat{\mu}_{k},\,\widehat{\nu}_{k})]
-W_{2}^{2}(\mu,\,\nu)
\leqslant 
2W_{2}(\mu,\,\nu)\mathbb{E}[a_k]+\mathbb{E}[a_k^2].
\end{equation*}
Now, by the Cauchy-Schwarz inequality,
\(
\mathbb{E}[a_{k}]\leqslant\mathbb{E}[W_{2}^{2}(\widehat{\mu}_{k},\,\mu)]^{1/2}+\mathbb{E}[W_{2}^{2}(\widehat{\nu}_{k},\,\nu)]^{1/2}.
\)
Moreover, as $(a+b)^{2}\leqslant2(a^{2}+b^{2})$, we have
\(
\mathbb{E}[a_{k}^{2}]\leqslant2\mathbb{E}[W_{2}^{2}(\widehat{\mu}_{k},\,\mu)]+2\mathbb{E}[W_{2}^{2}(\widehat{\nu}_{k},\,\nu)].
\)
Since $\mu,\,\nu \in \mathcal{P}_q(\mathbb{R}^d)$, by \citet[Theorem~1]{fournier2015}, both terms on the right-hand side
converge to zero as $k\to+\infty$. We conclude that
\(
\lim_{k \to +\infty}\mathbb{E}[W_{2}^{2}(\widehat{\mu}_{k},\,\widehat{\nu}_{k})] = W_{2}^{2}(\mu,\,\nu).
\)
\end{proof}

\subsection[Convergence rates for the expected batch OT cost]{Convergence rates for the expected batch OT cost (\cref{prop:expected-batch-ot-cost-convergence}(2))}
\label{app:expected-batch-ot-cost-rates}

We now prove the convergence rates in \cref{prop:expected-batch-ot-cost-convergence}(2). We first recall the dual formulation of semidiscrete OT in our setting~\citep[Section~5]{peyre}.

Assume that $\nu=\sum_{j=1}^{M}q_j\updelta_{v_j}$, where $q_j>0$ and the atoms $v_{1},\ldots,v_{M}\in\mathbb{R}^{d}$ are pairwise distinct.
If \(M=1\), the claim is immediate, hence we assume \(M\geqslant2\).
Let $(X_{1},\ldots,X_{k})$ and $(Y_{1},\ldots,Y_{k})$ be independent
batches with joint law $\mu^{\otimes k}\otimes\nu^{\otimes k}$. Denote the corresponding empirical measures by
$\widehat{\mu}_{k}=\frac{1}{k}\sum_{i=1}^{k}\updelta_{X_{i}}$ and
$\widehat{\nu}_{k}=\frac{1}{k}\sum_{i=1}^{k}\updelta_{Y_{i}}$. Writing $Y_{i}=v_{J_{i}}$, the
indices $J_{1},\ldots,J_{k}$ are independent with $\mathbb{P}(J_i=j)=q_j$, and we can write
\[
\widehat{\nu}_{k}=\sum_{j=1}^{M}\widehat{q}_{k,j}\updelta_{v_{j}}\quad\text{with}\quad\widehat{q}_{k,j}=\frac{1}{k}\sum_{i=1}^{k}\mathbbm{1}_{\{J_{i}=j\}}.
\]
For every $\lambda=(\lambda_{1},\ldots,\lambda_{M})\in\mathbb{R}^{M}$,
let $\psi_{\lambda}(x)=\min_{1\leqslant j\leqslant M}\{\lVert x-v_{j}\rVert^{2}-\lambda_{j}\}$,
and define 
\[
H(\lambda)=\int_{\mathbb{R}^{d}}\psi_{\lambda}(x)\,\mathrm{d}\mu(x)+\sum_{j=1}^{M}q_j\lambda_{j}\quad\text{and}\quad H_{k}(\lambda)=\int_{\mathbb{R}^{d}}\psi_{\lambda}(x)\,\mathrm{d}\widehat{\mu}_{k}(x)+\sum_{j=1}^{M}\widehat{q}_{k,j}\lambda_{j}.
\]
Since $H(\lambda+t\mathbf{1}_{M})=H(\lambda)$ for any $t\in\mathbb{R}$,
where $\mathbf{1}_{M}=(1,\dots,1)\in\mathbb{R}^{M}$, we work on $\Lambda=(\mathbb{R}\,\mathbf{1}_{M})^{\perp}$
to remove such invariance. Then, by the dual formulation of semidiscrete
OT \citep[Section~5]{peyre}, 
\[
W_{2}^{2}(\mu,\,\nu)=\sup_{\lambda\in\Lambda}H(\lambda)\quad\text{and}\quad W_{2}^{2}(\widehat{\mu}_{k},\,\widehat{\nu}_{k})=\sup_{\lambda\in\Lambda}H_{k}(\lambda).
\]

\begin{prop}
\label{prop:semidiscrete-two-sample-k-inverse-rate} Let $\mu\in\mathcal{P}_{2}(\mathbb{R}^{d})$
and $\nu=\sum_{i=1}^{M}q_i\updelta_{v_{i}}$ with $q_i>0$ and pairwise distinct $v_{i}\in\mathbb{R}^{d}$.
If Assumption A or B in \citet{genans} holds, then 
\(
\mathbb{E}[W_{2}^{2}(\widehat{\mu}_{k},\,\widehat{\nu}_{k})]-W_{2}^{2}(\mu,\,\nu)=O(k^{-1}).
\)
\end{prop}

The above Assumption B holds for $\mu=\mathcal{N}(0,\,I_d)$ by \citet[\S 4]{genans}, establishing
\cref{prop:expected-batch-ot-cost-convergence}(2).

Under Assumption A or B, by \citet[Proposition~4.1]{genans}, the restriction of $H$ to $\Lambda$ 
admits a unique maximizer $\lambda^\star$, $H$ is $\mathcal{C}^{2}$ in a neighborhood of $\lambda^\star$, 
and $D^{2}H(\lambda^{\star})$ is negative definite on $\Lambda$.
We also show that $\lambda\mapsto\psi_{\lambda}(x)$ is $1$-Lipschitz continuous for any $x\in\mathbb{R}^{d}$. 

\begin{lem}
\label{lem:psilip} For any $x\in\mathbb{R}^{d}$ and $\lambda,\,\lambda'\in\mathbb{R}^{M}$,
$|(\psi_{\lambda}-\psi_{\lambda'})(x)|\leqslant\lVert\lambda-\lambda'\rVert_{\infty}$.
\end{lem}

\begin{proof}
 Let $x\in\mathbb{R}^{d}$ and $\lambda,\,\lambda'\in\mathbb{R}^{M}$.
For any $\ell\in\{1,\ldots,M\}$, 
\[
\psi_{\lambda}(x) \leqslant\lVert x-v_{\ell}\rVert^{2}-\lambda_{\ell}
 =\lVert x-v_{\ell}\rVert^{2}-\lambda'_{\ell}+(\lambda'_{\ell}-\lambda_{\ell})
 \leqslant\lVert x-v_{\ell}\rVert^{2}-\lambda'_{\ell}+\lVert\lambda-\lambda'\rVert_{\infty}.
\]
Taking the minimum over $\ell$ gives $\psi_{\lambda}(x)-\psi_{\lambda'}(x)\leqslant\lVert\lambda-\lambda'\rVert_{\infty}$.
Additionally, by exchanging $\lambda$ and $\lambda'$, we obtain
$\psi_{\lambda'}(x)-\psi_{\lambda}(x)\leqslant\lVert\lambda'-\lambda\rVert_{\infty}$, hence the claim.
\end{proof}

To prove \Cref{prop:semidiscrete-two-sample-k-inverse-rate}, we
 reduce the cost bias
to a quantity defined on a deterministic compact $K\subseteq\Lambda$.

\begin{lem}
\label{lem:Zk-reduction}
Under the assumptions of \Cref{prop:semidiscrete-two-sample-k-inverse-rate},
there exists a compact $K\subseteq\Lambda$ containing $\lambda^{\star}$
with
\(
\mathbb{E}[W_{2}^{2}(\widehat{\mu}_{k},\,\widehat{\nu}_{k})]-W_{2}^{2}(\mu,\,\nu)=\mathbb{E}[Z_{k}]+o(k^{-1})
\),
where
\(
Z_{k}=\sup_{\lambda\in K}\{H_{k}(\lambda)-H_{k}(\lambda^{\star})\}.
\)
\end{lem}

\begin{proof}
Adapting the proofs of \citet[Lemma~3.1]{genans} and \citet[Lemma~E.2]{minimax},
we first show that, on an event $\mathcal{G}_{k}$ of
high probability, the empirical semi-dual $H_{k}$ admits a maximizer
in a deterministic compact $K\subseteq\Lambda$. By tightness of $\mu$, choose a
compact $K_{0}\subseteq\mathbb{R}^{d}$ such that $\mu(K_{0})\geqslant1-\frac{1}{8}\min_{1\leqslant j\leqslant M}q_j$,
and define the event $\mathcal{G}_{k}=\mathcal{G}_{k}^{s}\cap\mathcal{G}_{k}^{t}$, where
\[
\mathcal{G}_{k}^{s}=\left\{ \widehat{\mu}_{k}(K_{0})\geqslant1-\frac{1}{4}\min_{1\leqslant j\leqslant M}q_j\right\} \quad\text{and}\quad\mathcal{G}_{k}^{t}=\left\{ \min_{1\leqslant j\leqslant M}\widehat{q}_{k,\,j}\geqslant\frac{1}{2}\min_{1\leqslant j\leqslant M}q_j\right\} .
\]
By Hoeffding's inequality and the union bound, there exist $c,\,C>0$
such that $\mathbb{P}(\mathcal{G}_{k}^{c})\leqslant Ce^{-ck}$. $(\star)$

For $\lambda\in\mathbb{R}^{M}$, denote the Laguerre cells $C_{j}(\lambda)=\{x\in\mathbb{R}^{d},\,\psi_{\lambda}(x)=\lVert x-v_{j}\rVert^{2}-\lambda_{j}\}$.
By \citet[Proposition~2.4]{peyre}, $H_{k}$ admits
a maximizer $\lambda\in\mathbb{R}^{M}$. Let $\pi=(\pi_{ij})$ be
an OT plan between $\widehat{\mu}_{k}$ and $\widehat{\nu}_{k}$.
The primal-dual optimality condition
\citep[Eq.~2.23]{peyre} implies that if $\pi_{ij}>0$,
then $\psi_{\lambda}(X_{i})+\lambda_{j}=\lVert X_{i}-v_{j}\rVert^{2}$,
hence $X_{i}\in C_{j}(\lambda)$, and 
\[
\widehat{q}_{k,\,j}=\sum_{i=1}^{k}\pi_{ij}=\sum_{i=1}^{k}\pi_{ij}\mathbbm{1}_{X_{i}\in C_{j}(\lambda)}\leqslant\sum_{i=1}^{k}\frac{1}{k}\mathbbm{1}_{X_{i}\in C_{j}(\lambda)}=\widehat{\mu}_{k}(C_{j}(\lambda)).
\]
Applying the union bound to complements, 
\[
\widehat{\mu}_{k}(C_{j}(\lambda)\cap K_{0}) =1-\widehat{\mu}_{k}(C_{j}(\lambda)^{c}\cup K_{0}^{c})
 \geqslant1-(\widehat{\mu}_{k}(C_{j}(\lambda)^{c})+\widehat{\mu}_{k}(K_{0}^{c}))
 =\widehat{\mu}_{k}(C_{j}(\lambda))+\widehat{\mu}_{k}(K_{0})-1.
\]
On $\mathcal{G}_{k}$, $\widehat{\mu}_{k}(K_{0})\geqslant1-\frac{1}{4}\min_{j}q_j\geqslant1-\frac{1}{2}\min_{i}\widehat{q}_{k,\,i}$.
Therefore, we can choose $x_{j}\in C_{j}(\lambda)\cap K_{0}$ for
every $j$, as 
\[
\widehat{\mu}_{k}(C_{j}(\lambda)\cap K_{0})\geqslant\widehat{q}_{k,\,j}-\frac{1}{2}\min_{1\leqslant i\leqslant M}\widehat{q}_{k,\,i}\geqslant\frac{1}{2}\min_{1\leqslant i\leqslant M}\widehat{q}_{k,\,i}>0.
\]
Set $C_{0}=\max_{i}\sup_{x\in K_{0}}\lVert x-v_{i}\rVert^{2}$. As
$x_{j}\in C_{j}(\lambda)$, for every $i,\,j$, $\lambda_{i}-\lambda_{j}\leqslant\lVert x_{j}-v_{i}\rVert^{2}-\lVert x_{j}-v_{j}\rVert^{2}\leqslant C_{0}$,
and taking the maximum over $i,\,j$, we get $\max_{i}\lambda_{i}-\min_{j}\lambda_{j}\leqslant C_{0}$.
Since $H_{k}$ is invariant along $\mathbb{R}\mathbf{1}_{M}$, we
can assume that $\lambda\in\Lambda$, so that $\min_{j}\lambda_{j}\leqslant0\leqslant\max_{i}\lambda_{i}$,
thus $\lVert\lambda\rVert_{\infty}\leqslant C_{0}$. Therefore, on $\mathcal{G}_{k}$,
$H_{k}$ admits a maximizer $\widehat{\lambda}_{k}$ in the compact
$K=\{\lambda\in\Lambda,\;\lVert\lambda\rVert_{\infty}\leqslant C_{0}+\lVert\lambda^{\star}\rVert_{\infty}+1\}$
containing $\lambda^{\star}$, and defining $Z_{k}$ as in the statement,
\[
(W_{2}^{2}(\widehat{\mu}_{k},\,\widehat{\nu}_{k})-W_{2}^{2}(\mu,\,\nu))\mathbbm{1}_{\mathcal{G}_{k}}=(Z_{k}+(H_{k}-H)(\lambda^{\star}))\mathbbm{1}_{\mathcal{G}_{k}}.
\]
Since $\mathbb{E}[(H_{k}-H)(\lambda^{\star})]=0$, we have $\mathbb{E}[(H_{k}-H)(\lambda^{\star})\mathbbm{1}_{\mathcal{G}_{k}}]=-\mathbb{E}[(H_{k}-H)(\lambda^{\star})\mathbbm{1}_{\mathcal{G}_{k}^{c}}]$,
thus 
\begin{align*}
&\mathbb{E}[W_{2}^{2}(\widehat{\mu}_{k},\,\widehat{\nu}_{k})]-W_{2}^{2}(\mu,\,\nu) \\
&=\mathbb{E}[(W_{2}^{2}(\widehat{\mu}_{k},\,\widehat{\nu}_{k})-W_{2}^{2}(\mu,\,\nu))(\mathbbm{1}_{\mathcal{G}_{k}}+\mathbbm{1}_{\mathcal{G}_{k}^{c}})]\\
 & =\mathbb{E}[Z_{k}\mathbbm{1}_{\mathcal{G}_{k}}]+\mathbb{E}[(H_{k}-H)(\lambda^{\star})\mathbbm{1}_{\mathcal{G}_{k}}]+\mathbb{E}[(W_{2}^{2}(\widehat{\mu}_{k},\,\widehat{\nu}_{k})-W_{2}^{2}(\mu,\,\nu))\mathbbm{1}_{\mathcal{G}_{k}^{c}}]\\
 & =\mathbb{E}[Z_{k}]
 -\underbrace{\mathbb{E}[Z_{k}\mathbbm{1}_{\mathcal{G}_{k}^{c}}]}_{\mathrm{(i)}}
 -\underbrace{\mathbb{E}[(H_{k}-H)(\lambda^{\star})\mathbbm{1}_{\mathcal{G}_{k}^{c}}]}_{\mathrm{(ii)}}
 +\underbrace{\mathbb{E}[W_{2}^{2}(\widehat{\mu}_{k},\,\widehat{\nu}_{k})\mathbbm{1}_{\mathcal{G}_{k}^{c}}]}_{\mathrm{(iii)}}
 -\underbrace{W_{2}^{2}(\mu,\,\nu)\mathbb{P}(\mathcal{G}_{k}^{c})}_{\mathrm{(iv)}}.
\end{align*}
We now show that the four last terms are $o(k^{-1})$. 
\begin{enumerate}[leftmargin=*]
\item[(i)]  For each $\lambda\in K$, applying the triangle inequality and \Cref{lem:psilip}, 
\[
|H_{k}(\lambda)-H_{k}(\lambda^{\star})|\leqslant\frac{1}{k}\sum_{i=1}^{k}|(\psi_{\lambda}-\psi_{\lambda^{\star}})(X_{i})|+\sum_{j=1}^{M}\widehat{q}_{k,j}|\lambda_{j}-\lambda_{j}^{\star}|\leqslant2\lVert\lambda-\lambda^{\star}\rVert_{\infty}.
\]
Therefore $0\leqslant Z_{k}\leqslant2\sup_{\lambda\in K}\lVert\lambda-\lambda^{\star}\rVert_{\infty}<\infty$,
and $\mathbb{E}[Z_{k}\mathbbm{1}_{\mathcal{G}_{k}^{c}}]=O(\mathbb{P}(\mathcal{G}_{k}^{c}))=o(k^{-1})$.
\item[(ii)]  Since $|\psi_{\lambda^{\star}}(x)|\leqslant C(1+\lVert x\rVert^{2})$
and $(H_{k}-H)(\lambda^{\star})=\int_{\mathbb{R}^{d}}\psi_{\lambda^{\star}}\,\mathrm{d}(\widehat{\mu}_{k}-\mu)+\sum_{j=1}^{M}(\widehat{q}_{k,j}-q_j)\lambda_j^{\star}$,
\[
|H_{k}-H|(\lambda^{\star})\leqslant\frac{1}{k}\sum_{\ell=1}^{k}|\psi_{\lambda^{\star}}(X_{\ell})|+\int_{\mathbb{R}^{d}}|\psi_{\lambda^{\star}}(x)|\,\mathrm{d}\mu(x)+C\leqslant C\left(1+\frac{1}{k}\sum_{\ell=1}^{k}\lVert X_{\ell}\rVert^{2}\right).
\]
Recall that $\mathcal{G}_{k}=\mathcal{G}_{k}^{s}\cap\mathcal{G}_{k}^{t}$.
As $\mathcal{G}_{k}^{t}$ is independent of $X_{1},\ldots,X_{k}$,
we have 
\[
\mathbb{E}\left[\left(1+\frac{1}{k}\sum_{\ell=1}^{k}\lVert X_{\ell}\rVert^{2}\right)\mathbbm{1}_{(\mathcal{G}_{k}^{t})^{c}}\right]=\mathbb{E}_{X\sim\mu}[1+\lVert X\rVert^{2}]\mathbb{P}((\mathcal{G}_{k}^{t})^{c})=o(k^{-1}).
\]
Moreover, $\mathcal{G}_{k}^{s}=\{\frac{1}{k}\sum_{\ell=1}^{k}\mathbbm{1}_{X_{\ell}\not\in K_{0}}\leqslant\frac{1}{4}\min_j q_j\}$,
and $\mathbb{E}[\mathbbm{1}_{X_{\ell}\not\in K_{0}}]\leqslant\frac{1}{8}\min_j q_j$
by choice of $K_{0}$. Set \(B_m=\mathbbm{1}_{\{X_m\notin K_0\}}\). Chernoff's inequality applied conditionally
to the variables $(B_{m})_{m\neq\ell}$ gives, uniformly in $X_{\ell},$
$\mathbb{P}((\mathcal{G}_{k}^{s})^{c}\mid X_{\ell})\leqslant Ce^{-ck}$.
Therefore 
\begin{align*}
\mathbb{E}\left[\left(1+\frac{1}{k}\sum_{\ell=1}^{k}\lVert X_{\ell}\rVert^{2}\right)\mathbbm{1}_{(\mathcal{G}_{k}^{s})^{c}}\right] & =\frac{1}{k}\sum_{\ell=1}^{k}\mathbb{E}[(1+\lVert X_{\ell}\rVert^{2})\mathbb{P}((\mathcal{G}_{k}^{s})^{c}\mid X_{\ell})]\\
 & \leqslant Ce^{-ck}\mathbb{E}_{X\sim\mu}[1+\lVert X\rVert^{2}]=o(k^{-1}).
\end{align*}
Therefore $\mathbb{E}[|H_{k}-H|(\lambda^{\star})\mathbbm{1}_{\mathcal{G}_{k}^{c}}]=o(k^{-1})$.
\item[(iii)]  Using the coupling $\pi=\frac{1}{k}\sum_{\ell=1}^{k}\updelta_{(X_{\ell},\,Y_{\ell})}$,
we get that $W_{2}^{2}(\widehat{\mu}_{k},\,\widehat{\nu}_{k})\leqslant\frac{1}{k}\sum_{\ell=1}^{k}\lVert X_{\ell}-Y_{\ell}\rVert^{2}$,
where $\lVert X_{\ell}-Y_{\ell}\rVert^{2}\leqslant2\lVert X_{\ell}\rVert^{2}+2\max_{j}\lVert v_{j}\rVert^{2}$.
By the same calculation as above, $\mathbb{E}[W_{2}^{2}(\widehat{\mu}_{k},\,\widehat{\nu}_{k})\mathbbm{1}_{\mathcal{G}_{k}^{c}}]=o(k^{-1})$.
\item[(iv)]  By the Hoeffding bound $(\star)$, $W_{2}^{2}(\mu,\,\nu)\mathbb{P}(\mathcal{G}_{k}^{c})=o(k^{-1})$. 
\end{enumerate}
\end{proof}
We then control $\mathbb{E}[Z_{k}]$ by combining a peeling argument
with bounds on suprema of localized empirical processes, of the type
developed in \cite{vanderVaart1996} or \citet[Theorem~7]{boucheron}.
We can conveniently rewrite $H_{k}$ and $H$ as single empirical averages. Let
$P=\mu\otimes\sum_{j=1}^{M}q_j\updelta_j$ on $\mathcal{Z}=\mathbb{R}^{d}\times\{1,\ldots,M\}$.
Let $\zeta_{i}=(X_{i},\,J_{i})\sim P$ and $P_{k}=\frac{1}{k}\sum_{i=1}^{k}\updelta_{\zeta_{i}}$.
For $\lambda\in\mathbb{R}^{M}$, define $h_{\lambda}(x,\,j)=\psi_{\lambda}(x)+\lambda_{j}$
and $g_{\lambda}=h_{\lambda}-h_{\lambda^{\star}}$. Then $g_{\lambda^{\star}}=0$,
and 
\[
\begin{alignedat}{2}H(\lambda) & =\int_{\mathcal{Z}}h_{\lambda}\,\mathrm{d}P, & \qquad H_{k}(\lambda) & =\int_{\mathcal{Z}}h_{\lambda}\,\mathrm{d}P_{k},\\[0.5em]
H(\lambda)-H(\lambda^{\star}) & =\int_{\mathcal{Z}}g_{\lambda}\,\mathrm{d}P, & \qquad H_{k}(\lambda)-H_{k}(\lambda^{\star}) & =\int_{\mathcal{Z}}g_{\lambda}\,\mathrm{d}P_{k}.
\end{alignedat}
\]
Therefore, we can decompose the quantity defining $Z_{k}$ as 
\begin{equation}
H_{k}(\lambda)-H_{k}(\lambda^{\star})=H(\lambda)-H(\lambda^{\star})+\int_{\mathcal{Z}}g_{\lambda}\,\mathrm{d}(P_{k}-P).\label{eq:fundamental-decomp}
\end{equation}

To control the first term in \eqref{eq:fundamental-decomp}, we use the fact that, by \citet[Lemma 4.2 and Remark 4.3]{genans},
under the assumptions of \cref{prop:semidiscrete-two-sample-k-inverse-rate},
there exists $\kappa>0$ such that, for every $\lambda\in K$,
\begin{equation}\label{eq:quadratic-drift-on-K}
H(\lambda)-H(\lambda^{\star})\leqslant-\kappa\lVert\lambda-\lambda^{\star}\rVert_{\infty}^{2}.
\end{equation}

To control the second term in \eqref{eq:fundamental-decomp}, we bound $\int_{\mathcal{Z}}g_{\lambda}\,\mathrm{d}(P_{k}-P)$
locally. For $a\geqslant0$, define 
\[
S_{k}(a)=\sup_{\substack{\lambda\in\Lambda\\
\lVert\lambda-\lambda^{\star}\rVert_{\infty}\leqslant a
}
}\left|\int_{\mathcal{Z}}g_{\lambda}\,\mathrm{d}(P_{k}-P)\right|.
\]

\begin{lem} \label{lem:local-empirical-process} There exist $C>0$,
depending only on $M$, and $c>0$ such that, for every $a,\,t>0$,
\[
\forall k\in\mathbb{N}^{*},\quad\mathbb{E}[S_{k}(a)]\leqslant C\frac{a}{\sqrt{k}}\quad\text{and}\quad\mathbb{P}\left(S_{k}(a)-\mathbb{E}[S_{k}(a)]\geqslant t\right)\leqslant\exp\left(-c\,\frac{kt^{2}}{a^{2}}\right).
\]
\end{lem}

\begin{proof}
By definition of $g_{\lambda}$, for any $z=(x,j)\in\mathcal{Z}$,
\(
g_{\lambda}(x,\,j)-g_{\lambda'}(x,\,j)=\psi_{\lambda}(x)-\psi_{\lambda'}(x)+\lambda_{j}-\lambda'_{j}.
\)
By \cref{lem:psilip}, $|g_{\lambda}(x,\,j)-g_{\lambda'}(x,\,j)|\leqslant2\lVert\lambda-\lambda'\rVert_{\infty}.$
Hence, choosing $\lambda'=\lambda^{\star}$ and taking the sup over
$z=(x,j)\in\mathcal{Z}$ yields $\lVert g_{\lambda}\rVert_{\infty}\leqslant2\lVert\lambda-\lambda^{\star}\rVert_{\infty}.$
In particular, if $\lVert\lambda-\lambda^{\star}\rVert_{\infty}\leqslant a$
for $a>0$, then $\lVert g_{\lambda}\rVert_{\infty}\leqslant2a$.

For the first inequality, let $\mathscr{G}_{a}=\{g_{\lambda}:\lambda\in\Lambda,\ \lVert\lambda-\lambda^{\star}\rVert_{\infty}\leqslant a\}$.
The bound above gives that the class of functions $\mathscr{G}_{a}$
is uniformly bounded by the envelope
$z\mapsto2a$. By \citet[\S 2.14]{vanderVaart1996} (see also \citet[Lemma H.3]{minimax}),
\[
\mathbb{E}[S_{k}(a)]\leqslant\frac{C_{1}}{\sqrt{k}}\int_{0}^{C_{2}a}\sqrt{\log(2N(r,\mathscr{G}_{a},\lVert\cdot\rVert_{\infty}))}\,\mathrm{d}r,
\]
where $C_{1},\,C_{2}>0$ and $N(r,\,\mathscr{G}_{a},\,\lVert\cdot\rVert_{\infty})$
denotes the $r$-covering number in $\mathscr{G}_{a}$ with respect
to the sup norm. Since $\lambda\mapsto g_{\lambda}$ is $2$-Lipschitz
uniformly in $z$ by the bound above,
for every $r>0$, 
\[
N(r,\mathscr{G}_{a},\lVert\cdot\rVert_{\infty})\leqslant N(r/2,\,B_{\infty}(\lambda^{\star},\,a),\lVert\cdot\rVert_{\infty}).
\]
Since $N(r/2,\,B_{\infty}(\lambda^{\star},\,a),\lVert\cdot\rVert_{\infty})\leqslant(1+4a/r)^{M}$,
\begin{align*}
\mathbb{E}[S_{k}(a)] & \leqslant\frac{C_{1}}{\sqrt{k}}\int_{0}^{C_{2}a}\sqrt{\log(2(1+4a/r)^{M})}\,\mathrm{d}r\\
 & \leqslant C\frac{a}{\sqrt{k}}\left(\sqrt{\log2}+\int_{0}^{C_{2}}\sqrt{\log(1+4/s)}\,\mathrm{d}s\right),
\end{align*}
where the last line follows from $\sqrt{x+y}\leqslant\sqrt{x}+\sqrt{y}$
for $x,\,y\geqslant0$, and the change of variable $r=as$. Since
$\int_{0}^{C_{2}}\sqrt{\log(1+4/s)}\,\mathrm{d}s<+\infty$, this concludes
the first inequality.

We now prove the second inequality. By definition, we can write $S_{k}(a)=F(\zeta_{1},\ldots,\zeta_{k})$.
If one coordinate $\zeta_{i}$ is replaced by an arbitrary $\zeta_{i}'$,
then for any $\lambda$ satisfying $\lVert\lambda-\lambda^{\star}\rVert_{\infty}\leqslant a$,
\[
\left|\frac{1}{k}g_{\lambda}(\zeta_{i})-\frac{1}{k}g_{\lambda}(\zeta_{i}')\right|\leqslant\frac{1}{k}\left(|g_{\lambda}(\zeta_{i})|+|g_{\lambda}(\zeta_{i}')|\right)\leqslant\frac{4a}{k},
\]
by the bound above.
Taking the supremum over $\lambda$ gives that $F$ satisfies the
bounded-differences property with constant $4a/k$. The final result
follows from applying McDiarmid's inequality.
\end{proof}

We now have everything to control $\mathbb{E}[Z_k]$ and prove \cref{prop:semidiscrete-two-sample-k-inverse-rate} using \cref{lem:Zk-reduction}.

\begin{proof}[Proof of~\cref{prop:semidiscrete-two-sample-k-inverse-rate}]
By \cref{lem:Zk-reduction}, we only need to prove that $\mathbb{E}[Z_k]=O(k^{-1})$.
Let $D_{K}=\sup_{\lambda\in K}\lVert\lambda-\lambda^{\star}\rVert_{\infty}<\infty$.
For $\ell\geqslant0$, set $a_{\ell}=\frac{A2^{\ell}}{\sqrt{k}}$,
where $A\geqslant1$ will be chosen large enough as explained below.
Let $L_{k}$ be the smallest integer such that $a_{L_{k}}\geqslant D_{K}$.
Then, any $\lambda\in K$ either satisfies $\lVert\lambda-\lambda^{\star}\rVert_{\infty}\leqslant a_{0}$,
or belongs to one of the sets $a_{\ell}<\lVert\lambda-\lambda^{\star}\rVert_{\infty}\leqslant a_{\ell+1}$
with $0\leqslant\ell<L_{k}$. By \eqref{eq:quadratic-drift-on-K},
if $\lVert\lambda-\lambda^{\star}\rVert_{\infty}\leqslant a_{0}$,
then $H_{k}(\lambda)-H_{k}(\lambda^{\star})\leqslant S_{k}(a_{0}).$
On the set $a_{\ell}<\lVert\lambda-\lambda^{\star}\rVert_{\infty}\leqslant a_{\ell+1}$,
the same decomposition gives 
\[
H_{k}(\lambda)-H_{k}(\lambda^{\star})\leqslant-\kappa a_{\ell}^{2}+S_{k}(a_{\ell+1}).
\]
Taking the supremum over $K$, we obtain 
\[
Z_{k}\leqslant\max\Big(S_{k}(a_{0}),\,\max_{0\leqslant\ell<L_{k}}\{S_{k}(a_{\ell+1})-\kappa a_{\ell}^{2}\}\Big).
\]
Since $S_{k}(a_{0})\geqslant0$, denoting $x_{+}=\max(x,0)$, this
implies 
\begin{equation}
Z_{k}\leqslant S_{k}(a_{0})+\sum_{\ell=0}^{L_{k}-1}(S_{k}(a_{\ell+1})-\kappa a_{\ell}^{2})_{+}.\label{eq:Zk-peeling-bound}
\end{equation}
The first term is already of the right order, because \cref{lem:local-empirical-process}
gives 
\[
\mathbb{E}[S_{k}(a_{0})]\leqslant C\frac{a_{0}}{\sqrt{k}}=C\frac{A}{k}\quad\text{and}\quad\mathbb{E}[S_{k}(a_{\ell+1})]\leqslant C\frac{a_{\ell+1}}{\sqrt{k}}\leqslant C\frac{A2^{\ell}}{k}.
\]
On the other hand, $\kappa a_{\ell}^{2}=\kappa\frac{A^{2}4^{\ell}}{k}.$
Choosing $A$ large enough, depending only on the constants in the
above equation, ensures that for every $\ell\geqslant0$, $\mathbb{E}[S_{k}(a_{\ell+1})]\leqslant\frac{\kappa}{2}a_{\ell}^{2}.$
By~\cref{lem:local-empirical-process}, for every $t>0$, 
\[
\begin{aligned}\mathbb{P}(S_{k}(a_{\ell+1})-\kappa a_{\ell}^{2}>t) & \leqslant\mathbb{P}\left(S_{k}(a_{\ell+1})-\mathbb{E}[S_{k}(a_{\ell+1})]>\frac{\kappa}{2}a_{\ell}^{2}+t\right)\\
 & \leqslant\exp\left(-c\frac{k(\frac{\kappa}{2}a_{\ell}^{2}+t)^{2}}{a_{\ell+1}^{2}}\right).
\end{aligned}
\]
Using $\mathbb{E}[X_{+}]=\int_{0}^{\infty}\mathbb{P}(X>t)\,\mathrm{d}t$,
followed by the change of variables $u=\frac{\sqrt{k}}{a_{\ell+1}}\left(\frac{\kappa}{2}a_{\ell}^{2}+t\right),$
and the elementary Gaussian tail bound $\int_{x}^{\infty}e^{-cu^{2}}\mathrm{d}u\leqslant Ce^{-cx^{2}/2}$,
we get 
\[
\mathbb{E}[(S_{k}(a_{\ell+1})-\kappa a_{\ell}^{2})_{+}]\leqslant C\frac{a_{\ell+1}}{\sqrt{k}}\exp\left(-c\frac{ka_{\ell}^{4}}{a_{\ell+1}^{2}}\right).
\]
Since $a_{\ell+1}=2a_{\ell}$, $\frac{ka_{\ell}^{4}}{a_{\ell+1}^{2}}=\frac{ka_{\ell}^{2}}{4}=\frac{A^{2}4^{\ell}}{4}$
and $\frac{a_{\ell+1}}{\sqrt{k}}=\frac{2A2^{\ell}}{k}.$ Thus 
\[
\mathbb{E}[(S_{k}(a_{\ell+1})-\kappa a_{\ell}^{2})_{+}]\leqslant C\frac{A2^{\ell}}{k}\exp(-cA^{2}4^{\ell}).
\]
Summing over $\ell\in\{0,\ldots,L_{k}-1\}$ and extending the finite
sum to an infinite one, 
\[
\sum_{\ell=0}^{L_{k}-1}\mathbb{E}[(S_{k}(a_{\ell+1})-\kappa a_{\ell}^{2})_{+}]\leqslant\frac{CA}{k}\sum_{\ell=0}^{\infty}2^{\ell}\exp(-cA^{2}4^{\ell})\leqslant\frac{C}{k},
\]
because $A$ is fixed and the series is finite. Taking expectations
in \eqref{eq:Zk-peeling-bound}, we conclude that $\mathbb{E}[Z_{k}]\leqslant\frac{C}{k}.$
Substituting this into \cref{lem:Zk-reduction} proves the claim.
\end{proof}

\subsection[Convergence of the expected plan and comparison with costs]{Convergence of the expected plan and comparison with costs (\cref{prop:expected-batch-ot-plan-convergence})}
\label{app:expected-batch-ot-plan-convergence}

In this section, we prove the plan convergence results of \cref{prop:expected-batch-ot-plan-convergence} and
record the cost-to-plan estimates used to compare convergence of couplings with convergence of their transport costs.

Let $(\Omega,\mathcal{F},\mathbb{P})$
be a probability space carrying two independent i.i.d.\ sequences
$(X_{i})_{i\in\mathbb{N}^{*}}$ and $(Y_{j})_{j\in\mathbb{N}^{*}}$
with respective laws $\mu$ and $\nu$. For each $k\in\mathbb{N}^{*}$, set
$\mathbf{X}_{k}=(X_{1},\ldots,X_{k})$ and $\mathbf{Y}_{k}=(Y_{1},\ldots,Y_{k})$,
and write $\widehat{\mu}_{k}=\widehat{\mu}_{\mathbf{X}_{k}}$,
$\widehat{\nu}_{k}=\widehat{\nu}_{\mathbf{Y}_{k}}$ and
$\widehat{\pi}_{k}=\widehat{\pi}_{\mathbf{X}_{k},\,\mathbf{Y}_{k}}$. We assume,
as in \cref{prop:expected-batch-ot-plan-convergence}, that $\mu$ is absolutely
continuous, so that the OT plan between $\mu$ and $\nu$ is unique.
We denote it by $\pi^{\star}$.

\subsubsection{Weak convergence of the expected batch OT plan}

\begin{proof}[Proof of weak convergence in \cref{prop:expected-batch-ot-plan-convergence}.]
We prove weak convergence of the expected batch OT plan $\overline{\pi}_{k}$ to the unique OT plan $\pi^{\star}$ by first
showing that $\widehat{\pi}_{k}\rightharpoonup\pi^{\star}$ almost surely, and then passing
to the expected plan by dominated convergence.

\emph{Convergence of the empiricals (cf.~\citet[Proposition~D.1]{pooladian}).}
As $k\to\infty$, $\widehat{\pi}_{k}\rightharpoonup\pi^{\star}$
almost surely.
Indeed, \citet[Theorem~2.13]{bobkovledoux} give, almost surely, $W_{2}(\widehat{\mu}_{k},\,\mu)\to0$
and $W_{2}(\widehat{\nu}_{k},\,\nu)\to0$. Fix an outcome $\omega$ for
which both convergences hold. By \citet[Theorem~5.11]{santambrogio},
$\widehat{\mu}_{k}\rightharpoonup\mu$ and $\widehat{\nu}_{k}\rightharpoonup\nu$.
Moreover, the triangle inequality gives
\[
\int_{\mathbb{R}^{d}\times\mathbb{R}^{d}}\lVert x-y\rVert^{2}\,\mathrm{d}\widehat{\pi}_{k}=W_{2}^{2}(\widehat{\mu}_{k},\,\widehat{\nu}_{k})\leqslant\left(W_{2}(\widehat{\mu}_{k},\,\mu)+W_{2}(\mu,\,\nu)+W_{2}(\nu,\,\widehat{\nu}_{k})\right)^{2},
\]
so the sequence $(\widehat{\pi}_{k})$ has uniformly bounded quadratic
cost. Since each $\widehat{\pi}_{k}$ is an OT plan between $\widehat{\mu}_{k}$
and $\widehat{\nu}_{k}$, by \citet[Proposition~7.1.3]{ambrosio}, $(\widehat{\pi}_{k})$
is relatively compact in $\mathcal{P}(\mathbb{R}^{d}\times\mathbb{R}^{d})$,
and every weak limit point is an OT plan between $\mu$ and $\nu$.
As this OT plan is unique, every weak limit point equals $\pi^{\star}$,
thus $\widehat{\pi}_{k}\rightharpoonup\pi^{\star}$. Since this
holds on an event of probability one, the conclusion follows.

\emph{Convergence of the expected plan.}
We now show that as $k\to\infty$, $\overline{\pi}_{k}\rightharpoonup\pi^{\star}$.
Let $g$ be a bounded continuous map on $\mathbb{R}^{d}\times\mathbb{R}^{d}$.
Since $(\mathbf{X}_{k},\,\mathbf{Y}_{k})\sim\mu^{\otimes k}\otimes\nu^{\otimes k}$,
the transfer formula gives, for every bounded Borel map $f$ on $(\mathbb{R}^{d})^{k}\times(\mathbb{R}^{d})^{k}$,
\[
\int_{\Omega}f(\mathbf{X}_{k}(\omega),\,\mathbf{Y}_{k}(\omega))\,
\mathrm{d}\mathbb{P}(\omega)
=\int_{(\mathbb{R}^{d})^{k}\times(\mathbb{R}^{d})^{k}}
f(\mathbf{x}_{k},\,\mathbf{y}_{k})
\,\mathrm{d}(\mu^{\otimes k}\otimes\nu^{\otimes k})
(\mathbf{x}_{k},\,\mathbf{y}_{k}).
\]
Hence, by \cref{def:expected-batch-ot-plan}, using the bounded
Borel map
$f:(\mathbf{x}_{k},\,\mathbf{y}_{k})
\mapsto\int_{\mathbb{R}^{d}\times\mathbb{R}^{d}}g\,
\mathrm{d}\widehat{\pi}_{\mathbf{x}_{k},\,\mathbf{y}_{k}}$,
\[
\int_{\mathbb{R}^{d}\times\mathbb{R}^{d}}g\,\mathrm{d}\overline{\pi}_{k}
=\int_{\Omega}\left(\int_{\mathbb{R}^{d}\times\mathbb{R}^{d}}g\,
\mathrm{d}\widehat{\pi}_{k}(\omega)\right)
\,\mathrm{d}\mathbb{P}(\omega).
\]
By the convergence of empiricals, for $\mathbb{P}$-almost every $\omega$,
$\widehat{\pi}_{k}(\omega)\rightharpoonup\pi^{\star}$
as $k\to+\infty$, thus
\[
\int_{\mathbb{R}^{d}\times\mathbb{R}^{d}}g\,
\mathrm{d}\widehat{\pi}_{k}(\omega)
\xrightarrow[k\to+\infty]{}\int_{\mathbb{R}^{d}\times\mathbb{R}^{d}}g\,\mathrm{d}\pi^{\star}.
\]
Moreover,
\[
\left|\int_{\mathbb{R}^{d}\times\mathbb{R}^{d}}g\,
\mathrm{d}\widehat{\pi}_{k}(\omega)\right|
\leqslant\lVert g\rVert_{\infty}.
\]
Hence the dominated convergence theorem yields
\[
\int_{\mathbb{R}^{d}\times\mathbb{R}^{d}}g\,\mathrm{d}\overline{\pi}_{k}
\xrightarrow[k\to+\infty]{}\int_{\mathbb{R}^{d}\times\mathbb{R}^{d}}g\,\mathrm{d}\pi^{\star}.
\]
Since all measures \(\overline{\pi}_{k}\) and \(\pi^{\star}\) have marginals
\(\mu\) and \(\nu\), their second moments on
\(\mathbb{R}^{d}\times\mathbb{R}^{d}\) are equal. Thus, by
\citet[Theorem~5.11]{santambrogio}, the weak convergence just proved is
also convergence in \(W_{2}\).
\end{proof}

\subsubsection[Lower bound on plan distance]{Lower bound on plan distance (\cref{prop:expected-batch-ot-plan-convergence}(1))}
\label{app:plan-distance-lower-bound}

We now prove a lower bound that converts excess quadratic transport cost into a lower bound on the distance to the OT plan $\pi^{\star}$.
Specifically, for every $\pi\in\Pi(\mu,\,\nu)$, we show that 
\[
W_{2}(\pi,\,\pi^{\star})\geqslant\frac{1}{\sqrt{2}}\left(\left(\int_{\mathbb{R}^{d}\times\mathbb{R}^{d}}\lVert x-y\rVert^{2}\,\mathrm{d}\pi(x,\,y)\right)^{\frac{1}{2}}-W_{2}(\mu,\,\nu)\right).
\]

\begin{proof}[Proof of \cref{prop:expected-batch-ot-plan-convergence}(1)]
Let $S(x,\,y)=y-x$. The map $S$ is
$\sqrt{2}$-Lipschitz, since
\[
\lVert S(x,\,y)-S(x',\,y')\rVert\leqslant\lVert x-x'\rVert+\lVert y-y'\rVert\leqslant\sqrt{2}\lVert(x,\,y)-(x',\,y')\rVert_{\mathbb{R}^{d}\times\mathbb{R}^{d}}\,,
\]
where $\lVert(x,\,y)-(x',\,y')\rVert_{\mathbb{R}^{d}\times\mathbb{R}^{d}}^2=\lVert x-x'\rVert^{2}+\lVert y-y'\rVert^{2}$.

Let $\gamma\in\Pi(S_{\sharp}\pi,\,S_{\sharp}\pi^{\star})$, and let $(D,\,D^{\star})\sim\gamma$.
By the reverse triangle inequality for the $L_{2}$ norm,
\[
\left(\mathbb{E}[\lVert D-D^{\star}\rVert^{2}]\right)^{\frac{1}{2}}\geqslant\left|\left(\mathbb{E}[\lVert D\rVert^{2}]\right)^{\frac{1}{2}}-\left(\mathbb{E}[\lVert D^{\star}\rVert^{2}]\right)^{\frac{1}{2}}\right|.
\]
Taking the infimum over $\Pi(S_{\sharp}\pi,\,S_{\sharp}\pi^{\star})$
yields
\begin{align*}
W_{2}(S_{\sharp}\pi,\,S_{\sharp}\pi^{\star}) & \geqslant\left|\left(\mathbb{E}_{D\sim S_{\sharp}\pi}[\lVert D\rVert^{2}]\right)^{\frac{1}{2}}-\left(\mathbb{E}_{D^{\star}\sim S_{\sharp}\pi^{\star}}[\lVert D^{\star}\rVert^{2}]\right)^{\frac{1}{2}}\right|\\
 & =\left|\left(\mathbb{E}_{(X,\,Y)\sim\pi}[\lVert S(X,\,Y)\rVert^{2}]\right)^{\frac{1}{2}}-\left(\mathbb{E}_{(X,\,Y)\sim\pi^{\star}}[\lVert S(X,\,Y)\rVert^{2}]\right)^{\frac{1}{2}}\right|\\
 & =\left|\left(\int_{\mathbb{R}^{d}\times\mathbb{R}^{d}}\lVert x-y\rVert^{2}\,\mathrm{d}\pi(x,\,y)\right)^{\frac{1}{2}}-W_{2}(\mu,\,\nu)\right|.
\end{align*}
Since $\pi^{\star}$ is optimal,
\[
\left(\int_{\mathbb{R}^{d}\times\mathbb{R}^{d}}\lVert x-y\rVert^{2}\,\mathrm{d}\pi(x,\,y)\right)^{\frac{1}{2}}\geqslant W_{2}(\mu,\,\nu),
\]
and hence
\begin{equation}
W_{2}(S_{\sharp}\pi,\,S_{\sharp}\pi^{\star})\geqslant\left(\int_{\mathbb{R}^{d}\times\mathbb{R}^{d}}\lVert x-y\rVert^{2}\,\mathrm{d}\pi(x,\,y)\right)^{\frac{1}{2}}-W_{2}(\mu,\,\nu).\label{eq:plan-distance-cost-lower-bound}
\end{equation}

For any $\gamma\in\Pi(\pi,\,\pi^{\star})$, the push-forward
$(S\times S)_{\sharp}\gamma$ belongs to $\Pi(S_{\sharp}\pi,\,S_{\sharp}\pi^{\star})$.
The Lipschitz bound on $S$ gives
\[
W_{2}^{2}(S_{\sharp}\pi,\,S_{\sharp}\pi^{\star})\leqslant\mathbb{E}_{(Q,\,Q^{\star})\sim\gamma}[\lVert S(Q)-S(Q^{\star})\rVert^{2}]\leqslant2\mathbb{E}_{(Q,\,Q^{\star})\sim\gamma}[\lVert Q-Q^{\star}\rVert_{\mathbb{R}^{d}\times\mathbb{R}^{d}}^{2}].
\]
Taking the infimum over $\Pi(\pi,\,\pi^{\star})$ gives
$W_{2}(S_{\sharp}\pi,\,S_{\sharp}\pi^{\star})\leqslant\sqrt{2}\,W_{2}(\pi,\,\pi^{\star})$.
Combining this with \eqref{eq:plan-distance-cost-lower-bound} yields the result.
\end{proof}

\subsubsection[Upper bound on plan distance]{Upper bound on plan distance (\cref{prop:expected-batch-ot-plan-convergence}(2))}

We prove the converse estimate to the previous lower bound in the Gaussian-to-discrete setting: excess quadratic transport cost controls the \(W_{2}\)-distance to the semidiscrete OT plan.
Assume $\mu=\mathcal{N}(0,\,I_{d})$, $\nu=\sum_{j=1}^{M}q_j\updelta_{v_j}$ with $q_j>0$ and pairwise distinct atoms,
and denote $\pi^{\star}\in\Pi(\mu,\,\nu)$ the unique OT plan.
We show that there exists $A_{\nu}>0$ such that for any $\pi\in\Pi(\mu,\,\nu)$,
\[
W_{2}^{2}(\pi,\,\pi^{\star})\leqslant A_{\nu}\left(\int_{\mathbb{R}^{d}\times\mathbb{R}^{d}}\lVert x-y\rVert^{2}\,\mathrm{d}\pi(x,\,y)-W_{2}^{2}(\mu,\,\nu)\right)^{\frac{1}{2}}.
\]

We use the following formula for the distance from a point to a hyperplane:
\begin{lem}
\label{lem:point-hyperplane-distance}
Let $n_{H}\in\mathbb{R}^{d}\backslash\{0\}$,
$b\in\mathbb{R}$ and define, for $x\in\mathbb{R}^{d}$, $h(x)=\langle n_{H},\,x\rangle-b$.
Consider the hyperplane $H=\{x\in\mathbb{R}^{d}\,:\,h(x)=0\}$. Then
\(
\operatorname{dist}(x,\,H)=\frac{|h(x)|}{\lVert n_{H}\rVert}.
\)
\end{lem}

\begin{proof}
The orthogonal projection of $x$ onto $H$ is
$x-\frac{h(x)}{\lVert n_{H}\rVert^{2}}n_{H}$, hence the distance is
$\frac{|h(x)|}{\lVert n_{H}\rVert}$.
\end{proof}

\begin{proof}[Proof of \cref{prop:expected-batch-ot-plan-convergence}(2)]
We prove that $W_{2}^{2}(\pi,\,\pi^{\star})\leqslant A_{\nu}\,\mathrm{Exc}(\pi)^{\frac{1}{2}}$, where we denote
\[
\mathrm{Exc}(\pi)=\int_{\mathbb{R}^{d}\times\mathbb{R}^{d}}\lVert x-y\rVert^{2}\,\mathrm{d}\pi(x,\,y)-W_{2}^{2}(\mu,\,\nu).
\]
Write \(\nu=\sum_{j=1}^{M}q_j\updelta_{v_j}\) and
\(\mathcal{Y}=\{v_1,\ldots,v_M\}\). If \(M=1\), the claim is immediate, hence
we assume \(M\geqslant2\).

\paragraph{Reminders on semidiscrete OT.}

We refer the reader to \citet[Section~5]{peyre} for standard results on semidiscrete OT.
Because $\mu$ is absolutely continuous and $\nu$ is finitely supported, 
there exists a unique OT map $T^{\star}:\mathbb{R}^{d}\to\mathcal{Y}$ between $\mu$ and $\nu$. 
Specifically, there exist dual weights $\lambda_{1}^{\star},\ldots,\lambda_{M}^{\star}\in\mathbb{R}$
such that if we define $P_{j}^{\star}(x)=\lVert x-v_{j}\rVert^{2}-\lambda_{j}^{\star}$
and the Laguerre cells
\[
C_{j}=\left\{ x\in\mathbb{R}^{d},\;P_{j}^{\star}(x)\leqslant\min_{1\leqslant\ell\leqslant M}P_{\ell}^{\star}(x)\right\},
\]
then $T^{\star}(x)=v_{j}$ for \(x\in C_j\) away from the interfaces;
on the \(\mu\)-null interfaces, we use an arbitrary tie-breaking. For $j\ne\ell$,
the interface between $C_{j}$ and $C_{\ell}$ is contained in the
affine hyperplane
\(
H_{j\ell}=\{x\in\mathbb{R}^{d}\,:\,P_{j}^{\star}(x)=P_{\ell}^{\star}(x)\}.
\)
We write $\Sigma^{\star}=\cup_{1\leqslant j<\ell\leqslant M}H_{j\ell}$, which
contains the decision boundary of $T^{\star}$.

\paragraph{Reduction to misclassification.}

Now let $(X,\,Y)\sim\pi$ and define $Y^{\star}=T^{\star}(X)$, so
that $(X,\,Y^{\star})\sim\pi^{\star}$. Using this as a coupling between
$\pi$ and $\pi^{\star}$, we obtain
\[
W_{2}^{2}(\pi,\,\pi^{\star})\leqslant\mathbb{E}[\lVert(X,\,Y)-(X,\,Y^{\star})\rVert_{\mathbb{R}^{d}\times\mathbb{R}^{d}}^{2}]=\mathbb{E}[\lVert Y-Y^{\star}\rVert^{2}].
\]
Since $Y$ and $Y^{\star}$ both take values in the finite set $\mathcal{Y}$,
\(
\lVert Y-Y^{\star}\rVert^{2}\leqslant\mathrm{diam}(\mathcal{Y})^{2}\,\mathbbm{1}_{\{Y\ne Y^{\star}\}},
\)
hence
\begin{equation}
W_{2}^{2}(\pi,\,\pi^{\star})\leqslant\mathrm{diam}(\mathcal{Y})^{2}\,\mathbb{P}(Y\ne Y^{\star}).\label{eq:plan-distance-misclassification-bound}
\end{equation}

\paragraph{Excess cost is the expected dual slack.}

By the dual form of semidiscrete OT,
\[
W_{2}^{2}(\mu,\,\nu)=\int_{\mathbb{R}^{d}}\min_{1\leqslant m\leqslant M}P_{m}^{\star}(x)\,\mathrm{d}\mu(x)+\sum_{m=1}^{M}q_m\lambda_{m}^{\star}.
\]
For every $j$, define $\lambda^{\star}(v_j)=\lambda_{j}^{\star}$, and set
\(
\Delta(x,\,y)=\lVert x-y\rVert^{2}-(\min_{m}P_{m}^{\star}(x)+\lambda^{\star}(y)).
\)
By dual feasibility, $\Delta(x,\,y)\geqslant0$ for all $(x,\,y)$. Moreover,
because $\pi$ has marginals $\mu$ and $\nu$,
\begin{equation}
\mathbb{E}[\Delta(X,\,Y)]=\int_{\mathbb{R}^{d}\times\mathbb{R}^{d}}\lVert x-y\rVert^{2}\,\mathrm{d}\pi(x,\,y)-W_{2}^{2}(\mu,\,\nu)=\mathrm{Exc}(\pi).\label{eq:excess-cost-dual-slack}
\end{equation}

\paragraph{Excess cost controls misclassification away from the boundary. }

Now fix $\delta>0$ and define the $\delta$-neighborhood of the decision
boundary by
\(
\Sigma_{\delta}^{\star}=\{x\in\mathbb{R}^{d},\;\mathrm{dist}(x,\,\Sigma^{\star})\leqslant\delta\}.
\)
Take $x\in C_{j}\backslash\Sigma_{\delta}^{\star}$, let $\ell\ne j$
and consider $h_{j\ell}=P_{\ell}^{\star}-P_{j}^{\star}$. This is
an affine function of gradient $\nabla h_{j\ell}=2(v_{j}-v_{\ell})$, and
\(
\lVert\nabla h_{j\ell}\rVert=2\lVert v_{j}-v_{\ell}\rVert\geqslant2\,\mathrm{sep}(\mathcal{Y}).
\)
As $H_{j\ell}=h_{j\ell}^{-1}(\{0\})$, by \cref{lem:point-hyperplane-distance},
\[
|h_{j\ell}(x)|=\lVert\nabla h_{j\ell}\rVert\,\mathrm{dist}(x,\,H_{j\ell}).
\]
Because $x\not\in\Sigma_{\delta}^{\star}$, we have $\mathrm{dist}(x,\,H_{j\ell})>\delta$,
hence
\[
|h_{j\ell}(x)|\geqslant2\,\mathrm{sep}(\mathcal{Y})\,\delta.
\]
But $x\in C_{j}$, so $P_{j}^{\star}(x)\leqslant P_{\ell}^{\star}(x)$,
hence $h_{j\ell}(x)\geqslant0$, and
\begin{equation}
P_{\ell}^{\star}(x)-P_{j}^{\star}(x)\geqslant2\,\mathrm{sep}(\mathcal{Y})\,\delta.\label{eq:laguerre-slack-away-from-boundary}
\end{equation}
Now, since $x\in C_{j}$, $P_{j}^{\star}(x)=\min_{1\leqslant m\leqslant M}P_{m}^{\star}(x)$,
thus for $y=v_{\ell}$,
\[
\Delta(x,\,v_{\ell})=(\lVert x-v_{\ell}\rVert^{2}-\lambda_{\ell}^{\star})-\min_{1\leqslant m\leqslant M}P_{m}^{\star}(x)=P_{\ell}^{\star}(x)-P_{j}^{\star}(x).
\]
Using \eqref{eq:laguerre-slack-away-from-boundary}, we obtain $\Delta(x,\,v_{\ell})\geqslant2\,\mathrm{sep}(\mathcal{Y})\,\delta$,
\[
\Delta(X,\,Y)\,\mathbbm{1}_{\{X\not\in\Sigma_{\delta}^{\star},\,Y\ne Y^{\star}\}}\geqslant2\,\mathrm{sep}(\mathcal{Y})\,\delta\,\mathbbm{1}_{\{X\not\in\Sigma_{\delta}^{\star},\,Y\ne Y^{\star}\}}.
\]
Since $\Delta\geqslant0$ everywhere, \eqref{eq:excess-cost-dual-slack} implies
\(
\mathrm{Exc}(\pi)\geqslant2\,\mathrm{sep}(\mathcal{Y})\,\delta\,\mathbb{P}(X\not\in\Sigma_{\delta}^{\star},\,Y\ne Y^{\star}),
\)
hence
\begin{equation}
\mathbb{P}(X\not\in\Sigma_{\delta}^{\star},\,Y\ne Y^{\star})\leqslant\frac{\mathrm{Exc}(\pi)}{2\,\mathrm{sep}(\mathcal{Y})\,\delta}.\label{eq:misclassification-away-from-boundary}
\end{equation}

\paragraph{Gaussian anti-concentration near the boundary.}

Since $\Sigma^{\star}$ is the union of at most $\frac{M(M-1)}{2}$
affine hyperplanes, it suffices to bound the Gaussian mass of a $\delta$-tube
around one hyperplane $H=\{x\in\mathbb{R}^{d}\,:\,\langle n,\,x\rangle=b\}$
for $\lVert n\rVert=1$.

By \cref{lem:point-hyperplane-distance},
\(
\mathrm{dist}(x,\,H)=|\langle n,\,x\rangle-b|.
\)
If $X\sim\mathcal{N}(0,\,I_{d})$, then $\langle n,\,X\rangle\sim\mathcal{N}(0,\,1)$,
and
\[
\mathbb{P}(\mathrm{dist}(X,\,H)\leqslant\delta)=\mathbb{P}_{Z\sim\mathcal{N}(0,\,1)}(|Z-b|\leqslant\delta).
\]
The standard Gaussian density is bounded by $(2\pi)^{-\frac{1}{2}}$,
so
\[
\mathbb{P}_{Z\sim\mathcal{N}(0,\,1)}(|Z-b|\leqslant\delta)\leqslant\frac{2\delta}{\sqrt{2\pi}}=\sqrt{\frac{2}{\pi}}\delta.
\]
By the union bound,
\begin{equation}
\mathbb{P}(X\in\Sigma_{\delta}^{\star})\leqslant\alpha\delta\quad\text{with}\quad\alpha=\frac{M(M-1)}{2}\sqrt{\frac{2}{\pi}}.\label{eq:gaussian-boundary-tube}
\end{equation}

\paragraph{Conclusion.}

Combining \eqref{eq:misclassification-away-from-boundary} and \eqref{eq:gaussian-boundary-tube}, we obtain
\begin{align*}
\mathbb{P}(Y\ne Y^{\star}) & \leqslant\mathbb{P}(X\in\Sigma_{\delta}^{\star})+\mathbb{P}(X\not\in\Sigma_{\delta}^{\star},\,Y\ne Y^{\star})\\
 & \leqslant\alpha\delta+\frac{\mathrm{Exc}(\pi)}{2\,\mathrm{sep}(\mathcal{Y})\,\delta}.
\end{align*}
Minimizing the right-hand side over $\delta>0$,
\begin{equation}
\mathbb{P}(Y\ne Y^{\star})\leqslant2\sqrt{\frac{\alpha}{2\,\mathrm{sep}(\mathcal{Y})}}\mathrm{Exc}(\pi)^{\frac{1}{2}}.\label{eq:misclassification-excess-cost-bound}
\end{equation}
Finally, substituting \eqref{eq:misclassification-excess-cost-bound} into \eqref{eq:plan-distance-misclassification-bound},
$W_{2}^{2}(\pi,\,\pi^{\star})\leqslant A_{\nu}\,\mathrm{Exc}(\pi)^{\frac{1}{2}}$
with $A_{\nu}=2\,\mathrm{diam}(\mathcal{Y})^{2}\left(\frac{M(M-1)}{2\sqrt{2\pi}\,\mathrm{sep}(\mathcal{Y})}\right)^{\frac{1}{2}}$.
Taking $\pi=\overline{\pi}_{k}$ and using 
\cref{prop:expected-batch-ot-cost-identity,prop:expected-batch-ot-cost-convergence},
\[
W_{2}^{2}(\overline{\pi}_{k},\,\pi^{\star})
\leqslant
A_{\nu}\left(\mathbb{E}[W_{2}^{2}(\widehat{\mu}_{k},\,\widehat{\nu}_{k})]
-W_{2}^{2}(\mu,\,\nu)\right)^{1/2}
\leqslant Ck^{-1/2}.
\]
\end{proof}

\subsubsection{No universal cost-to-plan modulus}\label{app:no-universal-modulus}
We show that there is no universal modulus converting excess transport cost into \(W_{2}\)-closeness to the OT plan,
even when the OT plan is unique.

\begin{lem}
There is no function $\omega$ defined for $r\geqslant0$, with $\omega(r)\geqslant0$
and $\omega(r)\xrightarrow[r\to0]{}0$, such that for all pairs $\mu,\,\nu\in\mathcal{P}_{2}(\mathbb{R}^{2})$
whose OT plan $\pi^{\star}$ is unique, and for every $\pi\in\Pi(\mu,\,\nu)$,
\[
W_{2}(\pi,\,\pi^{\star})\leqslant\omega\left(\int_{\mathbb{R}^{2}\times\mathbb{R}^{2}}\lVert x-y\rVert^{2}\,\mathrm{d}\pi(x,\,y)-W_{2}^{2}(\mu,\,\nu)\right).
\]
\end{lem}

\begin{figure}[t]
  \centering
  \includegraphics[width=0.75\linewidth]{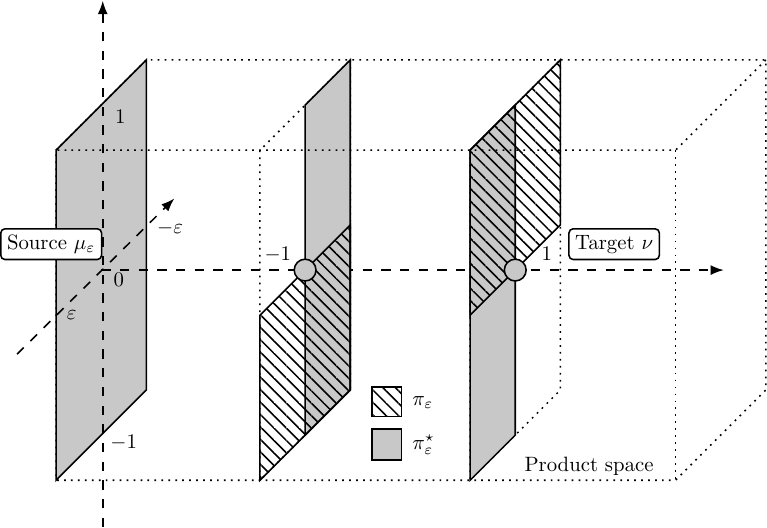}
  \caption{Couplings $\pi_{\varepsilon}$ and $\pi_{\varepsilon}^{\star}$ in $\Pi(\mu_{\varepsilon},\,\nu)$.}
\end{figure}

\begin{proof}
For $\varepsilon>0$, let $K_{\varepsilon}=[-\varepsilon,\,\varepsilon]\times[-1,\,1]$,
$\mu_{\varepsilon}=\operatorname{Unif}(K_{\varepsilon})$ and $\nu=\operatorname{Unif}\{(-1,\,0),\,(1,\,0)\}$.
The OT map $S_{\varepsilon}^{\star}$ between $\mu_{\varepsilon}$
and $\nu$ is unique by Brenier's theorem, as $\mu_{\varepsilon}$
is absolutely continuous~\citep[Theorem~1.22]{santambrogio}. Moreover,
for every $(x_{1},\,x_{2})\in K_{\varepsilon}$, we have $S_{\varepsilon}^{\star}(x_{1},\,x_{2})=(\operatorname{sign}(x_{1}),\,0)$,
as it is optimal to send $(x_{1},\,x_{2})$ to $(-1,\,0)$ if $x_{1}<0$,
and to $(1,\,0)$ otherwise. Define the map $S_{\varepsilon}:(x_{1},\,x_{2})\mapsto(\operatorname{sign}(x_{2}),\,0)$.
The values of $S_{\varepsilon}$ and $S_{\varepsilon}^{\star}$ do
not matter on the $\mu_{\varepsilon}$-negligible sets $\mathbb{R}\times\{0\}$
and $\{0\}\times\mathbb{R}$. Denote by $\pi_{\varepsilon}=(\operatorname{id},\,S_{\varepsilon})_{\sharp}\mu_{\varepsilon}$
and $\pi_{\varepsilon}^{\star}=(\operatorname{id},\,S_{\varepsilon}^{\star})_{\sharp}\mu_{\varepsilon}$
the corresponding transport plans.
Then
\[
\delta_{\varepsilon}=\int_{\mathbb{R}^{2}\times\mathbb{R}^{2}}\lVert x-y\rVert^{2}\,\mathrm{d}\pi_{\varepsilon}(x,\,y)-W_{2}^{2}(\mu_{\varepsilon},\,\nu)=\frac{1}{4\varepsilon}\int_{K_{\varepsilon}}\left(\lVert x-S_{\varepsilon}(x)\rVert^{2}-\lVert x-S_{\varepsilon}^{\star}(x)\rVert^{2}\right)\,\mathrm{d}x.
\]
Given $x=(x_{1},\,x_{2})\in\mathbb{R}^{2}$, we have $S_{\varepsilon}(x)\ne S_{\varepsilon}^{\star}(x)$
if and only if $x_{1}$ and $x_{2}$ have opposite signs, that is,
$x_{1}x_{2}<0$.
For instance, if $x_{1}<0$ and $x_{2}>0$,
\[
\lVert x-S_{\varepsilon}(x)\rVert^{2}-\lVert x-S_{\varepsilon}^{\star}(x)\rVert^{2}=\left((x_{1}-1)^{2}+x_{2}^{2}\right)-\left((x_{1}+1)^{2}+x_{2}^{2}\right)=4|x_{1}|,
\]
which also holds if $x_{1}>0$ and $x_{2}<0$ by symmetry. Therefore,
\[
\delta_{\varepsilon}=\frac{1}{4\varepsilon}\left(-\int_{-\varepsilon}^{0}\int_{0}^{1}4x_{1}\,\mathrm{d}x_{2}\,\mathrm{d}x_{1}+\int_{0}^{\varepsilon}\int_{-1}^{0}4x_{1}\,\mathrm{d}x_{2}\,\mathrm{d}x_{1}\right)=\frac{2}{\varepsilon}\int_{0}^{\varepsilon}x_{1}\,\mathrm{d}x_{1}=\varepsilon.
\]
Consider the map $\Psi:(x,\,y)\in\mathbb{R}^{2}\times\mathbb{R}^{2}\mapsto x_{2}y_{1}\in\mathbb{R}$.
On $K_{\varepsilon}\times\{(-1,\,0),\,(1,\,0)\}$, since $a+b\leqslant\sqrt{2(a^{2}+b^{2})}$,
\begin{align*}
|\Psi(x,\,y)-\Psi(x',\,y')| & =|x_{2}y_{1}-x_{2}'y_{1}'|\leqslant|x_{2}-x_{2}'|+|y_{1}-y_{1}'|\\
 & \leqslant\sqrt{2}\sqrt{\lVert x-x'\rVert^{2}+\lVert y-y'\rVert^{2}}=\sqrt{2}\lVert(x,\,y)-(x',\,y')\rVert_{\mathbb{R}^{2}\times\mathbb{R}^{2}}.
\end{align*}
Thus $\Psi$ is $\sqrt{2}$-Lipschitz, and as in the proof of
\cref{app:plan-distance-lower-bound},
\[
W_{2}(\pi_{\varepsilon},\,\pi_{\varepsilon}^{\star})\geqslant\frac{1}{\sqrt{2}}W_{2}(\Psi_{\sharp}\pi_{\varepsilon},\,\Psi_{\sharp}\pi_{\varepsilon}^{\star}).
\]

Let $(X,\,Y)\sim\pi_{\varepsilon}$. Then $Y_{1}=\operatorname{sign}(X_{2})$,
hence $\Psi(X,\,Y)=X_{2}\,\operatorname{sign}(X_{2})=|X_{2}|\sim\operatorname{Unif}[0,\,1]$,
and $\Psi_{\sharp}\pi_{\varepsilon}=\operatorname{Unif}[0,\,1]$, with
quantile function $F_{\Psi_{\sharp}\pi_{\varepsilon}}^{-}(u)=u$.
Likewise, let $(X,\,Y)\sim\pi_{\varepsilon}^{\star}$. Then $Y_{1}=\operatorname{sign}(X_{1})$
is independent of $X_{2}\sim\operatorname{Unif}[-1,\,1]$, so $\Psi(X,\,Y)=X_{2}Y_{1}\sim\operatorname{Unif}[-1,\,1]$,
hence $\Psi_{\sharp}\pi_{\varepsilon}^{\star}=\operatorname{Unif}[-1,\,1]$,
with quantile function $F_{\Psi_{\sharp}\pi_{\varepsilon}^{\star}}^{-}(u)=-1+2u$.

By \citet[Remark~2.30]{peyre}, in dimension $1$, the
2-Wasserstein distance is the $L^{2}$ distance between quantile functions:
\[
W_{2}^{2}(\Psi_{\sharp}\pi_{\varepsilon},\,\Psi_{\sharp}\pi_{\varepsilon}^{\star})=\int_{0}^{1}(\underbrace{F_{\Psi_{\sharp}\pi_{\varepsilon}}^{-}(u)}_{u}-\underbrace{F_{\Psi_{\sharp}\pi_{\varepsilon}^{\star}}^{-}(u)}_{-1+2u})^{2}\,\mathrm{d}u=\int_{0}^{1}(u-1)^{2}\,\mathrm{d}u=\frac{1}{3}.
\]
We built a family of instances $(\mu_{\varepsilon},\,\nu)$ with unique
OT plans $\pi_{\varepsilon}^{\star}$ and couplings $\pi_{\varepsilon}\in\Pi(\mu_{\varepsilon},\,\nu)$
such that $\int_{\mathbb{R}^{2}\times\mathbb{R}^{2}}\lVert x-y\rVert^{2}\,\mathrm{d}\pi_{\varepsilon}(x,\,y)-W_{2}^{2}(\mu_{\varepsilon},\,\nu)=\delta_{\varepsilon}\downarrow0$
but $W_{2}(\pi_{\varepsilon},\,\pi_{\varepsilon}^{\star})\geqslant1/\sqrt{6}$,
ruling out the existence of any such universal function $\omega$.
\end{proof}

\subsubsection{One-dimensional expected batch OT plan rate}
We show that in one dimension the expected batch OT plan converges to the monotone OT plan $\pi^{\star}$ at rate
\(O(k^{-1})\) in squared \(W_{2}\) distance, provided 
 $\mu,\,\nu\in\mathcal{P}_2(\mathbb{R})$ and one quantile map is Lipschitz.

\begin{lem}
If $\mu$ and $\nu$ have quantile functions
$F_{\mu}^{-}$ and $F_{\nu}^{-}$, and one of the quantiles is $L$-Lipschitz,
\[
W_{2}^{2}(\overline{\pi}_{k},\,\pi^{\star})\leqslant\frac{L^{2}}{3(k+1)}.
\]
\end{lem}

\begin{proof}
We give the argument for the case where $F_{\nu}^{-}$ is $L$-Lipschitz on $[0,\,1]$:
the case of $F_{\mu}^{-}$ is symmetric.

Let $(U_{1},\dots,U_{k},V_{1},\ldots,V_{k})\sim\operatorname{Unif}([0,\,1])^{\otimes2k}$
and write the corresponding order statistics
\[
U_{(1)}\leqslant\cdots\leqslant U_{(k)}\quad\text{and}\quad V_{(1)}\leqslant\cdots\leqslant V_{(k)}.
\]
Let also $I\sim\operatorname{Unif}(\{1,\ldots,k\})$ be independent of everything
else.

In one dimension, the OT plan between two empirical measures with
equal weights matches the $i$-th order statistic of one sample with
the $i$-th order statistic of the other. Hence $(X,\,Y)=(F_{\mu}^{-}(U_{(I)}),\,F_{\nu}^{-}(V_{(I)}))$
has law $\overline{\pi}_{k}$.

Now define $Y^{\star}=F_{\nu}^{-}(U_{(I)})$. Since $I$ is independent,
$U_{(I)}\sim\operatorname{Unif}([0,\,1])$. Therefore $(X,\,Y^{\star})=(F_{\mu}^{-}(U_{(I)}),\,F_{\nu}^{-}(U_{(I)}))$
has law $\pi^{\star}$.
Thus $((X,\,Y),\,(X,\,Y^{\star}))$ is a coupling of $\overline{\pi}_{k}$
and $\pi^{\star}$, and
\[
W_{2}^{2}(\overline{\pi}_{k},\,\pi^{\star})\leqslant\mathbb{E}[\lVert(X,\,Y)-(X,\,Y^{\star})\rVert_{\mathbb{R}^{2}}^{2}]=\mathbb{E}[|Y-Y^{\star}|^{2}].
\]
Since $F_{\nu}^{-}$ is $L$-Lipschitz, $|Y-Y^{\star}|\leqslant L\,|V_{(I)}-U_{(I)}|$,
and
\(
W_{2}^{2}(\overline{\pi}_{k},\,\pi^{\star})\leqslant L^{2}\,\mathbb{E}[|U_{(I)}-V_{(I)}|^{2}].
\)

Conditioning on $I$ gives
\[
\mathbb{E}[|U_{(I)}-V_{(I)}|^{2}]=\frac{1}{k}\sum_{i=1}^{k}\mathbb{E}[|U_{(i)}-V_{(i)}|^{2}].
\]
For each $i$, the random variables $U_{(i)}$ and $V_{(i)}$ are
independent and identically distributed, so
\[
\mathbb{E}[|U_{(i)}-V_{(i)}|^{2}]=2\,\operatorname{Var}(U_{(i)}).
\]
As $\operatorname{Var}(U_{(i)})=\frac{i(k+1-i)}{(k+1)^{2}(k+2)}$,
and using $\sum_{i=1}^{k}i(k+1-i)=\frac{k(k+1)(k+2)}{6}$, we obtain
\[
\mathbb{E}[|U_{(I)}-V_{(I)}|^{2}]=\frac{1}{3(k+1)}.
\]
Substituting this into the previous bound yields
\(
W_{2}^{2}(\overline{\pi}_{k},\,\pi^{\star})\leqslant\frac{L^{2}}{3(k+1)}.
\)
\end{proof}

\section[Proofs for Section 4]{Proofs for \cref{sec:flow-matching-expected-batch-ot}}\label{app:flow-matching-proofs}

This appendix follows the structure of \cref{sec:flow-matching-expected-batch-ot}.
We first prove the posterior formulas for Gaussian-to-discrete FM velocities (\cref{app:conditional-mean-posterior-formulas}).
We then show that the expected batch OT plan is rectifiable by establishing regularity of its induced velocity field and constructing the terminal flow map (\cref{app:expected-batch-ot-rectifiability}).
We next prove the posterior-concentration and one-step Euler estimates used to interpret the independent coupling (\cref{app:posterior-concentration-euler}).
Finally, we analyze the one-dimensional two-atom model and prove the asymptotic comparison between OT batch size and NFE (\cref{app:binary-asymptotics}).

\paragraph{Notation for the proofs.}
In the main text of \cref{sec:flow-matching-expected-batch-ot}, we write \((X_0,\,X_1,\,X_t)\), with
\(X_t=(1-t)X_0+tX_1\). In this appendix only, we write
\((X,\,Y,\,Z_t)\) for the same objects, with \(X=X_0\), \(Y=X_1\), and
\(Z_t=X_t=(1-t)X+tY\). Accordingly, a point \(z\) in the support of
\(Z_t\) corresponds to a point \(x\) in the support of \(X_t\) in the
main text whenever \(x\) is used there as the current time-\(t\) state.

\subsection[Conditional mean and posterior formulas]{Conditional mean and posterior formulas (\cref{lem:gaussian-discrete-velocity-posterior})}
\label{app:conditional-mean-posterior-formulas}

\begin{proof}[Proof of \cref{lem:gaussian-discrete-velocity-posterior}]
Let $(X,\,Y)\sim\pi$ and $Z_{t}=(1-t)X+tY$. Since
\(
Y-X=\frac{Y-Z_{t}}{1-t},
\)
taking conditional expectations given $Z_{t}=z$ yields
\[
u_{t}^{\pi}(z)=\mathbb{E}[Y-X\mid Z_{t}=z]=\frac{m_{t}^{\pi}(z)-z}{1-t}.
\]
The finite-support expression for $m_{t}^{\pi}$ follows by conditioning
on the values of $Y$. The posterior formula for $\pi$ is Bayes'
formula, at points where the density of \(Z_t\) is positive, applied after
the change of variables \(z=(1-t)x+tv_{j}\).
\end{proof}

\subsection[Rectifiability]{Rectifiability of the expected batch OT plan (\cref{prop:expected-batch-ot-rectifiability})}
\label{app:expected-batch-ot-rectifiability}

This section proves rectifiability of the expected batch OT plan by first expressing its FM velocity through averaged assignment probabilities, then proving the regularity needed for the pre-terminal ODE flow, and finally showing that the flow converges to a discrete terminal map.
If the target has a single atom, the claim is immediate, hence we assume \(M\geqslant2\).
Let $k\in\mathbb{N}^{*}$, $\mu=\mathcal{N}(0,\,I_{d})$ and $\nu=\sum_{j=1}^{M}q_j\updelta_{v_j}$
with $M\geqslant2$ and $q_j>0$, where the atoms $v_{1},\ldots,v_{M}$ are pairwise distinct.
Set \(q_{\min}=\min_{1\leqslant j\leqslant M}q_j\).
Let $\mathcal{Y}=\{v_{1},\ldots,v_{M}\}$ denote the support of $\nu$.

\subsubsection{Local Lipschitz regularity of the batch OT target velocity}

Draw $\mathbf{X}_{k}=(X_{1},\ldots,X_{k})\sim\mu^{\otimes k}$ and $\mathbf{Y}_{k}=(Y_{1},\ldots,Y_{k})\sim\nu^{\otimes k}$
independently, and let $I\sim\operatorname{Unif}\{1,\ldots,k\}$
be independent of everything else. Set $\widetilde{X}=X_{I}$
and $\widetilde{Y}=Y_{\widehat{\sigma}_{\mathbf{X}_{k},\,\mathbf{Y}_{k}}(I)}$, with
$\widehat{\sigma}_{\mathbf{X}_{k},\,\mathbf{Y}_{k}}$ as in \cref{app:empirical-ot-measurable-solver}.
By \cref{def:expected-batch-ot-plan},
$(\widetilde{X},\,\widetilde{Y})\sim\overline{\pi}_{k}$.

To prove that the velocity field $u_{t}^{\overline{\pi}_{k}}$ is Lipschitz, we first express it in \cref{lem:expected-batch-ot-posterior-representation} in terms of the assignment probabilities
$\bar{a}_{j}(x)=\mathbb{P}(\widetilde{Y}=v_{j}\mid \widetilde{X}=x)$.
We then show in \cref{lem:expected-batch-ot-assignment-probabilities} that the functions $\bar{a}_{j}$ are globally
Lipschitz and bounded from below, which then yields~\cref{prop:expected-batch-ot-rectifiability}(1).

For each $i\in\{1,\ldots,k\}$ and $x\in\mathbb{R}^{d}$, let $\widehat{Y}^{(i)}(x)$ be the target atom matched to the $i$-th source
position when $X_i$ is replaced by $x$, while all other source points and
all target points are kept fixed:
\[
\widehat{Y}^{(i)}(x)=
Y_{\widehat{\sigma}_{k}(X_{1},\,\ldots,\,X_{i-1},\,x,\,X_{i+1},\,\ldots,\,X_{k},\,\mathbf{Y}_{k})(i)}.
\]
For $j\in\{1,\ldots,M\}$, set
\(
a_{j}^{(i)}(x)=\mathbb{P}(\widehat{Y}^{(i)}(x)=v_{j}).
\)
Since $I$ is uniform, $\overline{a}_{j}(x)$ is obtained by averaging over the
possible positions $i$:
$\overline{a}_{j}(x)=\frac{1}{k}\sum_{i=1}^{k} a_{j}^{(i)}(x)$.

\begin{lem}
\label{lem:expected-batch-ot-posterior-representation}
Let $\varphi$ denote the standard
Gaussian density. For $t\in[0,\,1)$ and $z\in\mathbb{R}^{d}$, write
\[
x_{t,j}(z)=\frac{z-tv_{j}}{1-t}
\quad\text{and}\quad
w_{j}(t,\,z)=(1-t)^{-d}\varphi(x_{t,j}(z))\overline{a}_{j}(x_{t,j}(z)).
\]
Then the density of $\overline{\pi}_{k}$ with respect to the Lebesgue measure in $x$
and the counting measure on $\{v_{1},\ldots,v_{M}\}$ is
$
f_{\overline{\pi}_{k}}(x,\,v_{j})=\varphi(x)\overline{a}_{j}(x),
$
and
\[
m_{t}^{\overline{\pi}_{k}}(z)=\sum_{j=1}^{M} v_{j}\,
\frac{w_{j}(t,\,z)}{\sum_{\ell=1}^{M} w_{\ell}(t,\,z)},\quad
u_{t}^{\overline{\pi}_{k}}(z)=\sum_{j=1}^{M} \frac{v_{j}-z}{1-t}\,
\frac{w_{j}(t,\,z)}{\sum_{\ell=1}^{M} w_{\ell}(t,\,z)}.
\]
\end{lem}

\begin{proof}
Let $g$ be bounded and Borel. Conditioning on the random index $I$ gives
\[
\mathbb{E}[g(\widetilde{X},\,\widetilde{Y})]
=\frac{1}{k}\sum_{i=1}^{k} \mathbb{E}\bigl[g(X_{i},\widehat{Y}^{(i)}(X_{i}))\bigr].
\]
Now condition on $X_{i}=x$. Since $X_{i}\sim\mathcal{N}(0,\,I_{d})$ is
independent of $(X_{r})_{r\ne i}$ and $\mathbf{Y}_{k}$,
\[
\mathbb{E}\bigl[g(X_{i},\,\widehat{Y}^{(i)}(X_{i}))\bigr]
=
\sum_{j=1}^{M} \int_{\mathbb{R}^{d}} g(x,\,v_{j})\,\varphi(x)\,a_{j}^{(i)}(x)\,\mathrm{d}x.
\]
Averaging over $i$ yields
\[
\mathbb{E}[g(\widetilde{X},\,\widetilde{Y})]
=
\sum_{j=1}^{M} \int_{\mathbb{R}^{d}} g(x,\,v_{j})\,\varphi(x)\,\overline{a}_{j}(x)\,\mathrm{d}x.
\]
Thus $(\widetilde{X},\,\widetilde{Y})$ has density
$(x,\,v_{j})\mapsto \varphi(x)\overline{a}_{j}(x)$.

Denote $Z_{t}=(1-t)\widetilde{X}+t\widetilde{Y}$. The change of variables
$z=(1-t)x+tv_{j}$ shows that $(Z_{t},\,\widetilde{Y})$ has density
$(z,\,v_{j})\mapsto w_{j}(t,\,z)$. Hence the density of $Z_{t}$ is
$\sum_{\ell=1}^{M} w_{\ell}(t,\,z)$, and Bayes' formula gives
\[
\mathbb{P}(\widetilde{Y}=v_{j}\mid Z_{t}=z)
=
\frac{w_{j}(t,\,z)}{\sum_{\ell=1}^{M} w_{\ell}(t,\,z)}.
\]
The formulas for
$m_{t}^{\overline{\pi}_{k}}(z)=\mathbb{E}[\widetilde{Y}\mid Z_{t}=z]$ and $u_{t}^{\overline{\pi}_{k}}$ follow.
\end{proof}

\begin{lem}
\label{lem:expected-batch-ot-assignment-probabilities}
For every $j\in\{1,\ldots,M\}$,
the map $\overline{a}_{j}:\mathbb{R}^{d}\to[0,\,1]$ is globally Lipschitz:
\[
\forall x,\,x'\in\mathbb{R}^d, \quad
|\overline{a}_{j}(x)-\overline{a}_{j}(x')|
\leqslant
(k!)^{2}\frac{\operatorname{diam}(\mathcal{Y})}{\operatorname{sep}(\mathcal{Y})\sqrt{2\pi}}
\lVert x-x'\rVert.
\]
Moreover, $\overline{a}_{j}(x)\geqslant q_j^k$ for every $x\in\mathbb{R}^{d}$.
\end{lem}

\begin{proof}
If $k=1$, then $a_j^{(1)}(x)=q_j$ for every $x$, so the Lipschitz bound
and lower bound are immediate. We therefore assume $k\geqslant2$.
Fix $i\in\{1,\ldots,k\}$ and $x,\,x'\in\mathbb{R}^{d}$. Since
$
a_{j}^{(i)}(x)=\mathbb{P}(\widehat{Y}^{(i)}(x)=v_{j}),
$
we have
\begin{align*}
|a_{j}^{(i)}(x)-a_{j}^{(i)}(x')|
&=
|
\mathbb{E}[\mathbbm{1}_{\{\widehat{Y}^{(i)}(x)=v_{j}\}}
-\mathbbm{1}_{\{\widehat{Y}^{(i)}(x')=v_{j}\}}]
|\\
&\leqslant
\mathbb{E}[
|\mathbbm{1}_{\{\widehat{Y}^{(i)}(x)=v_{j}\}}
-\mathbbm{1}_{\{\widehat{Y}^{(i)}(x')=v_{j}\}}|
] \\
&\leqslant
\mathbb{P}(\widehat{Y}^{(i)}(x)\neq \widehat{Y}^{(i)}(x')).
\end{align*}
It remains to bound this probability by a constant times $\lVert x-x'\rVert$.
For each permutation $\sigma\in\mathfrak{S}_{k}$ and each $z\in \mathbb{R}^{d}$, define the cost of transporting
$(X_{1},\,\ldots,\,X_{i-1},\,z,\,X_{i+1},\,\ldots,X_{k})$ to $\mathbf{Y}_{k}$ using $\sigma$:
\[
C_{\sigma}^{(i)}(z)
=
\frac{1}{k}\bigg(
\lVert z-Y_{\sigma(i)}\rVert^{2}
+
\sum_{r\neq i}\lVert X_{r}-Y_{\sigma(r)}\rVert^{2}
\bigg).
\]
By definition, $\widehat{Y}^{(i)}(z)$ is equal to $Y_{\sigma(i)}$, where
$\sigma$ is the permutation selected by the solver, minimizing
$\sigma \mapsto C_{\sigma}^{(i)}(z)$.
Let $E_{\sigma,\,\sigma'}$ be the event that the solver selects $\sigma$
at $x$ and $\sigma'$ at $x'$, and that $\widehat{Y}^{(i)}(x)\neq \widehat{Y}^{(i)}(x')$.
Then, by the union bound,
\[
\mathbb{P}(\widehat{Y}^{(i)}(x)\neq \widehat{Y}^{(i)}(x'))
\leqslant \sum_{\sigma,\sigma'\in\mathfrak{S}_{k}}\mathbb{P}(E_{\sigma,\sigma'}).
\]
Fix a pair $\sigma,\,\sigma'\in\mathfrak{S}_{k}$. Then
\[
C_{\sigma}^{(i)}(z)-C_{\sigma'}^{(i)}(z)
=
\frac{2}{k}\langle Y_{\sigma'(i)}-Y_{\sigma(i)},\,z\rangle
-
U_{\sigma,\,\sigma'}^{(i)},
\quad\text{with}\quad 
U_{\sigma,\,\sigma'}^{(i)}
=
\frac{2}{k}\sum_{r\neq i}
\langle Y_{\sigma(r)}-Y_{\sigma'(r)},\,X_{r}\rangle.
\]
The target-norm terms cancel in this difference because \(\sigma\) and
\(\sigma'\) are permutations of the same target batch.
On the event $E_{\sigma,\,\sigma'}$, we have
$C_{\sigma}^{(i)}(x)-C_{\sigma'}^{(i)}(x)\leqslant 0$ and
$C_{\sigma}^{(i)}(x')-C_{\sigma'}^{(i)}(x')\geqslant 0$, which implies that the random variable
\(U_{\sigma,\,\sigma'}^{(i)}\)
belongs to the interval
\[
I_{\sigma,\,\sigma'}^{(i)}(x,\,x')
=
\left[
\frac{2}{k}\langle Y_{\sigma'(i)}-Y_{\sigma(i)},\,x\rangle,\,
\frac{2}{k}\langle Y_{\sigma'(i)}-Y_{\sigma(i)},\,x'\rangle
\right].
\]
Since $E_{\sigma,\sigma'}$ implies that the labels at position $i$ are
different, we have
\[
E_{\sigma,\,\sigma'}
\subseteq
\{Y_{\sigma(i)}\neq Y_{\sigma'(i)}\}
\cap
\{U_{\sigma,\,\sigma'}^{(i)}\in I_{\sigma,\,\sigma'}^{(i)}(x,\,x')\}.
\]
Therefore, conditionally on $\mathbf{Y}_{k}$,
\[
\mathbb{P}(E_{\sigma,\,\sigma'}\mid \mathbf{Y}_{k})
\leqslant
\mathbbm{1}_{\{Y_{\sigma(i)}\neq Y_{\sigma'(i)}\}}
\,
\mathbb{P}(U_{\sigma,\,\sigma'}^{(i)}\in I_{\sigma,\,\sigma'}^{(i)}(x,\,x')
\mid \mathbf{Y}_{k}).
\]

Since $\sigma,\,\sigma'$ are fixed, conditionally on $\mathbf{Y}_{k}$, the variable $U_{\sigma,\,\sigma'}^{(i)}\mid\mathbf{Y}_{k}$
is centered and Gaussian, and
\[
\operatorname{Var}(U_{\sigma,\,\sigma'}^{(i)}\mid \mathbf{Y}_{k})
=
\left(\frac{2}{k}\right)^{2}
\sum_{r\neq i}\lVert Y_{\sigma(r)}-Y_{\sigma'(r)}\rVert^{2}.
\]
Moreover, on the event $\{Y_{\sigma(i)}\neq Y_{\sigma'(i)}\}$, as the families
$(Y_{\sigma(r)})_{r\neq i}$ and $(Y_{\sigma'(r)})_{r\neq i}$ cannot coincide,
\[
\operatorname{Var}(U_{\sigma,\,\sigma'}^{(i)}\mid \mathbf{Y}_{k})
\geqslant
\left(\frac{2\,\operatorname{sep}(\mathcal{Y})}{k}\right)^{2}.
\]
Let $\Phi$ denote the standard Gaussian cumulative distribution function.
The function $\Phi$ is $1/\sqrt{2\pi}$-Lipschitz. It follows that on the event
$\{Y_{\sigma(i)}\neq Y_{\sigma'(i)}\}$,
\[
\mathbb{P}(U_{\sigma,\,\sigma'}^{(i)}\in I_{\sigma,\,\sigma'}^{(i)}(x,\,x')
\mid \mathbf{Y}_{k})
\leqslant
\frac{|I_{\sigma,\,\sigma'}^{(i)}(x,\,x')|}
{\sqrt{2\pi}\,\sqrt{\operatorname{Var}(U_{\sigma,\,\sigma'}^{(i)}\mid \mathbf{Y}_{k})}}
\leqslant
\frac{\operatorname{diam}(\mathcal{Y})}
{\operatorname{sep}(\mathcal{Y})\sqrt{2\pi}}
\,\lVert x-x'\rVert.
\]
Hence
\[
\mathbb{P}(E_{\sigma,\sigma'}\mid \mathbf{Y}_{k})
\leqslant \mathbbm{1}_{\{Y_{\sigma(i)}\neq Y_{\sigma'(i)}\}}
\frac{\operatorname{diam}(\mathcal{Y})}
{\operatorname{sep}(\mathcal{Y})\sqrt{2\pi}}
\lVert x-x'\rVert
\leqslant
\frac{\operatorname{diam}(\mathcal{Y})}
{\operatorname{sep}(\mathcal{Y})\sqrt{2\pi}}
\lVert x-x'\rVert.
\]
Since the right-hand side does not depend on $\mathbf{Y}_k$, the same upper bound holds for $\mathbb{P}(E_{\sigma,\sigma'})$, and
\[
\mathbb{P}(\widehat{Y}^{(i)}(x)\neq \widehat{Y}^{(i)}(x'))
\leqslant
(k!)^{2}
\frac{\operatorname{diam}(\mathcal{Y})}
{\operatorname{sep}(\mathcal{Y})\sqrt{2\pi}}
\,\lVert x-x'\rVert.
\]
Combining this with the first inequality of the proof shows that
$a_{j}^{(i)}$ is Lipschitz, and the same holds for $\overline{a}_{j}$ by averaging over $i$.
Finally, the event $\{Y_1=\cdots=Y_k=v_{j}\}$ has probability $q_j^k$,
and on this event one has $\widehat{Y}^{(i)}(x)=v_{j}$ for every $i$ and every $x$.
Thus $a_{j}^{(i)}(x)\geqslant q_j^k$, and therefore $\overline{a}_{j}(x)\geqslant q_j^k$.
\end{proof}

\begin{proof}[Proof of \cref{prop:expected-batch-ot-rectifiability}(1)]
As in \citet[proof of Theorem~14]{hertrich}, we derive the regularity of the FM velocity
from the regularity of the conditional posterior weights, which in our setting are the $\overline{a}_j$
controlled in \cref{lem:expected-batch-ot-assignment-probabilities}.
Let $\tau\in(0,\,1)$ and let $K\subseteq\mathbb{R}^{d}$ be compact.
We show that there exists a constant
$L>0$ such that, for every $t\in[0,\,\tau]$, the vector field
$u_{t}^{\overline{\pi}_{k}}$ is $L$-Lipschitz on $K$. Moreover, we show that
$(t,\,z)\in[0,\,1)\times K\mapsto u_{t}^{\overline{\pi}_{k}}(z)$ is
continuous.

First, observe that by \cref{lem:expected-batch-ot-assignment-probabilities}, each $\overline{a}_{j}$
is globally Lipschitz on $\mathbb{R}^d$. It is also bounded by $1$ as a probability. 
The standard Gaussian $\varphi$ is also globally bounded and globally Lipschitz on $\mathbb{R}^d$.

As products of globally Lipschitz and bounded functions, the functions $\varphi\overline{a}_{j}$ are globally Lipschitz and bounded.
Since $(t,\,z)\mapsto x_{t,j}(z)$ is also Lipschitz in $z$ with a constant uniform in $t\in[0,\,\tau]$, it follows that each function $w_{j}(t,\,z) = (1-t)^{-d} \varphi(x_{t,j}(z)) \overline{a}_{j}(x_{t,j}(z))$ is also bounded and globally Lipschitz, with Lipschitz constant uniform in $t\in[0,\,\tau]$.
All of these functions are also continuous in $(t,\,z)$, as products and compositions of continuous functions. 
Hence the sum $S(t,\,z)=\sum_{\ell=1}^{M}w_{\ell}(t,\,z)$
is also continuous on $[0,\,\tau]\times \mathbb{R}^d$, globally Lipschitz in $z$ uniformly in $t\in[0,\,\tau]$.

Now, let $K$ be a compact set. 
Since $(t,\,z)\mapsto x_{t,j}(z)$
is continuous on the compact $[0,\,\tau]\times K$, the set
\[
K'=\{x_{t,j}(z):\,t\in[0,\,\tau],\ z\in K,\ 1\leqslant j\leqslant M\}
\]
is also compact. The Gaussian $\varphi$ is bounded from below by a positive constant $c_{\varphi}>0$ on $K'$ and $\overline{a}_{j}$ is bounded from below everywhere by $(\min_j q_j)^k$.
It follows that for every $(t,\,z)\in[0,\,\tau]\times K$,
$w_{j}(t,\,z)\geqslant (1-t)^{-d}c_{\varphi}(\min_j q_j)^k>0$, and the sum $S(t,\,z)$
is also bounded from below on $[0,\,\tau]\times K$ by a positive constant.
Therefore each ratio $w_{j}/S$ is continuous on $[0,\,\tau]\times K$,
and Lipschitz in $z$ on $K$ with a constant uniform in $t\in[0,\,\tau]$.
Since, by \cref{lem:expected-batch-ot-posterior-representation},
\[
m_{t}^{\overline{\pi}_{k}}(z)=\sum_{j=1}^{M}v_{j}\,\frac{w_{j}(t,\,z)}{S(t,\,z)} \text{ and } u_{t}^{\overline{\pi}_{k}}(z)=\frac{m_{t}^{\overline{\pi}_{k}}(z)-z}{1-t}
\]
the same result holds for $(t,\,z)\mapsto m_{t}^{\overline{\pi}_{k}}(z)$ and $(t,\,z)\mapsto u_{t}^{\overline{\pi}_{k}}(z)$.
\end{proof}

Absolute continuity alone is insufficient for the regularity of the target velocity $u_{t}^{\overline{\pi}_{k}}$:
when $\mu$ has disconnected support, the velocity can even be discontinuous, as the following example shows.

\begin{example}
\label{ex:disconnected-support-discontinuity}
In dimension $1$, define $S_{+}=(\frac{1}{2},\,1)$ and $S_{-}=-S_{+}$.
Fix $p\in(0,\,1)$, let $\mu$ be the probability measure with density
$f=2p\mathbbm{1}_{S_{+}}+2(1-p)\mathbbm{1}_{S_{-}}$, and $\nu=\operatorname{Unif}\{-1,\,1\}$.
Then for every $z_{0}\in(0,\,1/7)$, setting $t_{0}=\frac{1+2z_{0}}{3}$,
the map $(t,\,z)\mapsto u_{t}^{\overline{\pi}_{1}}(z)$ is discontinuous at $(t_{0},\,z_{0})$,
even when restricted to either coordinate direction (see \cref{fig:disconnected-support-discontinuity}).
\end{example}

\begin{figure}[t]
\centering
\includegraphics[width=1\textwidth]{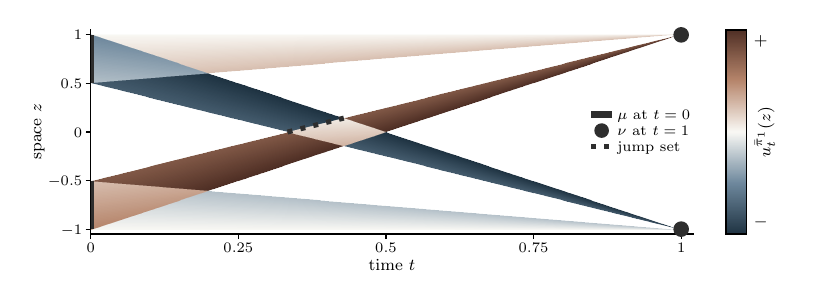}
\caption{The velocity field $u_{t}^{\overline{\pi}_{1}}$ described in \cref{ex:disconnected-support-discontinuity} exhibits jump discontinuities.}
\label{fig:disconnected-support-discontinuity}
\end{figure}

\subsubsection{Pre-terminal flow}

\begin{proof}[Proof of \cref{prop:expected-batch-ot-rectifiability}(2)]
Denote $u_{t}=u_{t}^{\overline{\pi}_{k}}$, $\rho_{t}=\rho_{t}^{\overline{\pi}_{k}}$,
and $R_{\mathcal{Y}}=\max_{1\leqslant j\leqslant M}\lVert v_{j}\rVert$.
The proof relies on combining \citet[Theorem~2.7]{pierret} and \citet[Proposition~3.2]{pierret}.
The main difference is that the velocity field $u_{t}^{\overline{\pi}_{k}}$ is locally Lipschitz
with a constant that may blow up as $t\to1$. For this reason, we work on $[0,\,\tau]$ for some $\tau<1$,
and since $\tau$ is arbitrary, this gives the conclusion on $[0,\,1)$.

\citet[Proposition~3.2]{pierret} ensures that the pair $(\rho^{\overline{\pi}_{k}},\,u^{\overline{\pi}_{k}})$
satisfies the continuity equation,  since $\mu$ and $\nu$ have finite second moments.
Moreover, for every $t\in[0,\,\tau]$ and every $z\in\mathbb{R}^{d}$,
since $m_{t}^{\overline{\pi}_{k}}(z)\in\operatorname{conv}(\mathcal{Y})$,
we have $\lVert m_{t}^{\overline{\pi}_{k}}(z)\rVert\leqslant R_{\mathcal{Y}}$. Thus
by \cref{lem:gaussian-discrete-velocity-posterior},
\begin{equation}
  \lVert u_{t}(z)\rVert=\frac{\lVert m_{t}^{\overline{\pi}_{k}}(z)-z\rVert}{1-t}\leqslant\frac{R_{\mathcal{Y}}+\lVert z\rVert}{1-\tau}.\label{eq:preterminal-linear-growth}
\end{equation}
Hence, for every compact $K\subseteq\mathbb{R}^{d}$, $\sup_{z\in K}\int_0^\tau \lVert u_t(z)\rVert \,\mathrm{d}t<+\infty$.
Moreover, the Lipschitz constant $L$ of \cref{prop:expected-batch-ot-rectifiability}(1) is integrable on $[0,\,\tau]$.
Therefore, applying \citet[Theorem 2.7]{pierret} on $[0,\,\tau]$
gives the existence and uniqueness of $\phi_t^{\overline{\pi}_{k}}$,
and the property $(\phi_{t}^{\overline{\pi}_{k}})_{\sharp}\mu=\rho_{t}$
on $[0,\,\tau]$.
The continuity of $x\mapsto\phi_{t}^{\overline{\pi}_{k}}(x)$ for every fixed $t\in[0,\,\tau]$ is a
consequence of the inequality~\eqref{eq:preterminal-linear-growth} and the Cauchy--Lipschitz theorem
\citep[Section 2.2]{pierret}.
\end{proof}

\subsubsection{Terminal flow}\label{app:terminal-flow}

We adapt the analysis of \citet{wan} to the expected batch OT setting.
Fix $\delta\in(0,\,\operatorname{sep}(\mathcal{Y})/8)$. The proof is divided into four steps. \\[-6mm]
\begin{enumerate}[leftmargin=*]
\item \emph{Uniform concentration near atoms.}
We show that the posterior mean \(m_t^{\overline{\pi}_k}(z)\) is exponentially close to \(v_j\), uniformly for \(z\in B(v_j,2\delta)\) and \(t\) close enough to \(1\).

\item \emph{Trapping near an atom.}
We show that once a trajectory enters \(B(v_j,\delta)\) sufficiently close to terminal time, it remains in \(B(v_j,2\delta)\) and converges to \(v_j\).

\item \emph{Almost-sure entry.}
We show that for \(\mu\)-almost every initial point, the trajectory enters one of the balls \(B(v_j,\delta)\) along a sequence of times \(t_n\uparrow1\), and therefore converges to the corresponding atom.

\item \emph{Terminal map and push-forward.}
We define the terminal map as the almost-sure limit of the flow and prove that it pushes \(\mu\) forward to \(\nu\).
\end{enumerate}

\paragraph{Step 1: Uniform concentration near atoms.}

Let $h=1-t$. By \cref{lem:expected-batch-ot-posterior-representation},
\[
m_{t}^{\overline{\pi}_{k}}(z)=\sum_{i=1}^{M}v_{i}\,\frac{w_{i}(t,\,z)}{\sum_{\ell=1}^{M}w_{\ell}(t,\,z)}\quad\text{and}\quad w_{i}(t,\,z)=h^{-d}\varphi\left(\frac{z-(1-h)v_{i}}{h}\right)\overline{a}_{i}\left(\frac{z-(1-h)v_{i}}{h}\right).
\]
Fix $z\in B(v_{j},\,2\delta)$ and $i\neq j$. Since $\delta<\operatorname{sep}(\mathcal{Y})/8$,
for $h$ small enough,
\begin{equation}
\lVert z-(1-h)v_{i}\rVert^{2}-\lVert z-(1-h)v_{j}\rVert^{2}\geqslant\frac{\operatorname{sep}(\mathcal{Y})^{2}}{4}.\label{eq:posterior-gaussian-ratio-gap}
\end{equation}
Using $\overline{a}_{i}\leqslant1$ and $\overline{a}_{j}\geqslant q_{\min}^{k}$ from \cref{lem:expected-batch-ot-assignment-probabilities},
it follows that
\[
\frac{w_{i}(t,\,z)}{w_{j}(t,\,z)}\leqslant q_{\min}^{-k}\exp\!\left(-\frac{\operatorname{sep}(\mathcal{Y})^{2}}{8h^{2}}\right).
\]
Summing over $i\neq j$ yields, as $\sum_{\ell=1}^{M}w_{\ell}(t,\,z)\geqslant w_{j}(t,\,z)$,
\[
1-\frac{w_{j}(t,\,z)}{\sum_{\ell=1}^{M}w_{\ell}(t,\,z)}\leqslant(M-1)q_{\min}^{-k}\exp\!\left(-\frac{\operatorname{sep}(\mathcal{Y})^{2}}{8(1-t)^{2}}\right).
\]
Hence
\[
\lVert m_{t}^{\overline{\pi}_{k}}(z)-v_{j}\rVert
\leqslant\operatorname{diam}(\mathcal{Y})\sum_{i\neq j}
\frac{w_{i}(t,\,z)}{\sum_{\ell=1}^{M}w_{\ell}(t,\,z)}
\leqslant\operatorname{diam}(\mathcal{Y})(M-1)q_{\min}^{-k}
\exp\!\left(-\frac{\operatorname{sep}(\mathcal{Y})^{2}}{8(1-t)^{2}}\right).
\]
Thus, choosing \(t_{\delta}\in(0,1)\) close enough to \(1\) so that
\eqref{eq:posterior-gaussian-ratio-gap} holds for all
\(t\in[t_{\delta},1)\), all \(z\in B(v_j,2\delta)\), all \(j\), and all
\(i\neq j\), we obtain
\[
\lVert m_{t}^{\overline{\pi}_{k}}(z)-v_{j}\rVert
\leqslant A\exp\left(-\frac{a}{(1-t)^{2}}\right),
\]
for every \(j\), every \(t\in[t_{\delta},1)\), and every
\(z\in B(v_j,2\delta)\), with
\[
A=\operatorname{diam}(\mathcal{Y})(M-1)q_{\min}^{-k},
\qquad
a=\frac{\operatorname{sep}(\mathcal{Y})^{2}}{8}.
\]

\paragraph{Step 2: Trapping near an atom.}
Increase \(t_\delta\), if necessary, so that
\[
\int_{t_{\delta}}^{1}
\frac{A}{1-s}
\exp\left(-\frac{a}{(1-s)^{2}}\right)
\,\mathrm{d}s<\delta .
\]
This preserves the conclusion of Step 1, since increasing \(t_\delta\)
only restricts the range of times under consideration.

Fix $x\in\mathbb{R}^{d}$, $j\in\{1,\ldots,M\}$ and $t_{*}\in[t_{\delta},\,1)$
such that $\lVert\phi_{t_{*}}^{\overline{\pi}_{k}}(x)-v_{j}\rVert<\delta$.
For $t\in[t_{*},\,1)$, set $z_{t}=\phi_{t}^{\overline{\pi}_{k}}(x)$ and
$y_{t}=z_{t}-v_{j}$.
With $\xi_{t}=m_{t}^{\overline{\pi}_{k}}(z_{t})-v_{j}$,
the ODE and the definition of \(\xi_t\) give 
\[
\frac{\mathrm{d}}{\mathrm{d}t}\left(\frac{y_{t}}{1-t}\right)=\frac{\xi_{t}}{(1-t)^{2}}.
\]
Hence, for every $t\in[t_{*},\,1)$, 
\begin{equation}
y_{t}=\frac{1-t}{1-t_{*}}y_{t_{*}}+(1-t)\int_{t_{*}}^{t}\frac{\xi_{s}}{(1-s)^{2}}\,\mathrm{d}s.\label{eq:yt-repr}
\end{equation}
Let $T$ be the supremum of the times $r\in[t_{*},\,1)$ such that
$\lVert y_{t}\rVert<2\delta$ for every $t\in[t_{*},\,r]$. By continuity
and the assumption at time $t_{*}$, $T>t_{*}$. For every $t<T$,
the trajectory remains in $B(v_{j},\,2\delta)$ on $[t_{*},\,t]$,
hence 
\begin{equation}
\forall s\in[t_{*},\,t],\quad\lVert\xi_{s}\rVert\leqslant A\exp\left(-\frac{a}{(1-s)^{2}}\right).\label{eq:bound-xi}
\end{equation}
Using \eqref{eq:yt-repr} and the inequality $1-t\leqslant1-s$ for
$s\leqslant t$, we obtain 
\begin{equation*}
\begin{aligned}\lVert y_{t}\rVert & \leqslant\frac{1-t}{1-t_{*}}\lVert y_{t_{*}}\rVert+(1-t)\int_{t_{*}}^{t}\frac{\lVert\xi_{s}\rVert}{(1-s)^{2}}\,\mathrm{d}s\\
 & <\delta+\int_{t_{*}}^{t}\frac{A}{1-s}\exp\!\left(-\frac{a}{(1-s)^{2}}\right)\,\mathrm{d}s\\
 & <2\delta.
\end{aligned}
\end{equation*}
Thus $\lVert y_{t}\rVert<2\delta$ for every $t\in[t_{*},\,T)$.
Letting \(t\uparrow T\) in the preceding estimate gives
\(\lVert y_T\rVert<2\delta\) if \(T<1\). Hence, by continuity, there
exists $\varepsilon\in(0,\,1-T)$ such that $\lVert y_{t}\rVert<2\delta$
for all $t\in[T,\,T+\varepsilon]$, contradicting the definition of
$T$ as a supremum. Therefore $T=1$.

It remains to prove convergence. Since $T=1$, \eqref{eq:bound-xi}
holds on $[t_{*},\,1)$. The first term in \eqref{eq:yt-repr} converges
to zero as $t\uparrow1$. For the second term, fix $\varepsilon>0$
and choose $T_{0}\in[t_{*},\,1)$ such that $\int_{T_{0}}^{1}\frac{A}{1-s}\exp\!\left(-\frac{a}{(1-s)^{2}}\right)\,\mathrm{d}s<\varepsilon$.
Then, for $t\geqslant T_{0}$, 
\[
\begin{aligned}(1-t)\int_{t_{*}}^{t}\frac{\lVert\xi_{s}\rVert}{(1-s)^{2}}\,\mathrm{d}s & \leqslant(1-t)\int_{t_{*}}^{T_{0}}\frac{\lVert\xi_{s}\rVert}{(1-s)^{2}}\,\mathrm{d}s\\
 & \quad+\int_{T_{0}}^{t}\frac{A}{1-s}\exp\!\left(-\frac{a}{(1-s)^{2}}\right)\,\mathrm{d}s.
\end{aligned}
\]
The first term tends to zero as $t\uparrow1$, while the second term
is bounded by $\varepsilon$. Since $\varepsilon>0$ was arbitrary,
the integral term in the representation of $y_{t}$ also tends to
zero. Hence $y_{t}\to0$, that is, $\phi_{t}^{\overline{\pi}_{k}}(x)\to v_{j}$
as $t\uparrow1$.

\paragraph{Step 3: Almost-sure entrance.}
Let $\mathcal{Y}_{\delta}
=\{z\in\mathbb{R}^{d}:\operatorname{dist}(z,\mathcal{Y})<\delta\}$.
Since $Z_t=(1-t)X+tY\to Y\in\mathcal{Y}$ almost surely as
$t\uparrow1$, we have
\(
\rho_t^{\overline{\pi}_{k}}(\mathcal{Y}_{\delta}^{c})\to0.
\)
Choose a sequence $t_n\uparrow1$ with $t_n\geqslant t_\delta$ and
\(
\rho_{t_n}^{\overline{\pi}_{k}}(\mathcal{Y}_{\delta}^{c})\leqslant 2^{-n}.
\)
By \cref{prop:expected-batch-ot-rectifiability}(2),
$(\phi_{t_n}^{\overline{\pi}_{k}})_\sharp\mu=\rho_{t_n}^{\overline{\pi}_{k}}$.
Hence
\[
\sum_{n\geqslant1}
\mu\!\left(\{x:\phi_{t_n}^{\overline{\pi}_{k}}(x)\notin\mathcal{Y}_{\delta}\}\right)
=
\sum_{n\geqslant1}
\rho_{t_n}^{\overline{\pi}_{k}}(\mathcal{Y}_{\delta}^{c})
<\infty.
\]
By the Borel--Cantelli lemma, for $\mu$-almost every $x$ there exists
$N(x)$ such that
\(
\phi_{t_n}^{\overline{\pi}_{k}}(x)\in \mathcal{Y}_{\delta}
\) 
for all $n\geqslant N(x)$.
Fix such an $x$, and choose an index $n_0\geqslant N(x)$.
Since $\delta<\operatorname{sep}(\mathcal Y)/8$, the balls
$B(v_1,\delta),\ldots,B(v_M,\delta)$ are pairwise disjoint.
Thus there is a unique $j(x)\in\{1,\ldots,M\}$ such that
\(
\phi_{t_{n_0}}^{\overline{\pi}_{k}}(x)\in B(v_{j(x)},\delta).
\)
Applying step 2 with
$t_*=t_{n_0}$ and $j=j(x)$ yields
\(
\phi_t^{\overline{\pi}_{k}}(x)\to v_{j(x)}
\) 
as $t\uparrow1$.
Consequently,
\(
\phi_{t_n}^{\overline{\pi}_{k}}(x)\in B(v_{j(x)},\delta)
\)
for all sufficiently large $n$.

\paragraph{Step 4: Terminal map and push-forward.}
By step 3,
for $\mu$-almost every $x$ the limit $\phi^{\overline{\pi}_{k}}(x)=\lim_{t\uparrow1}\phi_{t}^{\overline{\pi}_{k}}(x)$
exists and belongs to $\mathcal{Y}$. Choose a sequence $t_{n}\uparrow1$
as in step 3. Since each $\phi_{t_{n}}^{\overline{\pi}_{k}}$
is continuous by \cref{prop:expected-batch-ot-rectifiability}(2), the set
\[
E=\{x\in\mathbb{R}^{d}\,:\,\lim_{n\to\infty}\phi_{t_{n}}^{\overline{\pi}_{k}}(x)\text{ exists and belongs to }\mathcal{Y}\}
\]
is Borel, and the limit map is Borel on $E$. Since $\mu(E)=1$, extending
it by the fixed value $v_1$ on $E^{c}$ yields a Borel map $\phi^{\overline{\pi}_{k}}:\mathbb{R}^{d}\to\mathcal{Y}$.
Let $f$ be bounded and continuous. By dominated convergence,
\[
\int_{\mathbb{R}^{d}}f(\phi^{\overline{\pi}_{k}}(x))\,\mathrm{d}\mu(x)=\lim_{n\to\infty}\int_{\mathbb{R}^{d}}f(\phi_{t_{n}}^{\overline{\pi}_{k}}(x))\,\mathrm{d}\mu(x).
\]
Using \cref{prop:expected-batch-ot-rectifiability}(2),
\[
\int_{\mathbb{R}^{d}}f(\phi_{t_{n}}^{\overline{\pi}_{k}}(x))\,\mathrm{d}\mu(x)=\int_{\mathbb{R}^{d}}f(z)\,\mathrm{d}\rho_{t_{n}}^{\overline{\pi}_{k}}(z).
\]
Since $\rho_{t}^{\overline{\pi}_{k}}\rightharpoonup\nu$ as $t\uparrow1$,
it follows that
\[
\int_{\mathbb{R}^{d}}f(\phi^{\overline{\pi}_{k}}(x))\,\mathrm{d}\mu(x)=\int_{\mathbb{R}^{d}}f(y)\,\mathrm{d}\nu(y).
\]
Therefore $(\phi^{\overline{\pi}_{k}})_{\sharp}\mu=\nu$.
This proves \cref{prop:expected-batch-ot-rectifiability}(3).

\subsection[Rates for the terminal flow map]{Rates for the terminal flow map (\cref{cor:terminal-flow-map-rates})}\label{app:terminal-flow-map-rates}

\begin{proof}[Proof of \cref{cor:terminal-flow-map-rates}]
By \cref{prop:expected-batch-ot-rectifiability},
\((\phi^{\overline{\pi}_{k}})_{\sharp}\mu=\nu\), thus
\(
W_{2}^{2}(\mu,\,\nu)
\leqslant
\int_{\mathbb{R}^{d}}
\lVert\phi^{\overline{\pi}_{k}}(x)-x\rVert^{2}\,\mathrm{d}\mu(x).
\)
We upper bound the previous quantity using \citet[Theorem~D.5]{liu2022flow}: we derive their argument
in our present setting for completeness.
Fix \(\tau\in(0,\,1)\). By the Cauchy--Schwarz inequality,
\[
\lVert\phi_{\tau}^{\overline{\pi}_{k}}(x)-x\rVert^{2}
\leqslant
\tau\int_{0}^{\tau}
\lVert u_{t}^{\overline{\pi}_{k}}(\phi_{t}^{\overline{\pi}_{k}}(x))\rVert^{2}\,\mathrm{d}t .
\]
Since \((\phi_{t}^{\overline{\pi}_{k}})_{\sharp}\mu=\rho_{t}^{\overline{\pi}_{k}}\), integrating over \(x\) yields
\[
\int_{\mathbb{R}^d}\lVert\phi_{\tau}^{\overline{\pi}_{k}}(x)-x\rVert^{2}\,\mathrm{d}\mu(x)
\leqslant
\tau\int_{0}^{\tau}
\int_{\mathbb{R}^d}\lVert u_{t}^{\overline{\pi}_{k}}(z)\rVert^{2}\,\mathrm{d}\rho_{t}^{\overline{\pi}_{k}}(z)\,\mathrm{d}t .
\]
For each \(t<1\), Jensen's inequality applied to
\(u_{t}^{\overline{\pi}_{k}}(Z_t)=\mathbb{E}_{(X,\,Y)\sim\overline{\pi}_{k}}[Y-X\mid Z_t]\) gives
\[
\int_{\mathbb{R}^d}\lVert u_{t}^{\overline{\pi}_{k}}(z)\rVert^{2}\,\mathrm{d}\rho_{t}^{\overline{\pi}_{k}}(z)
\leqslant
\int_{\mathbb{R}^{d}\times\mathbb{R}^{d}}
\lVert y-x\rVert^{2}\,\mathrm{d}\overline{\pi}_{k}(x,\,y).
\]
Consequently,
\[
\int_{\mathbb{R}^d}\lVert\phi_{\tau}^{\overline{\pi}_{k}}(x)-x\rVert^{2}\,\mathrm{d}\mu(x)
\leqslant
\tau^2
\int_{\mathbb{R}^{d}\times\mathbb{R}^{d}}
\lVert x-y\rVert^{2}\,\mathrm{d}\overline{\pi}_{k}(x,\,y)
\leqslant
\int_{\mathbb{R}^{d}\times\mathbb{R}^{d}}
\lVert x-y\rVert^{2}\,\mathrm{d}\overline{\pi}_{k}(x,\,y).
\]
Letting \(\tau\uparrow1\), by Fatou's lemma and 
\cref{prop:expected-batch-ot-rectifiability}(3),
\[
\int_{\mathbb{R}^{d}}
\lVert\phi^{\overline{\pi}_{k}}(x)-x\rVert^{2}\,\mathrm{d}\mu(x)
\leqslant
\int_{\mathbb{R}^{d}\times\mathbb{R}^{d}}
\lVert x-y\rVert^{2}\,\mathrm{d}\overline{\pi}_{k}(x,\,y).
\]
Using \cref{prop:expected-batch-ot-cost-identity,prop:expected-batch-ot-cost-convergence},
the first rate follows.
Now let \(X\sim\mu\) and apply
\eqref{eq:misclassification-excess-cost-bound} to
\(
\pi
=
\operatorname{Law}\bigl(X,\,\phi^{\overline{\pi}_{k}}(X)\bigr)
\in\Pi(\mu,\,\nu)
\):
there exists a constant
\(C>0\), independent of \(k\), such that
\[
\mu\bigl(\{x\in\mathbb{R}^{d},\;
\phi^{\overline{\pi}_{k}}(x)\neq T^{\star}(x)\}\bigr)
\leqslant
C
\left(
\int_{\mathbb{R}^{d}}
\lVert\phi^{\overline{\pi}_{k}}(x)-x\rVert^{2}\,\mathrm{d}\mu(x)
-
W_{2}^{2}(\mu,\nu)
\right)^{1/2}
=
O(k^{-1/2}).
\]
Finally, for every \(j\),
\(
\Gamma_{k,\,j}\triangle C_{j}^{\star}
\subseteq
\{x\in\mathbb{R}^{d},\;
\phi^{\overline{\pi}_{k}}(x)\neq T^{\star}(x)\},
\)
hence
\(
\mu(\Gamma_{k,\,j}\triangle C_{j}^{\star})=O(k^{-1/2}).
\)
\end{proof}

\subsubsection{Non-monotonicity of the transport cost of the terminal map}
\label{app:four-point-nonmonotonicity}

We build an instance in which the transport cost $\int_{\mathbb{R}^d} \lVert \phi^{\overline\pi_k}(x) -x\rVert^2 \mathrm{d}\mu(x)$
is non-monotonic in $k$. We use the below result, which shows that $\phi^{\mu\otimes\nu}$ is the OT map
when $\mu$ and $\nu$ can be factored as the independent product of $d$ one-dimensional measures.
This extends \citet[Theorem~3(ii)]{hertrich}, which tackles non-degenerate Gaussian $\mu$ and $\nu$
with jointly diagonalizable covariances.

\begin{prop}
Let $\mu=\alpha_{1}\otimes\cdots\otimes\alpha_{d}$ and
$\nu=\beta_{1}\otimes\cdots\otimes\beta_{d}$ with
$\alpha_{i},\,\beta_{i}\in\mathcal{P}_{2}(\mathbb{R})$ for
$1\leqslant i\leqslant d$. Assume that the independent FM flow between
$\mu$ and $\nu$ exists, is unique on $[0,\,1)$, and admits a terminal map
$\phi^{\mu\otimes\nu}$ pushing $\mu$ to $\nu$. Then
\[
\int_{\mathbb{R}^{d}}\lVert\phi^{\mu\otimes\nu}(x)-x\rVert^{2}\,\mathrm{d}\mu(x)
=W_{2}^{2}(\mu,\,\nu).
\]
\end{prop}

\begin{proof}
Let $(X,\,Y)=(X_{1},\ldots,\,X_{d},\,Y_{1},\ldots,\,Y_{d})\sim\mu\otimes\nu$
and denote by $(e_{i})_{1\leqslant i\leqslant d}$ the canonical basis
of $\mathbb{R}^{d}$. Let $Z_{t}=(1-t)X+tY$, so that
$Z_{t,i}=(1-t)X_{i}+tY_{i}$. For fixed $t\in[0,\,1)$ and
$\mathcal{L}(Z_{t})$-almost every
$z=(z_{1},\ldots,z_{d})\in\mathbb{R}^{d}$, by \cref{lem:gaussian-discrete-velocity-posterior},
the velocity is
\[
u_{t}^{\mu\otimes\nu}(z)
=\frac{m_{t}^{\mu\otimes\nu}(z)-z}{1-t},
\]
where, by independence of the pairs $(X_{i},\,Y_{i})$,
\begin{align*}
m_{t}^{\mu\otimes\nu}(z)
 & =\mathbb{E}\!\left[Y \mid Z_{t}=z\right]
 =\sum_{i=1}^{d}\mathbb{E}\!\left[Y_{i}\mid Z_{t}=z\right]e_{i}
 =\sum_{i=1}^{d}\mathbb{E}\!\left[Y_{i}\mid Z_{t,i}=z_{i}\right]e_{i}\\
 & =\bigl(m_{t}^{\alpha_{1}\otimes\beta_{1}}(z_{1}),\,\ldots,\,
          m_{t}^{\alpha_{d}\otimes\beta_{d}}(z_{d})\bigr).
\end{align*}
Thus the $d$-dimensional ODE decouples into $d$ one-dimensional ODEs.
Writing $\phi_{t}$ for the corresponding flow maps, uniqueness gives
\[
\phi_{t}^{\mu\otimes\nu}(x_{1},\,\ldots,\,x_{d})
=\bigl(\phi_{t}^{\alpha_{1}\otimes\beta_{1}}(x_{1}),\,\ldots,\,
       \phi_{t}^{\alpha_{d}\otimes\beta_{d}}(x_{d})\bigr)
\]
for $\mu$-almost every $x$ and every $0\leqslant t<1$. Passing to the
terminal limit and using Fubini yields the same identity for the terminal
maps, $\mu$-almost everywhere. Moreover, taking coordinate marginals in
$\phi^{\mu\otimes\nu}_{\#}\mu=\nu$ shows that
$(\phi^{\alpha_{i}\otimes\beta_{i}})_{\#}\alpha_{i}=\beta_{i}$ for every
$1\leqslant i\leqslant d$. By \citet[Note~4.2]{pierret}, in one dimension,
the independent FM map is the OT map, hence
$\phi^{\alpha_{i}\otimes\beta_{i}}=T_{\alpha_{i}\to\beta_{i}}^{\star}$
for every $1\leqslant i\leqslant d$, and
\begin{equation*}
\int_{\mathbb{R}^{d}}\lVert\phi^{\mu\otimes\nu}(x)-x\rVert^{2}\,\mathrm{d}\mu(x)
=\sum_{i=1}^{d}\int_{\mathbb{R}}
  |\phi^{\alpha_{i}\otimes\beta_{i}}(x_{i})-x_{i}|^{2}
  \,\mathrm{d}\alpha_{i}(x_{i})
=\sum_{i=1}^{d}W_{2}^{2}(\alpha_{i},\,\beta_{i}).
\end{equation*}
Moreover, if $\pi\in\Pi(\mu,\,\nu)$, then its $(x_{i},\,y_{i})$-marginal
belongs to $\Pi(\alpha_{i},\,\beta_{i})$, and
\begin{equation*}
\int_{\mathbb{R}^{d}\times\mathbb{R}^{d}}\lVert x-y\rVert^{2}\,\mathrm{d}\pi(x,\,y)
=\sum_{i=1}^{d}\int_{\mathbb{R}^{d}\times\mathbb{R}^{d}}
  |x_{i}-y_{i}|^{2}\,\mathrm{d}\pi(x,\,y)
\geqslant\sum_{i=1}^{d}W_{2}^{2}(\alpha_{i},\,\beta_{i}).
\end{equation*}
Taking the infimum over $\pi\in\Pi(\mu,\,\nu)$ gives
\[
W_{2}^{2}(\mu,\,\nu)
\geqslant\sum_{i=1}^{d}W_{2}^{2}(\alpha_{i},\,\beta_{i}).
\]
Conversely, if $\pi_{i}^{\star}\in\Pi(\alpha_{i},\,\beta_{i})$ is an
optimal coupling for every $1\leqslant i\leqslant d$, then
$\pi_{1}^{\star}\otimes\cdots\otimes\pi_{d}^{\star}\in\Pi(\mu,\,\nu)$
under the canonical identification
$(\mathbb{R}\times\mathbb{R})^{d}\simeq\mathbb{R}^{d}\times\mathbb{R}^{d}$,
and therefore
\[
W_{2}^{2}(\mu,\,\nu)
\leqslant\sum_{i=1}^{d}W_{2}^{2}(\alpha_{i},\,\beta_{i}).
\]
It follows that
\[
\int_{\mathbb{R}^{d}}\lVert\phi^{\mu\otimes\nu}(x)-x\rVert^{2}\,\mathrm{d}\mu(x)
=W_{2}^{2}(\mu,\,\nu).
\]
\end{proof}

We then present in~\cref{fig:four-point-nonmonotonicity-cells} 
a numerical counterexample to the monotonicity of
$\int_{\mathbb{R}^d} \lVert \phi^{\overline\pi_k}(x)-x\rVert^2\,\mathrm{d}\mu(x)$.
Let $\mu=\mathcal{N}(0,\,I_2)$ and
\(
\nu=
(\tfrac12\updelta_{-1}+\tfrac12\updelta_1)
\otimes
(\tfrac{17}{20}\updelta_{-1}+\tfrac3{20}\updelta_1).
\)
By the preceding proposition,
$\phi^{\overline\pi_1}=\phi^{\mu\otimes\nu}$ is the OT map, unique by Brenier's theorem.
In contrast, we visually remark that $\phi^{\overline\pi_2}$ is suboptimal.

\begin{figure}[t]
\centering
\includegraphics[width=\linewidth]{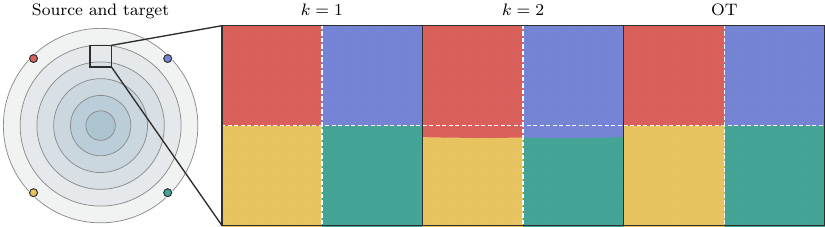}
\caption{Terminal cells for the counterexample to the monotonicity of the transport cost of $\phi^{\overline\pi_k}$.}
\label{fig:four-point-nonmonotonicity-cells}
\end{figure}

\subsubsection{Persistent unbounded rays in the four-point example}
\label{app:four-point-nonuniformity}

The convergence in \cref{cor:terminal-flow-map-rates} is a convergence in
\(\mu\)-measure, not a uniform 
convergence of the cells. We illustrate this by adapting the four-point
counterexample of \citet[Section~5.1]{pierret-tessellations} to 
batch OT (see~\cref{fig:four-point-flow-cells}). Consider \(\nu=\frac14\sum_{i=1}^4\delta_{v_i}\in\mathcal{P}(\mathbb{R}^2)\) supported on the four atoms
\[
v_1=\left(-\frac12,\,\frac{\sqrt3}{2}\right),\quad
v_2=\left(-\frac12,\,-\frac{\sqrt3}{2}\right),\quad
v_3=(1,\,0),\quad
v_4=(0,\,0).
\]
For every 
\(k\in\mathbb N^\ast\), the central FM cell \(\Gamma_{k,\,4}\)
contains the three rays
\[
\left\{
\left(-c\cos\left(\frac{2q\pi}{3}\right),\,
      -c\sin\left(\frac{2q\pi}{3}\right)\right),\;c\geqslant0
\right\}
\quad\text{for}\; q\in\{1,\,2,\,3\}.
\]
In particular, \(\Gamma_{k,\,4}\) has an unbounded connected component for
every finite \(k\), whereas the corresponding Laguerre cell \(C_4^\star\) is
bounded, as recalled in \citet[Section~5.1]{pierret-tessellations}.

\begin{figure}[t]
\centering
\begin{subfigure}{0.235\textwidth}
\centering
\includegraphics[width=\linewidth]{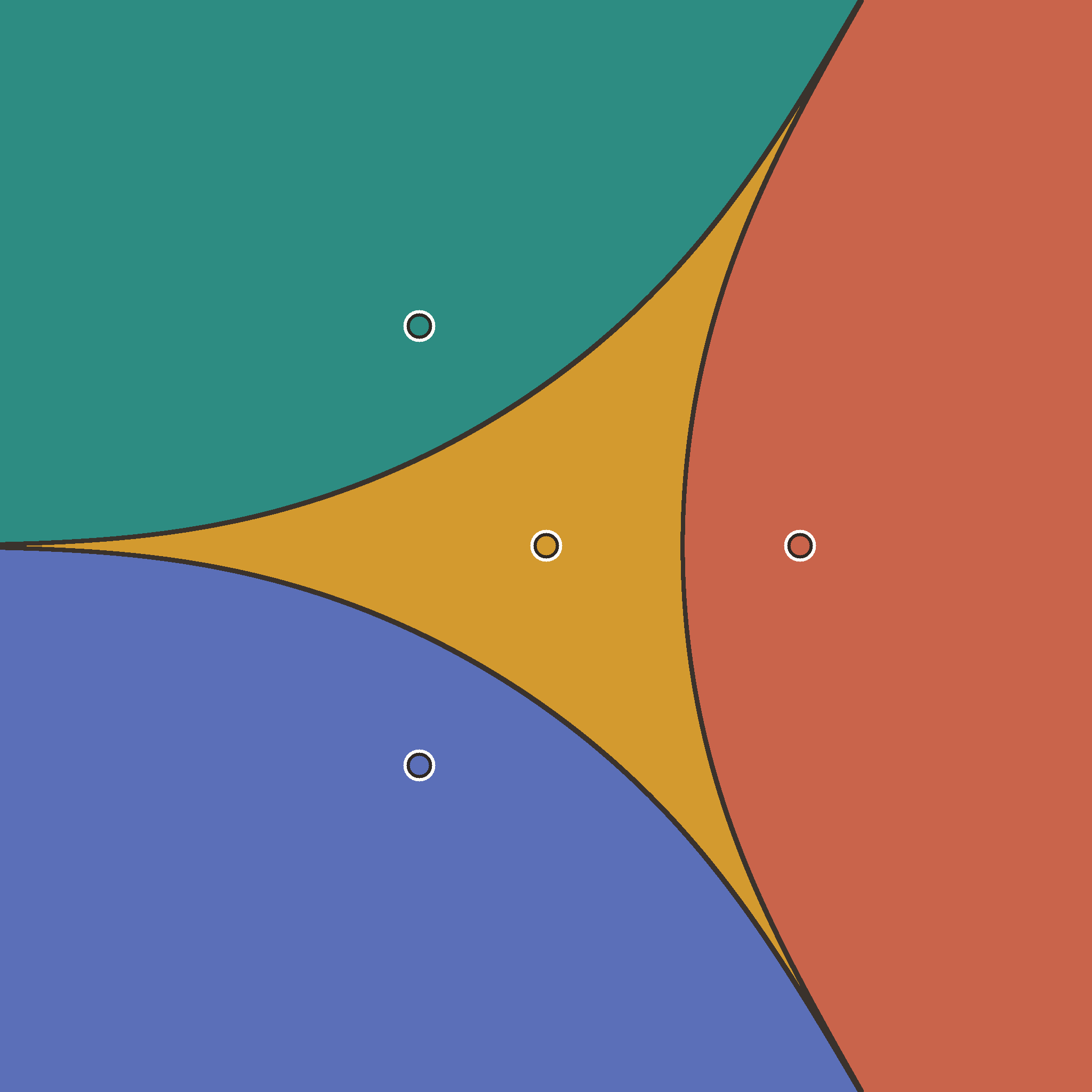}
\caption{$k=1$}
\end{subfigure}\hfill
\begin{subfigure}{0.235\textwidth}
\centering
\includegraphics[width=\linewidth]{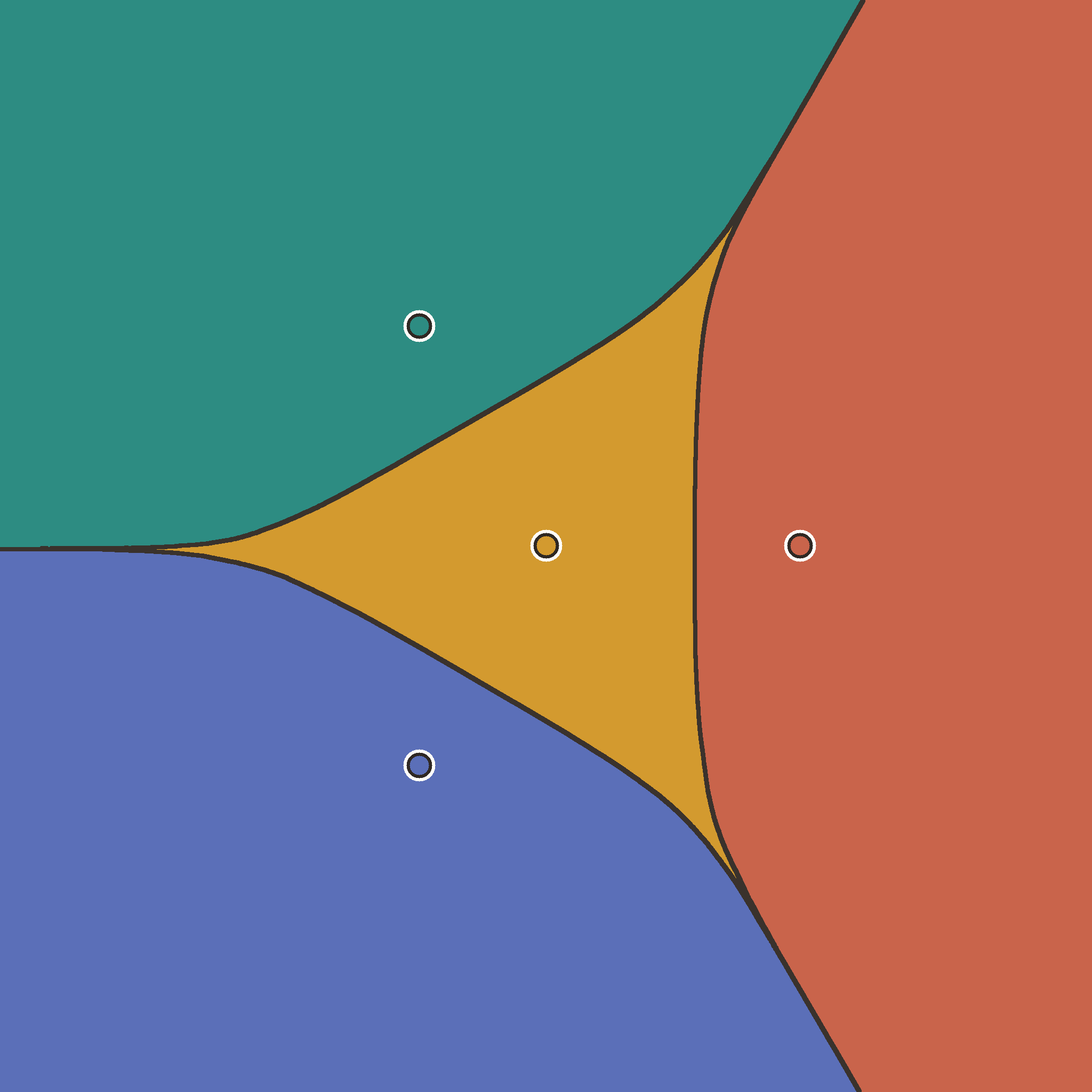}
\caption{$k=25$}
\end{subfigure}\hfill
\begin{subfigure}{0.235\textwidth}
\centering
\includegraphics[width=\linewidth]{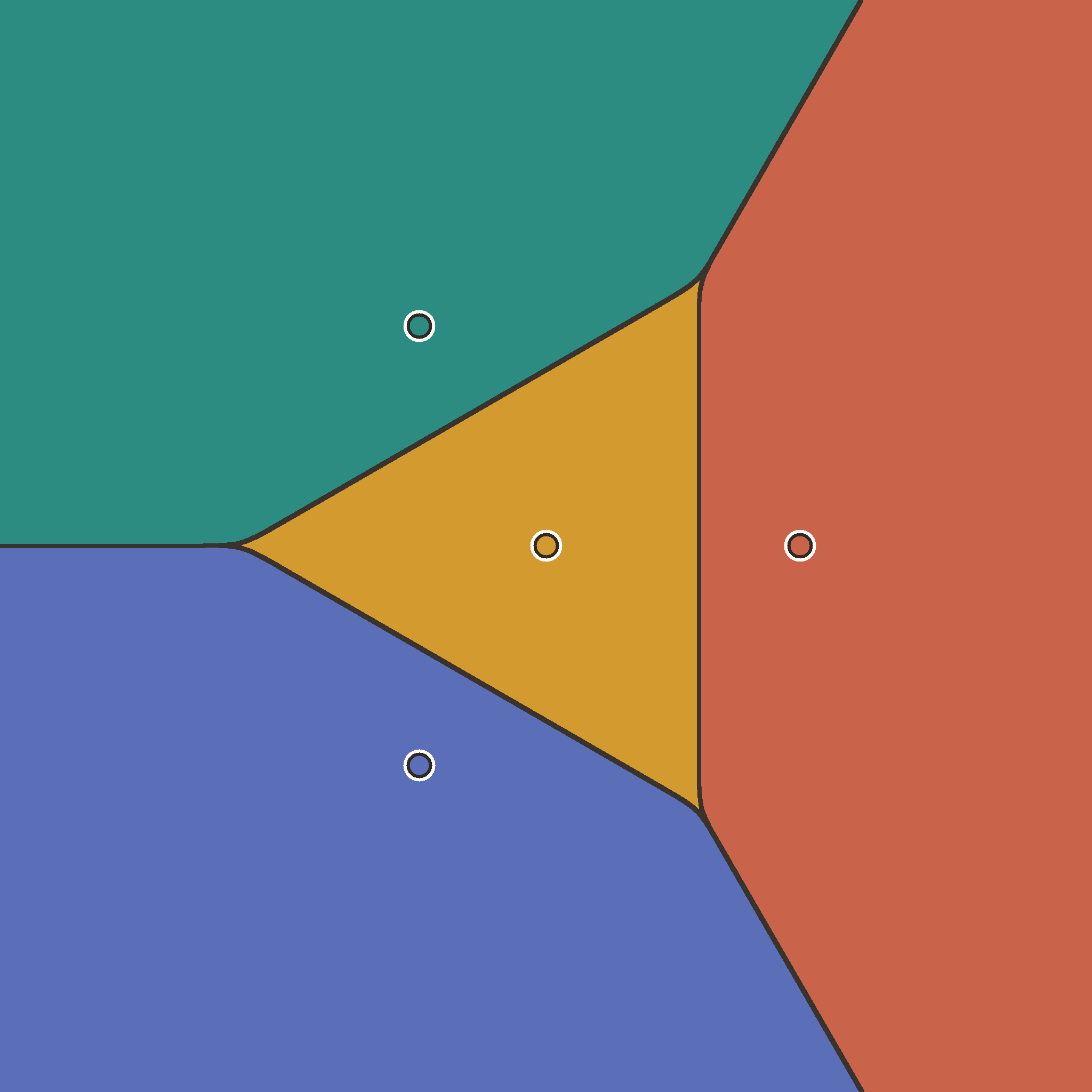}
\caption{$k=1\,000$}
\end{subfigure}\hfill
\begin{subfigure}{0.235\textwidth}
\centering
\includegraphics[width=\linewidth]{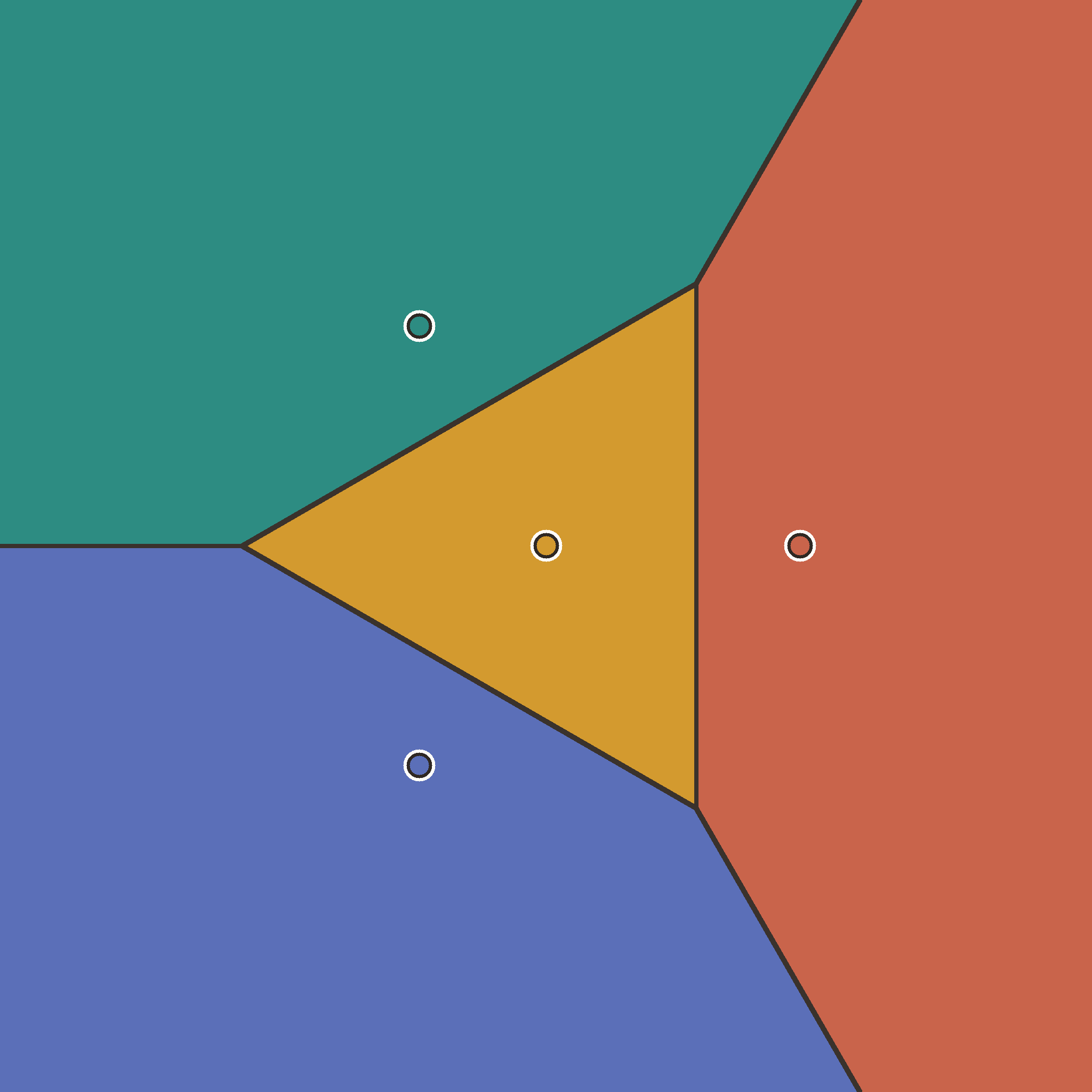}
\caption{OT plan $\pi^{\star}$}
\end{subfigure}
\caption{Terminal map cells in the symmetric four-point setting of
\citet[Section~5.1]{pierret-tessellations}.}
\label{fig:four-point-flow-cells}
\end{figure}

\begin{proof}
We closely follow the proof of \citet[Lemma~9]{pierret-tessellations}. The only change is that the
independent softmax weights are replaced by the batch OT posterior
weights.
By \cref{lem:gaussian-discrete-velocity-posterior}, 
\[
u_t^{\overline\pi_k}(x)
=
\sum_{i=1}^4 p_i^k(t,\,x)\frac{v_i-x}{1-t}
\quad\text{with}\quad
p_i^k(t,\,x)
=
\frac{
\varphi\!\left(\frac{x-tv_i}{1-t}\right)
\overline a_i^k\!\left(\frac{x-tv_i}{1-t}\right)
}{
\sum_{\ell=1}^4
\varphi\!\left(\frac{x-tv_\ell}{1-t}\right)
\overline a_\ell^k\!\left(\frac{x-tv_\ell}{1-t}\right)
}.
\]

By symmetry, it suffices to prove the claim for the ray
\(L_3=\{(-c,\,0),\;c\geqslant0\}\). First, for \(x=(b,\,0)\), symmetry gives
\(p_1^k(t,\,x)=p_2^k(t,\,x)\), hence \(u_t^{\overline\pi_k}(b,\,0)\in\mathbb{R}(1,\,0)\).
Moreover, at the origin one has
\(p_1^k(t,\,0)=p_2^k(t,\,0)=p_3^k(t,\,0)\), and since
\(v_1+v_2+v_3=0\), it follows that \(u_t^{\overline\pi_k}(0)=0\).
By uniqueness of the pre-terminal flow, \(\phi_t^{\overline\pi_k}(L_3)\subseteq L_3\).
Let \(b_0\leqslant0\) and write
\(\phi_t^{\overline\pi_k}(b_0,\,0)=(b_t,\,0)\): we prove that \(b_t\to0\) as
\(t\uparrow1\).

\emph{Step 1:} for every \(\varepsilon>0\), there exists \(t_0\in(0,\,1)\) such that
\(b_t\geqslant-\frac12-\varepsilon\) for all \(t\geqslant t_0\). Indeed,
\[
(1-t)\dot b_t
=
(p_1^k+p_2^k)\left(-\frac12-b_t\right)
+p_3^k(1-b_t)-p_4^k b_t
\geqslant -\frac12-b_t,
\]
where the weights are evaluated at \((t,(b_t,0))\). Hence
\(\dot b_t\geqslant(-\frac12-b_t)/(1-t)\), and Gronwall's lemma gives
\(b_t\geqslant(1-t)b_0-t/2\), which tends to \(-1/2\).

\emph{Step 2:} we show that for \(\varepsilon=0.4\), there exist \(t_1\in(0,\,1)\) and \(c>0\)
such that, for every \(t\geqslant t_1\) and every
\(b\in-\frac12+[-\varepsilon,\,\varepsilon]\),
\(
u_t^{\overline\pi_k}(b,\,0)=\bigl(w_t^k(b),\,0\bigr)
\)
with 
\(
w_t^k(b)\geqslant\frac{c}{1-t}.
\)
Indeed, 
\[
(1-t)w_t^k(b)
\geqslant
-b-\frac12\bigl(1-p_4^k(t,\,(b,\,0))\bigr)
\geqslant
\frac12-\varepsilon
-\frac12\bigl(1-p_4^k(t,\,(b,\,0))\bigr).
\]
It remains to show that \(p_4^k(t,\,(b,\,0))\to1\) uniformly on this interval.
From the posterior formula,
\[
p_4^k(t,(b,0))
=
\left(
1+
\sum_{i=1}^3
\frac{
\varphi\left(\frac{(b,0)-tv_i}{1-t}\right)
\overline a_i^k\left(\frac{(b,0)-tv_i}{1-t}\right)
}{
\varphi\left(\frac{(b,0)-tv_4}{1-t}\right)
\overline a_4^k\left(\frac{(b,0)-tv_4}{1-t}\right)
}
\right)^{-1}.
\]
By \cref{lem:expected-batch-ot-assignment-probabilities},
\(4^{-k}\leqslant\overline a_i^k\leqslant1\), and 
\(\|(b,\,0)-tv_i\|^2-\|(b,\,0)-tv_4\|^2\geqslant\delta>0\),
uniformly for \(b\in-\frac12+[-\varepsilon,\varepsilon]\) and \(t\) close enough to \(1\), for \(i\in\{1,\,2,\,3\}\), hence
\[
p_4^k(t,(b,0))
\geqslant
\left(
1+3\cdot4^k
\exp\left(-\frac{\delta}{2(1-t)^2}\right)
\right)^{-1}
\xrightarrow[t\to1]{}1.
\]

\emph{Step 3:} let \(r=\frac12-\varepsilon=0.1\). Since
\(\operatorname{sep}(\{v_1,\,v_2,\,v_3,\,v_4\})=1\), we have
\(r<\operatorname{sep}(\{v_1,\,v_2,\,v_3,\,v_4\})/8\). Choose \(t_r\) so that the trapping
argument around \(v_4\) from Step~2 of \cref{app:terminal-flow} applies for the ball \(B_r(v_4)\).
Set \(t_2=\max(t_0,\,t_1,\,t_r)\). We claim that there exists \(t\geqslant t_2\)
such that \(b_t>-\frac12+\varepsilon\). Otherwise, Step~1 gives
\(b_t\in-\frac12+[-\varepsilon,\,\varepsilon]\) for all \(t\in[t_2,\,1)\), and
Step~2 gives \(\dot b_t\geqslant c/(1-t)\). Thus
\(b_t\geqslant b_{t_2}+c\log((1-t_2)/(1-t))\), which diverges to \(+\infty\),
a contradiction.

Hence, for some \(t\geqslant t_r\), the trajectory belongs to \(B_r(v_4)\).
By the trapping argument cited above, it converges to \(v_4\). Thus
\((b_0,\,0)\in\Gamma_{k,\,4}\). Since \(b_0\leqslant0\) was arbitrary,
\(L_3\subseteq\Gamma_{k,\,4}\).
\end{proof}

\subsection[Posterior concentration and one-step Euler accuracy]{Posterior concentration and one-step Euler accuracy (\cref{sec:posterior-concentration})}
\label{app:posterior-concentration-euler}

In the following proposition, we show that in the Gaussian-to-discrete setting, under the independent coupling, there exist positive constants $C_1,\,C_2$ (depending only on $\nu$) such that for any $t\in[0,\,1)$,
\[
 0\leqslant1-\mathbb{E}\left[\max_{1\leqslant i\leqslant M} \mathbb{P}(Y=v_{i}\mid Z_{t})\right]\leqslant C_1 \exp\left(-C_2 \frac {t^2}{(1-t)^2}\right).
\]
The constant $C_2$ depends on $\nu$ through the separation of its support, which is a key reason why the posterior probabilities concentrate faster when the dimension increases: if the $M$ points $v_{j}$ are distributed uniformly on $[0,1]^{d}$ for instance, then the average minimal separation of two points of $\nu$ is of the order $M^{-2/d}$, which increases with $d$.

\begin{prop}[Posterior concentration and one-step Euler accuracy]
Let $X\sim\mathcal{N}(0,\,I_{d})$ and
$Y\sim\operatorname{Unif}(\mathcal{Y})$ be independent, where $\mathcal{Y}=\{v_{1},\ldots,v_{M}\}$ are distinct 
atoms in $\mathbb{R}^d$. Denote $Z_{t}=(1-t)X+tY$ and
\(
\operatorname{sep}(\mathcal{Y})
=\min_{1\leqslant i\ne j\leqslant M}\lVert v_{i}-v_{j}\rVert .
\)
Then, for every $t\in[0,\,1)$,
\[
0\leqslant1-\mathbb{E}\left[\max_{1\leqslant j\leqslant M}\mathbb{P}(Y=v_{j}\mid Z_{t})\right]
\leqslant(M-1)\exp\left(-\frac{t^{2}\operatorname{sep}(\mathcal{Y})^{2}}{8(1-t)^{2}}\right).
\]
If, in addition, $\mathcal{Y}\subseteq[-1,\,1]^{d}$, and if
$\rho_{t}^{\mu\otimes\nu}$ is the distribution of $Z_{t}$, then for every
$h\in[0,\,1-t]$,
\begin{align*}
W_{2}\left((\operatorname{id}+hu_{t}^{\mu\otimes\nu})_{\sharp}\rho_{t}^{\mu\otimes\nu},\,\rho_{t+h}^{\mu\otimes\nu}\right)
&\leqslant\frac{2h\sqrt{d(M-1)}}{1-t}\exp\left(-\frac{t^{2}\operatorname{sep}(\mathcal{Y})^{2}}{16(1-t)^{2}}\right).
\end{align*}
\end{prop}

\begin{proof}
At $t=0$, independence gives $\mathbb{P}(Y=v_j\mid Z_0)=1/M$, so the first inequality follows from $1-1/M\leqslant M-1$.
Now let $t\in(0,\,1)$. Write $s=\frac{t}{1-t}$ and $W_{t}=\frac{Z_{t}}{1-t}=X+sY$.
Almost surely, $\mathbb{P}(Y=v_{j}\mid Z_{t})=\mathbb{P}(Y=v_{j}\mid W_{t})$.
For $w\in\mathbb{R}^{d}$, choose a measurable maximizer
\[
\hat{\jmath}(w)\in\operatorname*{arg\,max}_{1\leqslant j\leqslant M}p_{j}(w)
\quad\text{where}\quad p_{j}(w)=\mathbb{P}(Y=v_{j}\mid W_{t}=w).
\]
Almost surely,
\(
\max_{1\leqslant j\leqslant M}\mathbb{P}(Y=v_{j}\mid W_{t})
=\mathbb{P}(Y=v_{\hat{\jmath}(W_{t})}\mid W_{t}),
\)
and therefore,
\[
\mathbb{E}\left[\max_{1\leqslant j\leqslant M}\mathbb{P}(Y=v_{j}\mid W_{t})\right]
=\mathbb{P}(Y=v_{\hat{\jmath}(W_{t})}).
\]
Now, as $X$ and $Y$ are independent,
\[
W_{t}\mid\{Y=v_{j}\} = X+sv_{j}\sim\mathcal{N}(sv_{j},\,I_{d}),
\]
hence the density of $W_{t}\mid\{Y=v_{j}\}$ is proportional to
$\exp(-\frac{1}{2}\lVert w-sv_{j}\rVert^{2})$, and 
\[
\hat{\jmath}(w)\in\argmin_{1\leqslant j\leqslant M}\lVert w-sv_{j}\rVert.
\]
Fix $i\in\{1,\ldots,M\}$ and condition on $Y=v_i$. On the event that $\hat{\jmath}(W_{t})=j$ with $j\ne i$,
$\lVert W_{t}-sv_{j}\rVert\leqslant\lVert W_{t}-sv_{i}\rVert$, hence
\[
\lVert X+s(v_{i}-v_{j})\rVert\leqslant\lVert X\rVert.
\]
Squaring and simplifying, we obtain 
\(
\langle X,\,v_{i}-v_{j}\rangle\leqslant-\frac{s}{2}\lVert v_{i}-v_{j}\rVert^{2}.
\)
Since $X$ is a standard Gaussian, $\langle X,\,v_{i}-v_{j}\rangle\sim\mathcal{N}(0,\,\lVert v_{i}-v_{j}\rVert^{2})$,
and denoting by $\Phi$ the c.d.f. of the standard Gaussian, we obtain
\[
\mathbb{P}(\hat{\jmath}(W_{t})=j\mid Y=v_{i})\leqslant\Phi\left(-\frac{s}{2}\lVert v_{i}-v_{j}\rVert\right)\leqslant\Phi\left(-\frac{s}{2}\operatorname{sep}(\mathcal{Y})\right).
\]
Summing over $j\in \{1,\ldots,M\}\backslash \{i\}$ gives
\[
\mathbb{P}(\hat{\jmath}(W_{t})\ne i\mid Y=v_{i})\leqslant(M-1)\Phi\left(-\frac{s}{2}\operatorname{sep}(\mathcal{Y})\right).
\]
Since $Y$ is uniform, averaging over $i$ yields
\[
\mathbb{P}(Y\ne v_{\hat{\jmath}(W_{t})})=\frac{1}{M}\sum_{i=1}^{M}\mathbb{P}(\hat{\jmath}(W_{t})\ne i\mid Y=v_{i})\leqslant(M-1)\Phi\left(-\frac{s}{2}\operatorname{sep}(\mathcal{Y})\right).
\]
Using $\Phi(-u)\leqslant e^{-u^{2}/2}$ for $u\geqslant0$, we conclude that
\[
\mathbb{P}(Y=v_{\hat{\jmath}(W_{t})})\geqslant1-(M-1)\exp\left(-\frac{s^{2}}{8}\operatorname{sep}(\mathcal{Y})^{2}\right).
\]
Recalling that $s=\frac{t}{1-t}$ and that conditioning on $Z_t$ or on $W_t$
is equivalent gives the claim.
See also \citet[Lemma~B.1]{dodson2026two} for a related posterior-concentration bound.

We now prove the second inequality for $t\in[0,\,1)$, assuming $\mathcal{Y}\subseteq[-1,\,1]^d$.
Let $\tilde{Z}_{t,\,h}=Z_{t}+hu_{t}^{\mu\otimes\nu}(Z_{t})$. Then
$\tilde{Z}_{t,\,h}\sim(\operatorname{id}+hu_{t}^{\mu\otimes\nu})_{\sharp}\rho_{t}^{\mu\otimes\nu}$,
and $Z_{t+h}\sim\rho_{t+h}^{\mu\otimes\nu}$. Hence, by the definition
of $W_{2}$,
\begin{equation}
W_{2}^{2}\left((\operatorname{id}+hu_{t}^{\mu\otimes\nu})_{\sharp}\rho_{t}^{\mu\otimes\nu},\,\rho_{t+h}^{\mu\otimes\nu}\right)\leqslant\mathbb{E}[\lVert\tilde{Z}_{t,\,h}-Z_{t+h}\rVert^{2}].\label{eq:one-step-euler-coupling}
\end{equation}
Now $Z_{t+h}=Z_{t}+h(Y-X)$, and by \cref{lem:gaussian-discrete-velocity-posterior},
\[
\tilde{Z}_{t,\,h}-Z_{t+h}=\frac{h}{1-t}\left(m_{t}^{\mu\otimes\nu}(Z_{t})-Y\right).
\]
Moreover, denoting $p_{\ell}(z)=\mathbb{P}(Y=v_{\ell}\mid Z_{t}=z)$
and $p_{*}(z)=\max_{\ell}p_{\ell}(z)$,
as $\lVert v_{i}-v_{j}\rVert\leqslant2\sqrt{d}$,
\begin{align*}
\mathbb{E}[\lVert Y-m_{t}^{\mu\otimes\nu}(Z_{t})\rVert^{2}\mid Z_{t}=z]
&= \sum_{i=1}^{M}\lVert v_{i}-m_{t}^{\mu\otimes\nu}(z)\rVert^{2}p_{i}(z)
= \frac{1}{2}\sum_{1\leqslant i,j\leqslant M}p_{i}(z)p_{j}(z)\lVert v_{i}-v_{j}\rVert^{2} \\
&\leqslant 2d\left(1-\sum_{j=1}^{M}p_{j}(z)^{2}\right).
\end{align*}
Moreover, since $\sum_{j=1}^M p_{j}(z)^{2}\geqslant p_{*}(z)^{2}$, we obtain
\[
1-\sum_{j=1}^{M}p_{j}(z)^{2}\leqslant1-p_{*}(z)^{2}=(1-p_{*}(z))(1+p_{*}(z))\leqslant2(1-p_{*}(z)).
\]
Hence taking expectations,
\begin{equation}
\mathbb{E}[\lVert Y-m_{t}^{\mu\otimes\nu}(Z_{t})\rVert^{2}]\leqslant4d\left(1-\mathbb{E}\left[\max_{1\leqslant j\leqslant M}\mathbb{P}(Y=v_{j}\mid Z_{t})\right]\right).\label{eq:one-step-euler-posterior-bound}
\end{equation}
Combining \eqref{eq:one-step-euler-coupling} and \eqref{eq:one-step-euler-posterior-bound}, and then applying
the first inequality yields the result.
\end{proof}

\subsection[Two-atom tractable case]{Two-atom tractable case (\cref{prop:binary-integration-asymptotics})}
\label{app:binary-asymptotics}

Classical Euler bounds are poorly adapted to the present regime: they
are driven by global Lipschitz estimates on the velocity field, whereas
in the binary model the relevant phenomenon is the formation of a
sharp transition layer near the origin. We therefore argue directly
with the expected batch OT plan in the binary case.
From now on, we work with $\mu=\mathcal{N}(0,\,1)$ and
$\nu=\operatorname{Unif}\{-1,\,+1\}$.

\subsubsection{Auxiliary formulas}

For $k\in\mathbb{N}^{*}$ and $x\in\mathbb{R}$, define
\(
q_{k}(x)=\mathbb{P}_{(X,\,Y)\sim\overline{\pi}_{k}}(Y=-1\mid X=x).
\)

\begin{lem}
\label{lem:binary-assignment-probability}
For $x\in\mathbb{R}$, let $u=\Phi(x)$, where $\Phi$ denotes the
c.d.f. of a standard Gaussian. Let
$N_{-}\sim\operatorname{Bin}(k,\,\tfrac{1}{2})$ and define
$B_{r}=\mathbb{P}(N_{-}\geqslant r)$ for $r\in\{1,\ldots,k\}$. Then $q_{k}(-x)=1-q_{k}(x)$, and
\begin{equation}
q_{k}(x)=\sum_{j=0}^{k-1}\binom{k-1}{j}u^{j}(1-u)^{k-1-j}B_{j+1}.
\label{eq:binary-assignment-probability}
\end{equation}
\end{lem}

\begin{proof}
The coupling is symmetric under the map $(x,\,y)\mapsto(-x,\,-y)$.
Thus if $(X,\,Y)\sim\overline{\pi}_{k}$, then $(-X,\,-Y)\sim\overline{\pi}_{k}$, and
\(
q_{k}(-x)=\mathbb{P}(Y=-1\mid X=-x)=\mathbb{P}(Y=+1\mid X=x)=1-q_{k}(x).
\)

Since the OT map in dimension $1$ is the monotone rearrangement
\citep[Section~2.6]{peyre}, a sample from $\overline{\pi}_{k}$ may be
generated as follows. Draw
$(X_{1},\ldots,X_{k},\,Y_{1},\ldots,Y_{k})\sim\mu^{\otimes k}\otimes\nu^{\otimes k}$,
let $Y_{(1)}\leqslant\cdots\leqslant Y_{(k)}$ be the sorted target
labels, choose an index $J\sim\operatorname{Unif}\{1,\ldots,k\}$
independently, let $L$ be the rank of $X_{J}$ among
$(X_{1},\ldots,X_{k})$, and output $(X,\,Y)=(X_{J},\,Y_{(L)})$.

Condition on the event $\{X=X_{J}=x\}$. The remaining $k-1$ source
samples are still i.i.d.\ $\mathcal{N}(0,\,1)$ under this conditioning.
Let
\[
J_{x}=\sum_{i\neq J}\mathbbm{1}_{\{X_{i}\leqslant x\}}
\sim\operatorname{Bin}(k-1,\,u)\quad\text{and}\quad
N_{-}=\sum_{j=1}^{k}\mathbbm{1}_{\{Y_{j}=-1\}}
\sim\operatorname{Bin}(k,\,1/2).
\]
Then the rank of $X_{J}$ is $L=J_{x}+1$, and
$Y_{(r)}=-1$ if and only if $r\leqslant N_{-}$. Therefore, for every
$j\in\{0,\ldots,k-1\}$,
\[
\mathbb{P}(Y=-1\mid X=x,\,J_{x}=j)
=\mathbb{P}(Y_{(j+1)}=-1)=\mathbb{P}(N_{-}\geqslant j+1)=B_{j+1}.
\]
Averaging over the law of $J_{x}$ yields \eqref{eq:binary-assignment-probability}.
\end{proof}

For $t\in[0,\,1)$ and $z\in\mathbb{R}$, set
$s=t/(1-t)\in[0,\,\infty)$ and $\xi=z/(1-t)\in\mathbb{R}$.

\begin{lem}
\label{lem:binary-conditional-mean}
For every $k\in\mathbb{N}^{*}$, $t\in[0,\,1)$ and $z\in\mathbb{R}$,
\[
m_{t}^{\overline{\pi}_{k}}(z)=\tanh\!\left(s\xi+\frac{1}{2}\log\frac{q_{k}(s-\xi)}{q_{k}(s+\xi)}\right)
\quad\text{and}\quad
m_{t}^{\mu\otimes\nu}(z)=\tanh\!\left(\frac{tz}{(1-t)^{2}}\right).
\]
\end{lem}

\begin{proof}
Fix $k\in\mathbb{N}^{*}$, $t\in[0,\,1)$ and $z\in\mathbb{R}$. Since
$Y\in\{-1,\,1\}$,
\[
\begin{aligned}
m_{t}^{\overline{\pi}_{k}}(z) &=\mathbb{E}[Y\mid Z_{t}=z] \\
&=\mathbb{P}(Y=+1\mid Z_{t}=z)-\mathbb{P}(Y=-1\mid Z_{t}=z) \\
&= \tanh\!\left(\frac{1}{2}\log\psi_{t}^{\overline{\pi}_{k}}(z)\right)\quad
\text{with}\quad \psi_{t}^{\overline{\pi}_{k}}(z)=
\frac{\mathbb{P}(Y=+1\mid Z_{t}=z)}{\mathbb{P}(Y=-1\mid Z_{t}=z)}.
\end{aligned}
\]
For $y\in\{-1,\,1\}$, applying Bayes' formula gives
\[
\mathbb{P}(Y=y\mid Z_{t}=z)
=\frac{f_{Z_{t}\mid Y=y}(z)\mathbb{P}(Y=y)}{f_{Z_{t}}(z)}.
\]
The marginal probabilities $\mathbb{P}(Y=y)=1/2$ and the density
$f_{Z_{t}}(z)$ cancel in the odds ratio. On the event $(Z_{t}=z,\,Y=y)$,
the identity $Z_{t}=(1-t)X+tY$ forces
$x_{y}=(z-ty)/(1-t)$. Changing variables between $Z_{t}$ and $X$
in the conditional density gives, up to the common factor $(1-t)^{-1}$,
\[
f_{Z_{t}\mid Y=y}(z)\propto\varphi(x_{y})\mathbb{P}(Y=y\mid X=x_{y}),
\]
where $\varphi$ is the standard Gaussian density. Consequently,
\begin{equation}
\psi_{t}^{\overline{\pi}_{k}}(z)
=\frac{\varphi(x_{+1})}{\varphi(x_{-1})}
\frac{\mathbb{P}(Y=+1\mid X=x_{+1})}{\mathbb{P}(Y=-1\mid X=x_{-1})}.
\label{eq:binary-posterior-odds}
\end{equation}
The ratio of Gaussian densities satisfies
\[
\log\frac{\varphi(x_{+1})}{\varphi(x_{-1})}
=\frac{x_{-1}^{2}-x_{+1}^{2}}{2}=\frac{2tz}{(1-t)^{2}}=2s\xi.
\]
Moreover, $\mathbb{P}(Y=-1\mid X=x)=q_{k}(x)$ and
$\mathbb{P}(Y=+1\mid X=x)=q_{k}(-x)$ by
\cref{lem:binary-assignment-probability}. Since
$x_{+1}=\xi-s$ and $x_{-1}=\xi+s$, taking logarithms in
\eqref{eq:binary-posterior-odds} yields
\[
\log\psi_{t}^{\overline{\pi}_{k}}(z)
=2s\xi+\log\frac{q_{k}(s-\xi)}{q_{k}(s+\xi)}.
\]
This proves the first formula. Then \cref{lem:binary-assignment-probability}
gives $q_{1}\equiv1/2$ and the second formula follows.
\end{proof}

\begin{lem}[Structure of the Euler iterates in the binary case]
\label{lem:binary-euler-structure}
Fix $k\in\mathbb{N}^{*}$ and $n\geqslant1$. For every $0\leqslant m\leqslant n$,
the map $f_{m,\,n}^{\overline{\pi}_{k}}$ is odd. In the independent case
$\overline{\pi}_{1}=\mu\otimes\nu$, the map $f_{m,\,n}^{\mu\otimes\nu}$
is increasing and concave for $x\geqslant0$. Finally,
$\operatorname{im}(f_{n}^{\overline{\pi}_{k}})\subseteq[-1,\,1]$. Therefore,
\(
\mathscr{E}_{n,\,k}=2\,\mathbb{E}_{X\sim\mu}[(1-f_{n}^{\overline{\pi}_{k}}(X))\mathbbm{1}_{\{X>0\}}].
\)
\end{lem}

\begin{proof}
We first prove oddness. The initial map
$f_{0,\,n}^{\overline{\pi}_{k}}=\operatorname{id}$ is odd. Assume that
$f_{m,\,n}^{\overline{\pi}_{k}}$ is odd for some $0\leqslant m<n$, and write
$r=n-m$ and $t=m/n$. By \cref{lem:gaussian-discrete-velocity-posterior}, the Euler recursion
\eqref{eq:euler-recursion} becomes
\[
f_{m+1,\,n}^{\overline{\pi}_{k}}(x)
=\left(1-\frac{1}{r}\right)f_{m,\,n}^{\overline{\pi}_{k}}(x)
+\frac{1}{r}\,m_{t}^{\overline{\pi}_{k}}\!\left(f_{m,\,n}^{\overline{\pi}_{k}}(x)\right).
\]
By \cref{lem:binary-conditional-mean}, the map $m_{t}^{\overline{\pi}_{k}}$
is odd, hence the right-hand side is odd, and induction gives the claim.

We next specialize to $\overline{\pi}_{1}=\mu\otimes\nu$. By
\cref{lem:binary-conditional-mean},
\(
m_{t}^{\mu\otimes\nu}(z)=\tanh(\frac{tz}{(1-t)^{2}}).
\)
For $z\geqslant0$, this map is increasing and concave. Since
$f_{0,\,n}^{\mu\otimes\nu}=\operatorname{id}$ is increasing and concave
for $x\geqslant0$, and since the composition of increasing concave
maps is increasing and concave, the recursion shows by induction that
$f_{m,\,n}^{\mu\otimes\nu}$ is increasing and concave for $x\geqslant0$
for every $m$.

For the terminal bound, take $m=n-1$, so that $r=1$. Then
\(
f_{n}^{\overline{\pi}_{k}}(x)=m_{1-1/n}^{\overline{\pi}_{k}}(f_{n-1,\,n}^{\overline{\pi}_{k}}(x)).
\)
By \cref{lem:binary-conditional-mean}, this map takes values in
$[-1,\,1]$.
Finally, as we are in dimension one, \citet[Note~4.2]{pierret} gives
$\phi^{\overline{\pi}_{k}}(X)=\operatorname{sign}(X)$ almost surely, hence
\(
\mathscr{E}_{n,\,k}
=\mathbb{E}\bigl[|f_{n}^{\overline{\pi}_{k}}(X)-\operatorname{sign}(X)|\bigr].
\)
For $x>0$, the terminal bound gives
$|f_{n}^{\overline{\pi}_{k}}(x)-1|=1-f_{n}^{\overline{\pi}_{k}}(x)$. For $x<0$,
oddness and the terminal bound give
$|f_{n}^{\overline{\pi}_{k}}(x)+1|=1-f_{n}^{\overline{\pi}_{k}}(-x)$. Splitting
the expectation over the positive and negative half-lines and using
the symmetry of the Gaussian density yields the formula.
\end{proof}

\subsubsection{One Euler step: the effect of the OT batch size}

For one Euler step, the error is entirely statistical: it measures the
probability that the expected batch OT plan assigns the wrong sign.

\begin{proof}[Proof of the asymptotic for $\mathscr{E}_{1,\,k}$ in \cref{prop:binary-integration-asymptotics}]
Since $f_{1}^{\overline{\pi}_{k}}(x)=\mathbb{E}[Y\mid X=x]$, we have
\[
|f_{1}^{\overline{\pi}_{k}}(X)-\operatorname{sign}(X)|
=|\mathbb{E}[Y\mid X]-\operatorname{sign}(X)|
=\mathbb{E}[|Y-\operatorname{sign}(X)|\mid X].
\]
Since both $Y$ and $\operatorname{sign}(X)$ take values in $\{-1,\,+1\}$,
the integrand is equal to $0$ if $Y=\operatorname{sign}(X)$ and to
$2$ otherwise. Therefore
\(
\mathscr{E}_{1,\,k}=2\,\mathbb{P}(Y\neq\operatorname{sign}(X)).
\)

We compute this probability through the rank representation of
$\overline{\pi}_{k}$. Let $U_{1},\ldots,U_{k},V_{1},\ldots,V_{k}$ be i.i.d.\
uniform random variables on $[0,\,1]$, let
$L\sim\operatorname{Unif}\{1,\ldots,k\}$ be independent of everything else,
and set
\[
X=\Phi^{-1}(U_{(L)})\quad\text{and}\quad
Y=\begin{cases}
-1, & \text{if } V_{(L)}\leqslant1/2,\\
+1, & \text{otherwise}.
\end{cases}
\]
Then $(X,\,Y)\sim\overline{\pi}_{k}$. Define
\[
B=\sum_{i=1}^{k}\mathbbm{1}_{\{U_{i}\leqslant1/2\}}\quad\text{and}\quad
B'=\sum_{i=1}^{k}\mathbbm{1}_{\{V_{i}\leqslant1/2\}}.
\]
Then $B$ and $B'$ are independent variables with law
$\operatorname{Bin}(k,\,1/2)$, and
\begin{align*}
\operatorname{sign}(X)=-1&\iff U_{(L)}<1/2\iff L\leqslant B,\\
Y=-1&\iff V_{(L)}\leqslant1/2\iff L\leqslant B'.
\end{align*}
Therefore, $Y\neq\operatorname{sign}(X)$ if and only if
$\min(B,\,B')<L\leqslant\max(B,\,B')$.
Conditionally on $(B,\,B')$, the number
of such ranks is exactly $|B-B'|$, so
\[
\mathbb{P}(Y\neq\operatorname{sign}(X)\mid B,\,B')=\frac{|B-B'|}{k}.
\]
Taking expectations gives
\begin{equation}
\mathscr{E}_{1,\,k}=\frac{2}{k}\,\mathbb{E}|B-B'|.
\label{eq:binary-one-step-binomial}
\end{equation}

Write $B-B'=\sum_{i=1}^{k}D_{i}$, where the $D_{i}$ are i.i.d.\ with
\[
\mathbb{P}(D_{i}=1)=\mathbb{P}(D_{i}=-1)=\frac{1}{4}\quad\text{and}\quad
\mathbb{P}(D_{i}=0)=\frac{1}{2}.
\]
Then $\mathbb{E}[D_{i}]=0$ and $\operatorname{Var}(D_{i})=1/2$. By the
central limit theorem,
\[
\frac{B-B'}{\sqrt{k}}\xrightarrow[k\to+\infty]{\mathrm{d}}Z
\quad\text{with}\quad Z\sim\mathcal{N}(0,\,1/2).
\]
Since the variables $(B-B')/\sqrt{k}$ have uniformly bounded second
moments, their absolute values are uniformly integrable. Therefore
\[
\frac{1}{\sqrt{k}}\,\mathbb{E}|B-B'|\xrightarrow[k\to+\infty]{}\mathbb{E}|Z|=\frac{1}{\sqrt{\pi}}.
\]
Combining this with \eqref{eq:binary-one-step-binomial} yields
$\mathscr{E}_{1,\,k}\sim2/\sqrt{\pi k}$, and taking logarithms gives
the result.
\end{proof}

\subsubsection{Many Euler steps: the effect of the numerical integration budget}

We now fix the coupling to be the independent one,
$\overline{\pi}_{1}=\mu\otimes\nu$, and study the effect of increasing the
number of Euler steps. The Euler map is an odd increasing concave sigmoid
approximating the sign map. The error is therefore controlled by the
width of the transition layer near the origin, and that width is governed
by the derivative at the origin.

From now on, write $f_{m,\,n}=f_{m,\,n}^{\mu\otimes\nu}$ and
$f_{n}=f_{n,\,n}$. By \cref{lem:binary-conditional-mean}, for
$0\leqslant m<n$, if $r=n-m$ and $\alpha_{m,\,n}=mn/r^{2}$, then
\begin{equation}
f_{m+1,\,n}(x)
=\left(1-\frac{1}{r}\right)f_{m,\,n}(x)
+\frac{1}{r}\,\tanh\!\left(\alpha_{m,\,n}f_{m,\,n}(x)\right).
\label{eq:binary-independent-euler-recursion}
\end{equation}
Fix once and for all a constant $A>2\pi/\sqrt{3}$, and define
\(
x_{n}=\frac{A n^{2/3}}{f_{n,\,n}'(0)}.
\)

\begin{lem}
\label{lem:binary-euler-slope-product}
For $0\leqslant m<n$, writing $r=n-m$,
\[
f'_{m+1,\,n}(0)=c_{r}\,f'_{m,\,n}(0)\quad\text{and}\quad
c_{r}=1+\frac{n^{2}-nr-r^{2}}{r^{3}}.
\]
Consequently, for $n\geqslant2$,
$\log f_{n,\,n}'(0)=\sum_{r=1}^{n}\log c_{r}$.
\end{lem}

\begin{proof}
Differentiate \eqref{eq:binary-independent-euler-recursion} at $0$.
Since $f_{m,\,n}(0)=0$ by \cref{lem:binary-euler-structure} and
$\tanh'(0)=1$,
\[
f'_{m+1,\,n}(0)=\left(1-\frac{1}{r}+\frac{\alpha_{m,\,n}}{r}\right)f'_{m,\,n}(0).
\]
Now $\alpha_{m,\,n}=mn/r^{2}=n(n-r)/r^{2}$, which gives the announced
coefficient $c_{r}$.
\end{proof}

\begin{lem}
\label{lem:binary-euler-slope-asymptotic}
\(
\frac{1}{n^{2/3}}\log f_{n,\,n}'(0)
\xrightarrow[n\to\infty]{}\frac{2\pi}{\sqrt{3}}.
\)
\end{lem}

\begin{proof}
The $r=1$ term in \cref{lem:binary-euler-slope-product} is $\log(n(n-1))=O(\log n)$,
which is $o(n^{2/3})$. For $r\geqslant2$, factor
$c_{r}=(1+u_{r})(1-x_{r})$ with
\(
u_{r}=\frac{n^{2}}{r^{3}}
\) and
\(
x_{r}=\frac{n/r^{2}+1/r}{1+u_{r}}
\).
Then
\[
\log c_{r}=\log\left(1+\frac{n^{2}}{r^{3}}\right)+\log(1-x_{r}).
\]
Splitting the sum at $R_{n}=\lceil n^{2/3}\rceil$ shows that
$\sum_{r=2}^{n}|\log(1-x_{r})|=o(n^{2/3})$. Hence
\[
\log f_{n,\,n}'(0)=\sum_{r=2}^{n}\log\left(1+\frac{n^{2}}{r^{3}}\right)+o(n^{2/3}).
\]
By integral comparison,
\[
\sum_{r=2}^{n}\log\left(1+\frac{n^{2}}{r^{3}}\right)
=\int_{1}^{n}\log\left(1+\frac{n^{2}}{u^{3}}\right)\,\mathrm{d}u+o(n^{2/3}).
\]
With the change of variables $u=n^{2/3}s$, the right-hand side becomes
\[
n^{2/3}\int_{n^{-2/3}}^{n^{1/3}}\log(1+s^{-3})\,\mathrm{d}s+o(n^{2/3}).
\]
Since $\int_{0}^{\infty}\log(1+s^{-3})\,\mathrm{d}s=2\pi/\sqrt{3}$,
the claim follows.
\end{proof}

\begin{lem}
\label{lem:binary-euler-derivative-comparison}
For every $0\leqslant m\leqslant n-2$,
\(
f'_{m,\,n}(0)\leqslant\frac{8(n-m)}{n^{2}}\,f'_{n-1,\,n}(0).
\)
\end{lem}

\begin{proof}
Let $r=n-m$. By \cref{lem:binary-euler-slope-product},
\[
f'_{n-1,\,n}(0)=f'_{m,\,n}(0)\prod_{q=2}^{r}c_{q}.
\]
For $q\in\{3,\ldots,r\}$, since $n(n-q)\geqslant0$,
\[
c_{q}=1-\frac{1}{q}+\frac{n(n-q)}{q^{3}}\geqslant1-\frac{1}{q}=\frac{q-1}{q}.
\]
By telescoping,
\[
\prod_{q=3}^{r}c_{q}\geqslant\prod_{q=3}^{r}\frac{q-1}{q}=\frac{2}{r}.
\]
Also $c_{2}=(n^{2}-2n+4)/8\geqslant n^{2}/16$. Thus
\[
f'_{n-1,\,n}(0)\geqslant f'_{m,\,n}(0)\,\frac{n^{2}}{8r},
\]
which is the desired inequality.
\end{proof}

\begin{lem}
\label{lem:binary-transition-saturation}
For all sufficiently large $n$,
\(
0\leqslant1-f_{n,\,n}(x_{n})\leqslant2e^{-A n^{2/3}}.
\)
\end{lem}

\begin{proof}
Fix $0\leqslant m\leqslant n-2$ and write $r=n-m$. Since $f_{m,\,n}$ is
concave for $x\geqslant0$ and vanishes at $0$,
\[
0\leqslant f_{m,\,n}(x_{n})\leqslant f'_{m,\,n}(0)x_{n}.
\]
By \cref{lem:binary-euler-derivative-comparison},
$f'_{m,\,n}(0)\leqslant8r f'_{n-1,\,n}(0)/n^{2}$. Using also
$\alpha_{m,\,n}\leqslant n^{2}/r^{2}$, we obtain
\[
\alpha_{m,\,n}f_{m,\,n}(x_{n})
\leqslant\frac{8}{r}\,f'_{n-1,\,n}(0)x_{n}.
\]
Since
\(
x_{n}=\frac{A n^{2/3}}{f_{n,\,n}'(0)}
=\frac{A n^{2/3}}{n(n-1)f'_{n-1,\,n}(0)},
\)
and $r\geqslant2$,
\[
\alpha_{m,\,n}f_{m,\,n}(x_{n})
\leqslant\frac{4A n^{2/3}}{n(n-1)}\leqslant8A n^{-4/3}\eqqcolon\delta_{n}.
\]
Thus $\delta_{n}\to0$, and for all large enough $n$ we have
$\delta_{n}\leqslant1$. On $[0,\,\delta_{n}]$,
\[
\tanh(u)\geqslant u-\frac{u^{3}}{3}\geqslant u\left(1-\frac{\delta_{n}^{2}}{3}\right).
\]
Applying this estimate in \eqref{eq:binary-independent-euler-recursion} at
$x_{n}$ yields
\begin{align*}
f_{m+1,\,n}(x_{n})
&\geqslant\left(1-\frac{1}{r}\right)f_{m,\,n}(x_{n})
+\frac{1}{r}\left(1-\frac{\delta_{n}^{2}}{3}\right)
\alpha_{m,\,n}f_{m,\,n}(x_{n})\\
&=\left(c_{r}-\frac{\delta_{n}^{2}}{3}\frac{\alpha_{m,\,n}}{r}\right)f_{m,\,n}(x_{n}).
\end{align*}
Since
\(
c_{r}-\frac{\delta_{n}^{2}}{3}\frac{\alpha_{m,\,n}}{r}
-\left(1-\frac{\delta_{n}^{2}}{3}\right)c_{r}
=\frac{\delta_{n}^{2}}{3}\left(1-\frac{1}{r}\right)\geqslant0,
\)
we obtain
\[
f_{m+1,\,n}(x_{n})
\geqslant\left(1-\frac{\delta_{n}^{2}}{3}\right)c_{r}f_{m,\,n}(x_{n}).
\]
Iterating from $m=0$ to $m=n-2$ gives
\[
f_{n-1,\,n}(x_{n})
\geqslant\left(1-\frac{\delta_{n}^{2}}{3}\right)^{n-1}
\left(\prod_{r=2}^{n}c_{r}\right)x_{n}
=\left(1-\frac{\delta_{n}^{2}}{3}\right)^{n-1}f'_{n-1,\,n}(0)x_{n}.
\]
Since $n\delta_{n}^{2}\to0$, the prefactor is at least $1/2$ for all
large enough $n$. Therefore
\[
n(n-1)f_{n-1,\,n}(x_{n})\geqslant\frac{A}{2}n^{2/3}.
\]
Finally,
$f_{n,\,n}(x)=\tanh\!\left(n(n-1)f_{n-1,\,n}(x)\right)$, and
$1-\tanh(u)\leqslant2e^{-2u}$ for $u\geqslant0$. Hence
\[
1-f_{n,\,n}(x_{n})
\leqslant2\exp\!\left(-2n(n-1)f_{n-1,\,n}(x_{n})\right)
\leqslant2e^{-A n^{2/3}}.
\]
\end{proof}

\begin{lem}
\label{lem:binary-error-slope-equivalence}
There exist constants $c,\,C>0$ such that, for all large enough $n$,
\[
\frac{c}{f_{n,\,n}'(0)}\leqslant\mathscr{E}_{n,\,1}
\leqslant\frac{Cn^{2/3}}{f_{n,\,n}'(0)}.
\]
In particular,
$\log\mathscr{E}_{n,\,1}=-\log f_{n,\,n}'(0)+O(\log n)$.
\end{lem}

\begin{proof}
We use \cref{lem:binary-euler-structure}, which gives the formula for
$\mathscr{E}_{n,\,1}$ and the monotonicity and concavity of $f_{n,\,n}$
for $x\geqslant0$.

For the lower bound, concavity and $f_{n,\,n}(0)=0$ give
$f_{n,\,n}(x)\leqslant f_{n,\,n}'(0)x$ for $x\geqslant0$. Set
$a_{n}=1/(2f_{n,\,n}'(0))$. Then $1-f_{n,\,n}(x)\geqslant1/2$ on
$[0,\,a_{n}]$, and
\[
\mathscr{E}_{n,\,1}
=2\mathbb{E}_{X\sim\mu}\bigl[(1-f_{n,\,n}(X))\mathbbm{1}_{\{X>0\}}\bigr]
\geqslant\int_{0}^{a_{n}}\varphi(x)\,\mathrm{d}x.
\]
By \cref{lem:binary-euler-slope-asymptotic}, $f_{n,\,n}'(0)\to\infty$, so
$a_{n}\to0$. For large enough $n$, $a_{n}\leqslant1$, hence
\[
\mathscr{E}_{n,\,1}\geqslant\frac{\varphi(1)}{2f_{n,\,n}'(0)}.
\]

For the upper bound, $f_{n,\,n}$ is increasing and
$0\leqslant1-f_{n,\,n}(x)\leqslant1$ for $x\geqslant0$. Thus
\[
(1-f_{n,\,n}(X))\mathbbm{1}_{\{X>0\}}
\leqslant\mathbbm{1}_{\{0\leqslant X\leqslant x_{n}\}}
+(1-f_{n,\,n}(x_{n}))\mathbbm{1}_{\{X\geqslant x_{n}\}}.
\]
Taking expectations and using \cref{lem:binary-transition-saturation},
\[
\mathscr{E}_{n,\,1}
\leqslant2\mu([0,\,x_{n}])+4e^{-A n^{2/3}}
\leqslant2\varphi(0)\frac{A n^{2/3}}{f_{n,\,n}'(0)}+4e^{-A n^{2/3}}.
\]
Since $A>2\pi/\sqrt{3}$, \cref{lem:binary-euler-slope-asymptotic} implies
$e^{-A n^{2/3}}\leqslant n^{2/3}/f_{n,\,n}'(0)$ for all large enough $n$.
This proves the upper bound, and the logarithmic equivalence follows.
\end{proof}

\begin{proof}[Proof of the asymptotic for $\mathscr{E}_{n,\,1}$ in \cref{prop:binary-integration-asymptotics}]
By \cref{lem:binary-error-slope-equivalence},
$\log\mathscr{E}_{n,\,1}=-\log f_{n,\,n}'(0)+O(\log n)$. By
\cref{lem:binary-euler-slope-asymptotic},
$\log f_{n,\,n}'(0)\sim(2\pi/\sqrt{3})n^{2/3}$. Therefore
\[
\log\mathscr{E}_{n,\,1}\sim-\frac{2\pi}{\sqrt{3}}n^{2/3}.
\]
\end{proof}

\subsubsection{Numerical verification}

Numerically, $\mathscr{E}_{n,\,1}$ decreases stretched-exponentially in
$n$, whereas $\mathscr{E}_{1,\,k}$ decreases only polynomially in $k$.
Moreover, the asymptotic equivalent $2/\sqrt{\pi k}$ is accurate already for
$\mathscr{E}_{1,\,k}$, while
$\exp(-(2\pi/\sqrt{3})n^{2/3})$ captures the decay rate of
$\mathscr{E}_{n,\,1}$ up to a multiplicative constant
(\cref{fig:binary-asymptotics-numerics}).
\begin{figure}[t]
\centering
\includegraphics[width=1\linewidth]{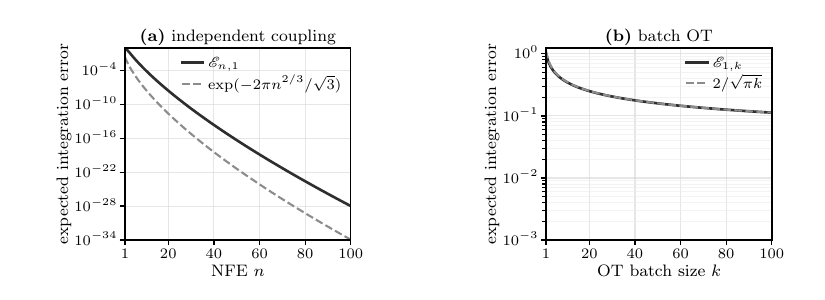}
\caption{Numerical verification of the two asymptotic regimes in
\cref{prop:binary-integration-asymptotics}.}
\label{fig:binary-asymptotics-numerics}
\end{figure}

\section{Experimental details}
\label{app:experimental-details}

This appendix gives the implementation details needed to reproduce the experiments.

\subsection{Small-scale experiments}

These experiments were run on CPU on an Apple M2 MacBook
Air with 8 CPU cores and 16\,GB RAM, using macOS~26.1, Python~3.9.6, and PyTorch~2.4.1.

\subsubsection[Flow map induced by the expected batch OT flow]{Flow map induced by $\phi^{\overline{\pi}_{k}}$ (\cref{fig:expected-batch-ot-flow-cells})}\label{app:velocity-compute}
For a given $x\in\mathbb{R}^{2}$, we approximate $\phi^{\overline{\pi}_{k}}(x)$
by integrating the flow ODE with $50$ uniform second-order Runge--Kutta steps.
When $k=1$, we use the closed-form expression for $\mu\otimes\nu$ in \cref{lem:gaussian-discrete-velocity-posterior}.
When $k>1$, we approximate $u_{t}^{\overline{\pi}_{k}}$
by Monte Carlo over $R=2{,}000$ common random batches.

\subsubsection[Posterior concentration vs. batch size and dimension]{Posterior concentration vs. $k$ and $d$ (\cref{fig:posterior-concentration})}
We consider $M=100$ target atoms in $\nu$. When varying the OT batch size $k$, we keep one atom cloud fixed in dimension
$d=20$. When varying the dimension, we draw a new atom cloud uniformly in $[-1,\,1]^{d}$ for each
value of $d$. Each curve
averages $2{,}000$ explicit Euler trajectories with $100$ function evaluations.
We approximate $u_{t}^{\overline{\pi}_{k}}$ as in \cref{app:velocity-compute}.
The largest standard error over sampled trajectories is below \(0.006\)
on the plotted \(0\)--\(1\) scale, so standard-error bands are not drawn.

\subsubsection[Two-atom case]{Two-atom case (\cref{fig:binary-error-contours})}
In the two-atom case, we do not estimate the expected batch OT velocity by
Monte Carlo. We use the closed-form expression for the minibatch OT assignment
probability $q_{k}$ from \cref{lem:binary-assignment-probability}, insert it into the closed-form
conditional mean in \cref{lem:binary-conditional-mean}, and evaluate
$\mathscr{E}_{n,\,k}$ using the one-dimensional integral representation in
\cref{lem:binary-euler-structure}. Thus the computation is closed form up to a one-dimensional
quadrature, which enables the precise curves reported in
\cref{fig:binary-error-contours}.

\subsection{Larger-scale experiments}

These experiments were run on several machines running AlmaLinux~9.7, each equipped with
an Intel Core i7-14700K CPU, 62\,GiB RAM, and a single NVIDIA RTX~4000 Ada Generation GPU with 20\,GB
VRAM, using Python~3.9.25 and PyTorch~2.5.0~\citep{torchvision} with CUDA~12.4 and cuDNN~9.1.
Their wall-clock time is reported in~\cref{tab:experiment-wall-time}.

\begin{table}[ht]
\caption{Total wall-clock time for the experiments, including exploratory runs.}
\label{tab:experiment-wall-time}
\centering
\begin{tabular}{@{}llr@{}}
\toprule
\textbf{Task} & \textbf{Dataset} & \textbf{Wall-clock time} \\
\midrule
Training & CIFAR-10 & 4{,}884.3\,h \\
         & SVHN     & 2{,}151.0\,h \\
\addlinespace[2pt]
Inference \& FID & CIFAR-10 & 957.8\,h \\
                 & SVHN     & 248.9\,h \\
\addlinespace[2pt]
Estimation of $\mathscr{E}_{n,\,k}$ & Synthetic & $\approx$ 540\,h \\
\midrule
\multicolumn{2}{@{}r}{Total} & $\approx$ 8{,}780\,h \\
\bottomrule
\end{tabular}
\end{table}

\subsubsection[Rates on expected batch OT cost and plan]{Rates on expected batch OT cost and plan (\cref{fig:expected-batch-ot-cost-plan-rates})}
To avoid using stochastic semidiscrete solvers to compute $W_{2}^{2}(\mu,\,\nu)$, we follow the idea of
\citet[Section~4]{hundrieser}: we first discretize $\mu$ by a large empirical measure,
and then choose the target weights so that the nearest-neighbor projection is optimal for this
empirical reference problem. Specifically, fix $N_{\mathrm{ref}}=2\times10^6$, draw
$X_{1},\ldots,X_{N_{\mathrm{ref}}}\sim\mu$ independently, and let
$i^{\star}(x)\in\argmin_{1\leqslant i\leqslant M}\lVert x-v_{i}\rVert^{2}$. We set
\[
\mu_{\mathrm{ref}}=\frac{1}{N_{\mathrm{ref}}}\sum_{\ell=1}^{N_{\mathrm{ref}}}\updelta_{X_{\ell}},
\qquad
\widehat{a}_{i}=\frac{1}{N_{\mathrm{ref}}}\sum_{\ell=1}^{N_{\mathrm{ref}}}\mathbbm{1}_{\{i^{\star}(X_{\ell})=i\}},
\qquad
\nu=\sum_{i=1}^{M}\widehat{a}_{i}\updelta_{v_{i}}.
\]
By construction, the map $T^{\star}:X\mapsto v_{i^{\star}(X)}$ pushes
$\mu_{\mathrm{ref}}$ forward to $\nu$. Since each $v_{i^{\star}(X_{\ell})}$
is a nearest neighbor of $X_{\ell}$, this map is optimal, and if $N_{\mathrm{ref}}$ is large enough,
\[
W_{2}^{2}(\mu,\,\nu)
\approx
W_{2}^{2}(\mu_{\mathrm{ref}},\,\nu)
=
\frac{1}{N_{\mathrm{ref}}}\sum_{\ell=1}^{N_{\mathrm{ref}}}\lVert X_{\ell}-v_{i^{\star}(X_{\ell})}\rVert^{2}.
\]
For each Monte Carlo repetition $r\in\{1,\ldots,2000\}$ and each
$k\in\{2, \lfloor2^{1.25}\rfloor, \lfloor2^{1.5}\rfloor, \ldots, 2^{15}\}$,
we draw batches $(\mathbf{x}_{k}^{(r)},\,\mathbf{y}_{k}^{(r)})\sim\mu^{\otimes k}\otimes\nu^{\otimes k}$. We then
compute the OT plan $\widehat{\pi}_{k}^{(r)}$ between $\widehat{\mu}_{k}^{(r)}$ and $\widehat{\nu}_{k}^{(r)}$, and store
$W_{2}^{2}(\widehat{\mu}_{k}^{(r)},\,\widehat{\nu}_{k}^{(r)})$ and $\mathbb{E}_{(X,\,Y)\sim\widehat{\pi}_{k}^{(r)}}[\lVert Y - T^{\star}(X) \rVert^{2}]$.

\subsubsection[Multi-atom synthetic case]{Multi-atom synthetic case (\cref{fig:nfe-vs-batch-size})}
Atoms are sampled once uniformly in $[-1,\,1]^{20}$ and then kept fixed.
For each pair $(n,\,k)$, we consider the target velocity field $u_{t}^{\overline{\pi}_{k}}$, starting
from $10^4$ independent samples $X\sim\mathcal{N}(0,\,I_{20})$. We
compare a coarse explicit Euler integration with $n$ function evaluations
to a second-order Runge--Kutta reference integration with $200$ function evaluations. We report
the empirical mean error
\[
\widehat{\mathscr{E}}_{n,\,k}
=\frac{1}{10^4}\sum_{r=1}^{10^4}
\lVert f_{n}^{\overline{\pi}_{k}}(X_{r})-\widetilde{f}_{200}^{\overline{\pi}_{k}}(X_{r})\rVert,
\]
where $\widetilde{f}_{200}^{\overline{\pi}_{k}}$ denotes the Runge--Kutta reference map.
We approximate $u_{t}^{\overline{\pi}_{k}}$ as in \cref{app:velocity-compute}.
The error bars in \cref{fig:nfe-vs-batch-size} are standard errors over the sampled initial
conditions.

\subsubsection[Expected batch OT cost on image datasets]{Expected batch OT cost on image datasets (\cref{fig:image-cost-fid-vs-batch-size})}
For each \(k\in\{1,\,2,\,4,\,8,\ldots,8192\}\), we estimate the expected batch OT cost by averaging over
\(B_k=\lceil\frac{2^{20}}{k}\rceil\)
independent OT problems of size \(k\). Thus, for every \(k\), the total number of sampled pairs is
kept close to \(2^{20}\). We report the empirical mean of the empirical OT costs over
these \(B_k\) batches, where each batch matches \(k\) i.i.d. Gaussian samples in \(\mathbb{R}^{3072}\)
to \(k\) samples drawn uniformly from the training set.
In the left panel of \cref{fig:image-cost-fid-vs-batch-size} standard errors over
these \(B_k\) independent OT batches are below plotting resolution.

\subsubsection[Image generation experiments]{Image generation experiments (\cref{fig:image-cost-fid-vs-batch-size})}  \label{sec:comp_tradeoff}

The flow model is the U-Net from the TorchCFM library \citep{tong} 
and the FID is computed using \texttt{clean-fid} \citep{cleanfid}. 
See \cref{tab:flow-model-details} for more details.

We measure the computational tradeoff between increasing OT batch
size $k$ during training and using more function evaluations at inference.
On CIFAR-10, we measure the wall-clock time of one training
step as a function of the OT batch size $k$, and the wall-clock time
of one U-Net evaluation during sampling. For $k=1$, we use the independent
coupling, obtained by pairing a Gaussian noise batch with a CIFAR-10
data batch at random, without solving OT. For $k>1$, we construct
pairs by solving exact OT between empirical batches of size $k$,
using POT \citep{pot}. Each timed training step includes pair construction,
linear interpolation, U-Net forward pass, loss computation, backward
pass, gradient clipping, Adam update, scheduler step, and EMA update.

Up to $k=1024$, training time for CIFAR-10 is nearly
unchanged and increases by about $3\%$ relative to the independent coupling. At $k=8192$, training
becomes $55\%$ more expensive. One U-Net function evaluation
takes $0.861 \pm 0.0002$ ms per image. Therefore, on our hardware, for a $400{,}000$-step run,
if $N$ function evaluations are saved at inference, the extra training
cost is amortized after roughly $\frac{45.5}{N}$ million generated
images.

Final-checkpoint FID values for both image datasets are reported in \cref{tab:image-final-fid}.
Larger OT batch sizes are associated with lower training loss and lower FID during training (\cref{fig:cifar-training-curves}).
We display a rolling mean with window $w=100$ for clarity,
which does not change the ordering of the curves.

\subsection{Software licenses}
The experiments were run using \texttt{POT}~(\cite{pot}, MIT license),
\texttt{torchcfm}~1.0.7 (\cite{tong}, MIT), \texttt{torchdiffeq}~0.2.5 (\cite{torchdiffeq}, MIT),
\texttt{torchdyn}~1.0.6 (\cite{torchdyn}, Apache-2.0), \texttt{torchvision}~0.20.0 (\cite{torchvision-software}, BSD-3-Clause),
\texttt{numpy}~2.0.2 (\cite{numpy}, BSD-style), \texttt{scipy}~1.13.1 (\cite{scipy}, BSD-style),
and \texttt{clean-fid}~0.1.35 (\cite{cleanfid}, MIT).

\begin{table}[p]
\centering
\caption{
CIFAR-10 wall-clock training cost as a function of OT batch size.
Costs assume $400{,}000$ training steps on one GPU. Values are mean $\pm$ standard error.
}
\small
\begin{tabular}{rrrr}
\toprule
$k$ & Step time (ms) & Training cost (GPU-h) & Extra cost vs. $k=1$ (GPU-h) \\
\midrule
1    & $180.79 \pm 0.03$ & $20.09 \pm 0.00$ & $0.00$ \\
2    & $186.33 \pm 0.04$ & $20.70 \pm 0.00$ & $0.62 \pm 0.01$ \\
4    & $184.07 \pm 0.01$ & $20.45 \pm 0.00$ & $0.36 \pm 0.00$ \\
8    & $182.83 \pm 0.15$ & $20.31 \pm 0.02$ & $0.23 \pm 0.02$ \\
16   & $182.37 \pm 0.14$ & $20.26 \pm 0.02$ & $0.18 \pm 0.02$ \\
32   & $182.24 \pm 0.03$ & $20.25 \pm 0.00$ & $0.16 \pm 0.00$ \\
64   & $181.94 \pm 0.03$ & $20.22 \pm 0.00$ & $0.13 \pm 0.00$ \\
128  & $181.81 \pm 0.01$ & $20.20 \pm 0.00$ & $0.11 \pm 0.00$ \\
256  & $182.09 \pm 0.02$ & $20.23 \pm 0.00$ & $0.14 \pm 0.00$ \\
512  & $183.06 \pm 0.02$ & $20.34 \pm 0.00$ & $0.25 \pm 0.00$ \\
1024 & $185.80 \pm 0.02$ & $20.64 \pm 0.00$ & $0.56 \pm 0.00$ \\
2048 & $195.89 \pm 0.18$ & $21.77 \pm 0.02$ & $1.68 \pm 0.02$ \\
4096 & $221.50 \pm 0.25$ & $24.61 \pm 0.03$ & $4.52 \pm 0.03$ \\
8192 & $278.81 \pm 0.17$ & $30.98 \pm 0.02$ & $10.89 \pm 0.02$ \\
\bottomrule
\end{tabular}
\label{tab:training-inference-wall-time}
\end{table}

\begin{table}[p]
\caption{Flow model training and FID evaluation protocol.}
\label{tab:flow-model-details}
\centering
\begin{tabularx}{0.9\linewidth}{@{}c l Y@{}}
\toprule
& \textbf{Setting} & \textbf{Value} \\
\midrule
\multirow{16}{*}{\rotatebox[origin=c]{90}{\textsc{Training}}}
& Architecture & U-Net from \texttt{torchcfm} \citep{tong} \\
& Input size & $3\times 32\times 32$ \\
& Base channels & 128 \\
& Residual blocks per level & 2 \\
& Channel multipliers & $[1,2,2,2]$ \\
& Attention heads & 4 \\
& Head channels & 64 \\
& Attention resolution & 16 \\
& Dropout & 0.1 \\
& OT batch size & $k \in \{1,\,2,\,4,\,8,\ldots,8192\}$ \\
& Optimization minibatch size & 64 \\
& Optimizer & Adam, learning rate $2\times 10^{-4}$ \\
& Learning-rate schedule & Linear warmup for 5{,}000 steps, then constant \\
& Gradient clipping & Global gradient norm clipped to 1.0 \\
& Exponential moving average & Decay 0.9999 \\
& Other fixed hyperparameters & $\sigma = 0$, total training steps $= 400{,}000$ \\
\midrule
\multirow{6}{*}{\rotatebox[origin=c]{90}{\textsc{FID eval.}}}
& Implementation & \texttt{clean-fid} \citep{cleanfid} \\
& Sampler & Forward Euler \\
& Discretization & NFE for CIFAR-10: \(\{10,20,50,100\}\); for SVHN: \(\{5,10,50,100\}\) \\
& Number of generated samples & 50{,}000 \\
& FID batch size & 1024 \\
& Image conversion & Clamp to $[-1,1]$ and convert to $[0,255]$ \\
\bottomrule
\end{tabularx}
\end{table}

\begin{table}[p]
\centering
\caption{FID at the final checkpoint. Values are mean $\pm$ std. dev. over six training seeds.}
\label{tab:image-final-fid}
\small
\setlength{\tabcolsep}{0pt}
\renewcommand{\arraystretch}{0.86}
\begin{tabular*}{0.96\linewidth}{@{\extracolsep{\fill}}crcccc@{}}
\toprule
& $k$ & NFE 10 & NFE 20 & NFE 50 & NFE 100 \\
\midrule
\multirow{14}{*}{\rotatebox[origin=c]{90}{\textsc{CIFAR-10}}}
& 1    & $14.69 \pm 0.27$ & $9.05 \pm 0.14$ & $6.28 \pm 0.10$ & $5.34 \pm 0.10$ \\
& 2    & $14.17 \pm 0.24$ & $9.06 \pm 0.12$ & $6.40 \pm 0.09$ & $5.46 \pm 0.08$ \\
& 4    & $13.89 \pm 0.24$ & $9.14 \pm 0.19$ & $6.50 \pm 0.15$ & $5.56 \pm 0.13$ \\
& 8    & $13.57 \pm 0.11$ & $9.14 \pm 0.08$ & $6.52 \pm 0.07$ & $5.59 \pm 0.07$ \\
& 16   & $13.39 \pm 0.12$ & $9.14 \pm 0.13$ & $6.57 \pm 0.11$ & $5.63 \pm 0.10$ \\
& 32   & $13.42 \pm 0.22$ & $9.24 \pm 0.14$ & $6.64 \pm 0.07$ & $5.68 \pm 0.05$ \\
& 64   & $13.14 \pm 0.14$ & $9.11 \pm 0.11$ & $6.53 \pm 0.10$ & $5.58 \pm 0.09$ \\
& 128  & $13.18 \pm 0.16$ & $9.20 \pm 0.14$ & $6.59 \pm 0.10$ & $5.63 \pm 0.08$ \\
& 256  & $12.92 \pm 0.16$ & $9.09 \pm 0.17$ & $6.53 \pm 0.16$ & $5.57 \pm 0.14$ \\
& 512  & $12.76 \pm 0.23$ & $9.03 \pm 0.19$ & $6.51 \pm 0.12$ & $5.54 \pm 0.09$ \\
& 1024 & $12.67 \pm 0.20$ & $9.04 \pm 0.21$ & $6.50 \pm 0.17$ & $5.50 \pm 0.14$ \\
& 2048 & $12.41 \pm 0.15$ & $8.95 \pm 0.12$ & $6.43 \pm 0.10$ & $5.46 \pm 0.07$ \\
& 4096 & $12.31 \pm 0.16$ & $8.86 \pm 0.13$ & $6.39 \pm 0.12$ & $5.42 \pm 0.11$ \\
& 8192 & $12.06 \pm 0.12$ & $8.66 \pm 0.11$ & $6.24 \pm 0.09$ & $5.28 \pm 0.08$ \\
\midrule
& $k$ & NFE 5 & NFE 10 & NFE 50 & NFE 100 \\
\midrule
\multirow{14}{*}{\rotatebox[origin=c]{90}{\textsc{SVHN}}}
& 1    & $60.99 \pm 0.89$ & $31.34 \pm 0.54$ & $5.67 \pm 0.14$ & $3.42 \pm 0.12$ \\
& 2    & $59.89 \pm 0.44$ & $31.49 \pm 0.31$ & $5.76 \pm 0.05$ & $3.43 \pm 0.03$ \\
& 4    & $57.18 \pm 0.73$ & $30.92 \pm 0.57$ & $5.69 \pm 0.21$ & $3.40 \pm 0.12$ \\
& 8    & $55.59 \pm 0.59$ & $30.34 \pm 0.33$ & $5.65 \pm 0.12$ & $3.40 \pm 0.09$ \\
& 16   & $54.08 \pm 0.44$ & $29.99 \pm 0.33$ & $5.73 \pm 0.13$ & $3.47 \pm 0.06$ \\
& 32   & $53.32 \pm 0.89$ & $29.91 \pm 0.78$ & $5.77 \pm 0.28$ & $3.48 \pm 0.17$ \\
& 64   & $52.15 \pm 0.74$ & $29.61 \pm 0.50$ & $5.80 \pm 0.17$ & $3.50 \pm 0.08$ \\
& 128  & $51.83 \pm 0.92$ & $29.98 \pm 0.78$ & $6.02 \pm 0.21$ & $3.68 \pm 0.12$ \\
& 256  & $50.93 \pm 0.59$ & $29.47 \pm 0.43$ & $5.99 \pm 0.18$ & $3.69 \pm 0.13$ \\
& 512  & $49.83 \pm 0.77$ & $29.16 \pm 0.66$ & $6.15 \pm 0.24$ & $3.82 \pm 0.16$ \\
& 1024 & $49.82 \pm 0.40$ & $29.64 \pm 0.35$ & $6.52 \pm 0.14$ & $4.09 \pm 0.10$ \\
& 2048 & $49.54 \pm 0.64$ & $29.84 \pm 0.69$ & $6.80 \pm 0.23$ & $4.39 \pm 0.14$ \\
& 4096 & $49.81 \pm 0.46$ & $30.05 \pm 0.55$ & $7.13 \pm 0.23$ & $4.68 \pm 0.18$ \\
& 8192 & $50.10 \pm 0.66$ & $30.19 \pm 0.60$ & $7.06 \pm 0.33$ & $4.59 \pm 0.21$ \\
\bottomrule
\end{tabular*}
\end{table}

\begin{figure}[p]
\centering
\includegraphics[width=0.86\linewidth]{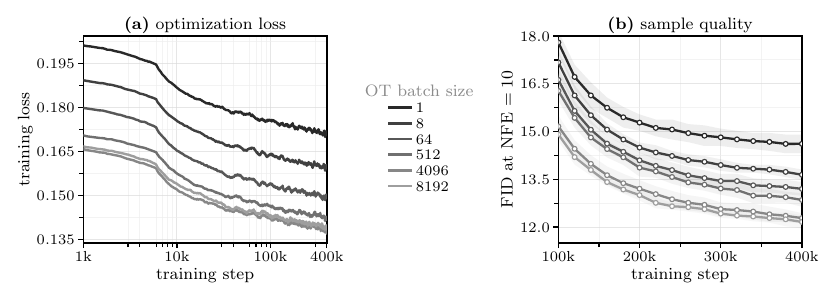}
\caption{Training curves on CIFAR-10 for a representative subset of OT batch
sizes; training loss and FID at fixed $\mathrm{NFE}=10$ versus optimization
step. Larger values of $k$ achieve lower loss and lower FID during training.
Loss curves show rolling means over available runs; FID bands show $\pm1$
empirical standard deviation over seeds.}
\label{fig:cifar-training-curves}
\end{figure}

\begin{figure}[p]
\centering
\includegraphics[width=\textwidth,height=0.92\textheight,keepaspectratio]{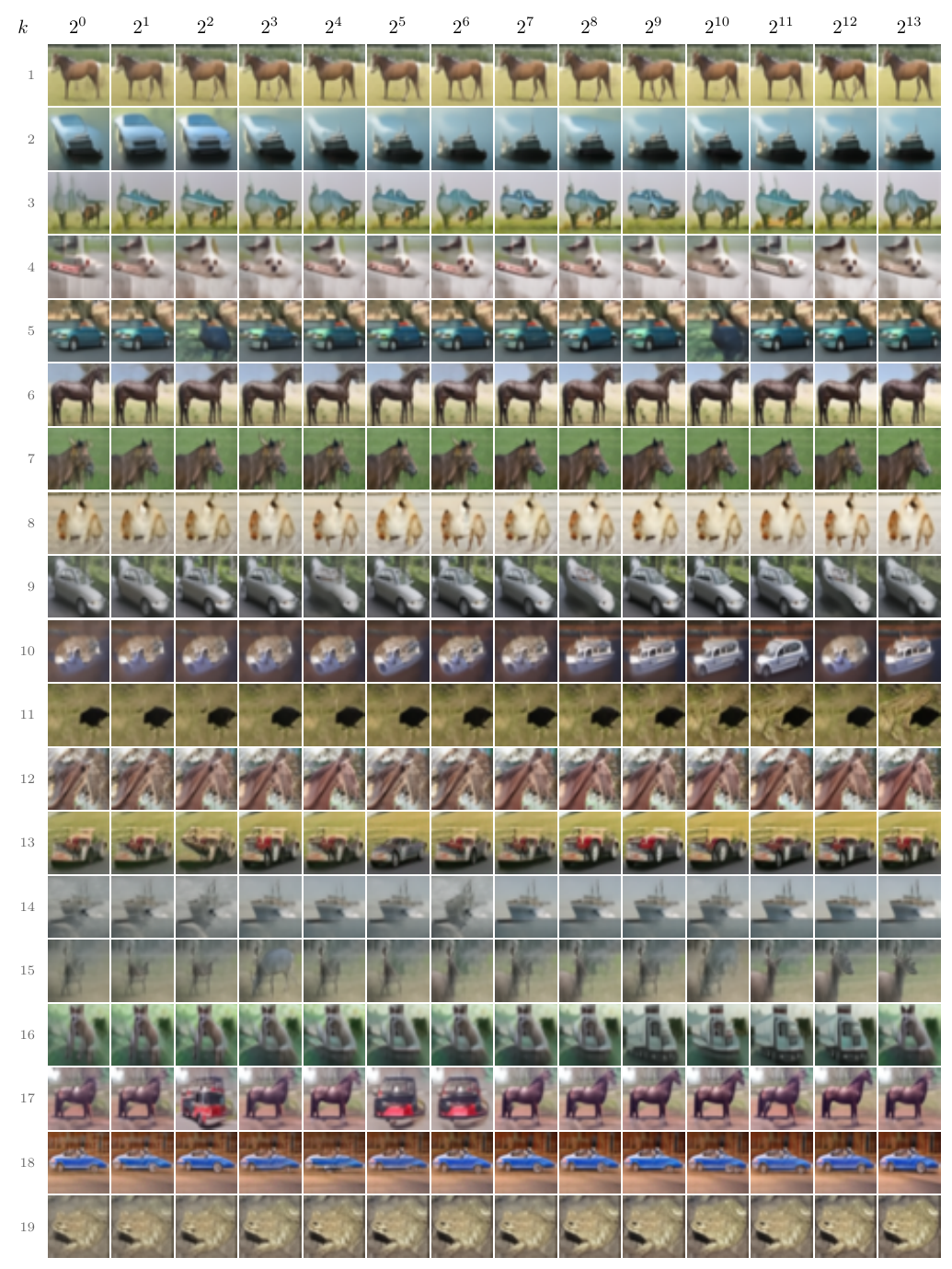}
\caption{Generated CIFAR-10 samples for $k\in\{1,2,4,\ldots,8192\}$ at $\mathrm{NFE}=10$.}\label{fig:cifar-nfe10}
\end{figure}

\begin{figure}[p]
\centering
\includegraphics[width=\textwidth,height=0.92\textheight,keepaspectratio]{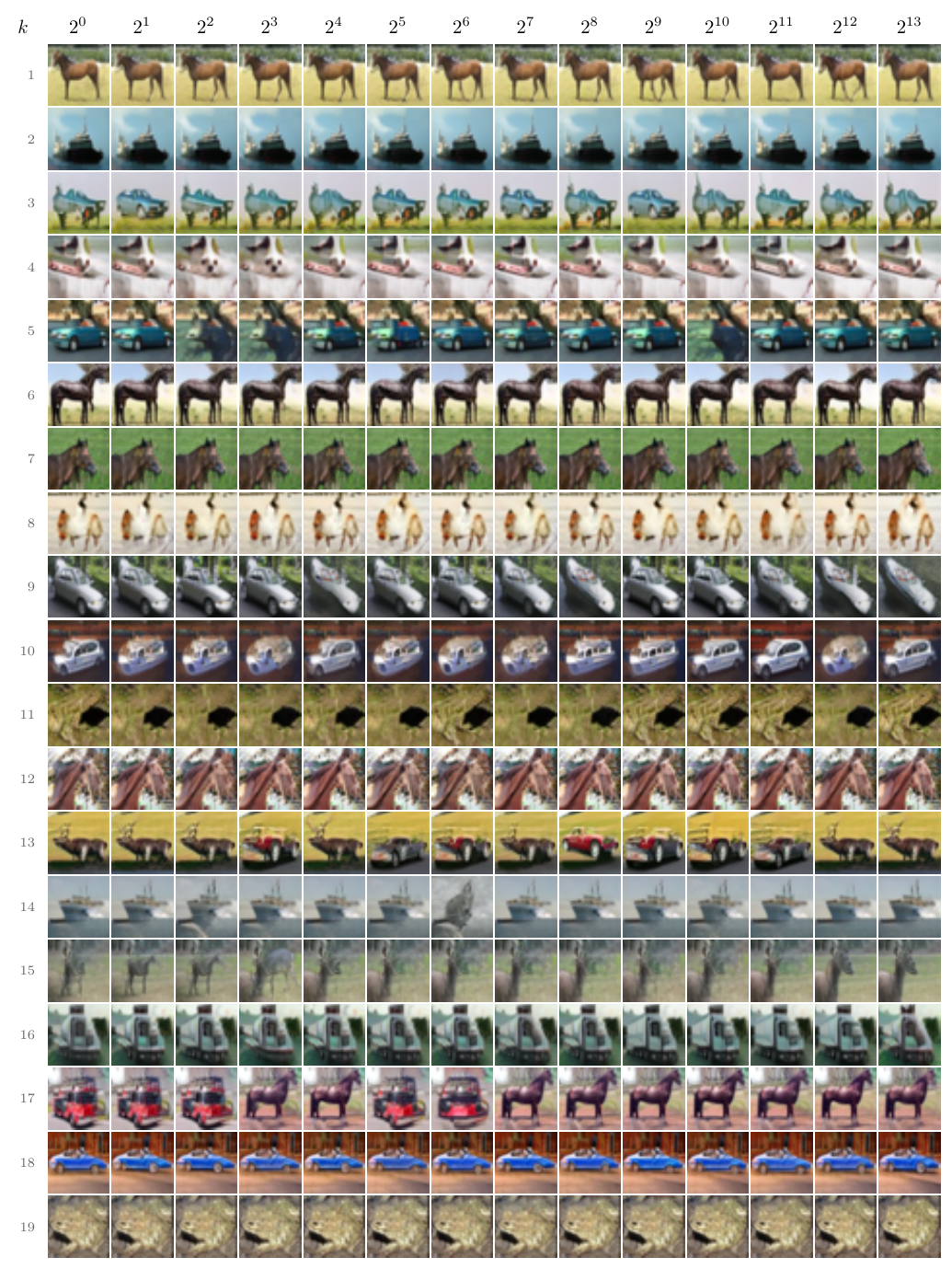}
\caption{Generated CIFAR-10 samples for $k\in\{1,2,4,\ldots,8192\}$ at $\mathrm{NFE}=100$.}\label{fig:cifar-nfe100}
\end{figure}

\begin{figure}[p]
\centering
\includegraphics[width=\textwidth,height=0.92\textheight,keepaspectratio]{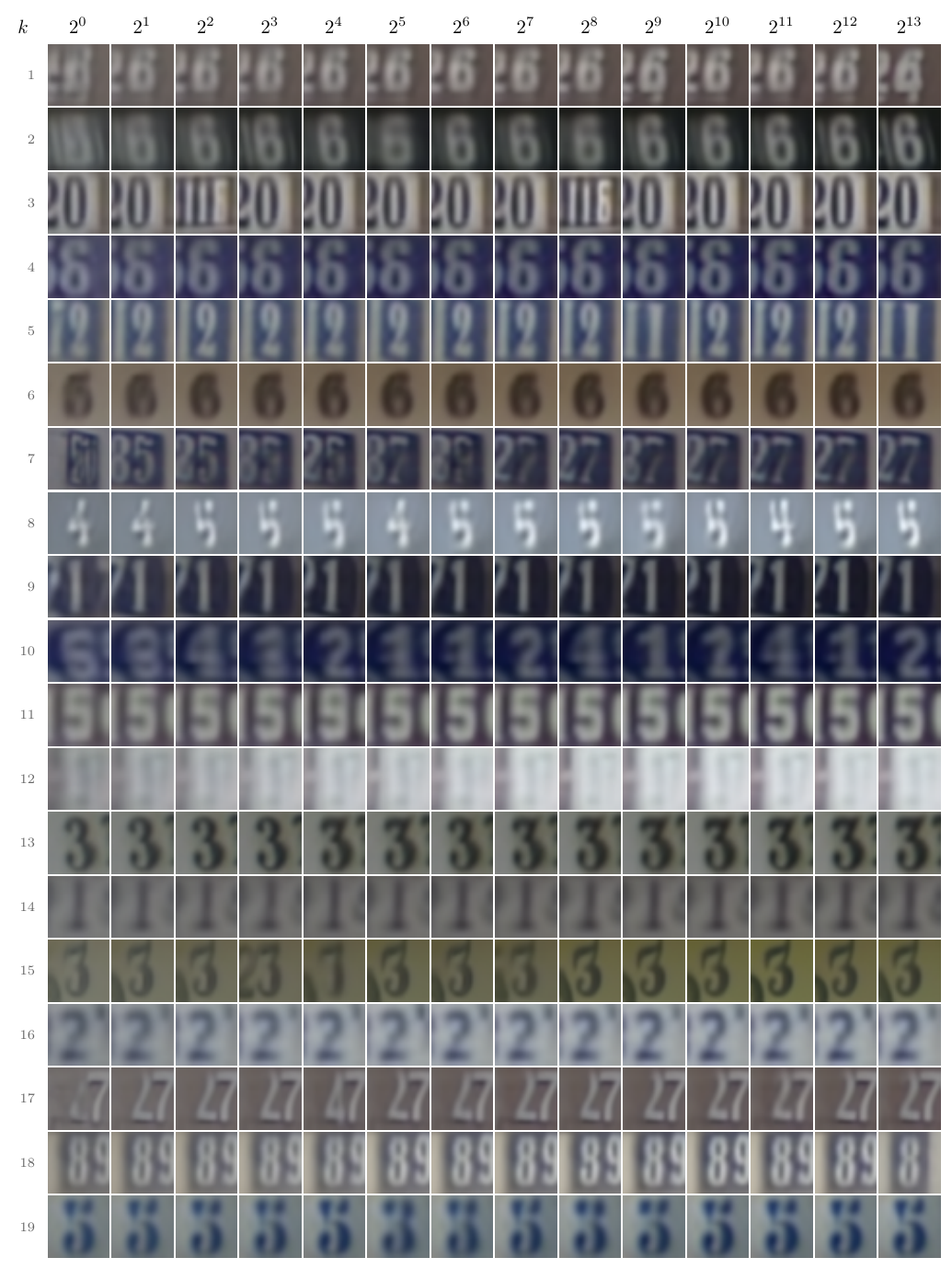}
\caption{Generated SVHN samples for $k\in\{1,2,4,\ldots,8192\}$ at $\mathrm{NFE}=5$.}\label{fig:svhn-nfe5}
\end{figure}

\begin{figure}[p]
\centering
\includegraphics[width=\textwidth,height=0.92\textheight,keepaspectratio]{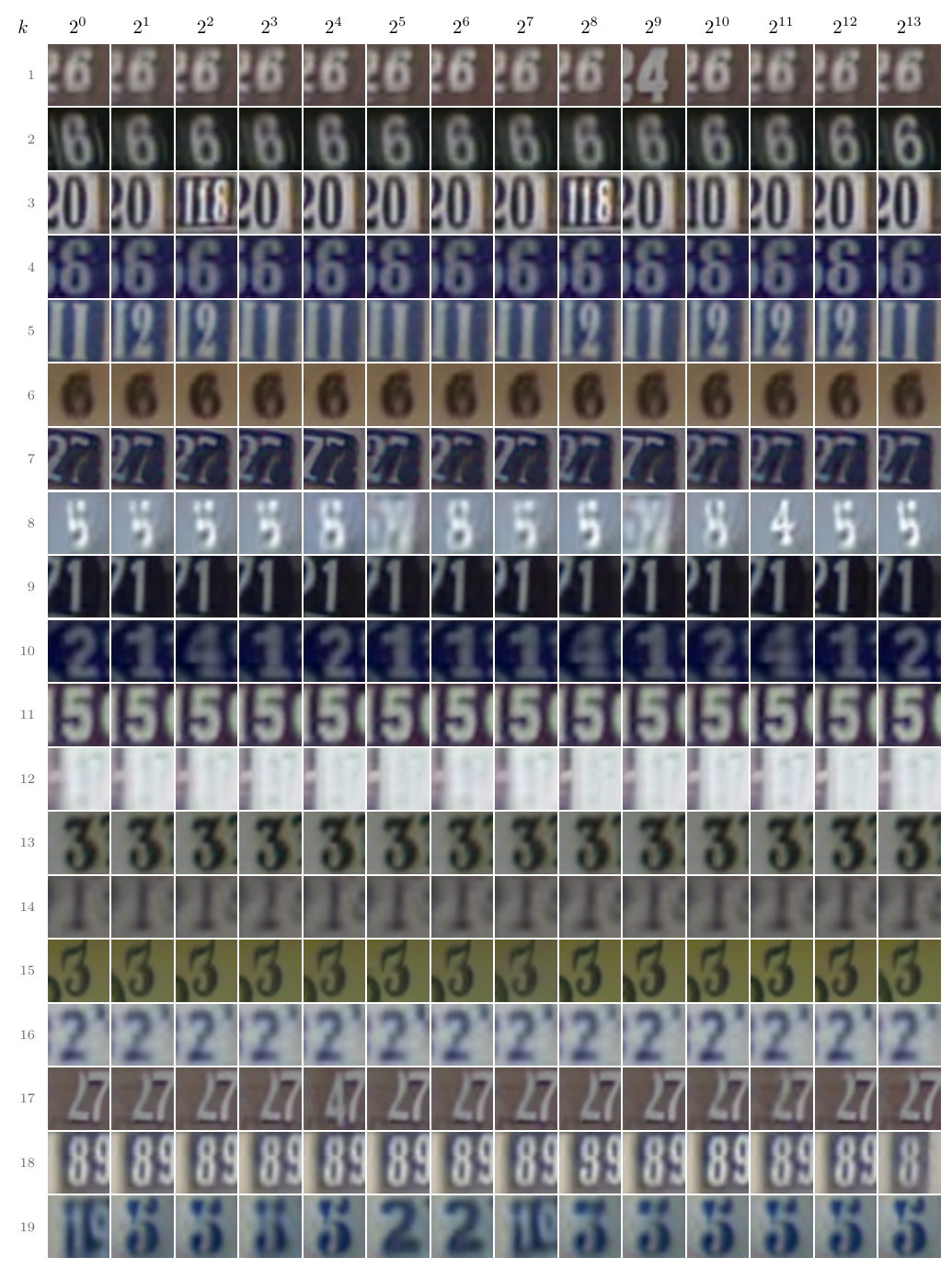}
\caption{Generated SVHN samples for $k\in\{1,2,4,\ldots,8192\}$ at $\mathrm{NFE}=100$.}\label{fig:svhn-nfe100}
\end{figure}

\section{Broader impact}\label{app:broader-impact}

This work is primarily theoretical: it studies the coupling induced by repeated minibatch optimal
transport and its consequences for flow matching. It may help make generative modeling methods
easier to analyze and optimize, but any resulting gains are likely to be realized unevenly in a
field shaped by unequal access to compute and deployment infrastructure.
Improvements in generative modeling can also be used in systems for spam and deceptive content,
although this paper does not release a high-capability generative model.

\clearpage
\newpage
\section*{NeurIPS Paper Checklist}

\begin{enumerate}

\item {\bf Claims}
    \item[] Question: Do the main claims made in the abstract and introduction accurately reflect the paper's contributions and scope?
    \item[] Answer: \answerYes{}.
    \item[] Justification: The claims match the contributions and are supported by \cref{sec:expected-batch-ot-plan,sec:flow-matching-expected-batch-ot} and the experiments in \cref{fig:expected-batch-ot-cost-plan-rates,fig:nfe-vs-batch-size,fig:image-cost-fid-vs-batch-size}.

\item {\bf Limitations}
    \item[] Question: Does the paper discuss the limitations of the work performed by the authors?
    \item[] Answer: \answerYes{}.
    \item[] Justification: The ``Limitations and future directions'' paragraph in the conclusion discusses the semidiscrete scope of the quantitative results, the restriction of the theory to exact velocities, and the lack of general bounds for \(\mathscr{E}_{n,k}\).

\item {\bf Theory assumptions and proofs}
    \item[] Question: For each theoretical result, does the paper provide the full set of assumptions and a complete (and correct) proof?
    \item[] Answer: \answerYes{}.
    \item[] Justification: Assumptions are stated with the theoretical results, and complete proofs are provided in \cref{app:expected-batch-ot-proofs,app:flow-matching-proofs}.

    \item {\bf Experimental result reproducibility}
    \item[] Question: Does the paper fully disclose all the information needed to reproduce the main experimental results of the paper to the extent that it affects the main claims and/or conclusions of the paper (regardless of whether the code and data are provided or not)?
    \item[] Answer: \answerYes{}.
    \item[] Justification: Experimental protocols, hyperparameters, evaluation procedures, software versions, and compute resources are described in \cref{app:experimental-details}. The introduction links to the code repository, and the accompanying README gives reproduction instructions.

\item {\bf Open access to data and code}
    \item[] Question: Does the paper provide open access to the data and code, with sufficient instructions to faithfully reproduce the main experimental results, as described in supplemental material?
    \item[] Answer: \answerYes{}.
    \item[] Justification: The introduction links to the code repository: the accompanying code includes dependencies, reproduction scripts, and saved results. The experiments use public CIFAR-10 and SVHN datasets, which the training scripts download.

\item {\bf Experimental setting/details}
    \item[] Question: Does the paper specify all the training and test details (e.g., data splits, hyperparameters, how they were chosen, type of optimizer) necessary to understand the results?
    \item[] Answer: \answerYes{}.
    \item[] Justification: Training and evaluation details are specified in \cref{app:experimental-details} and the accompanying code, including preprocessing, architecture, optimizer, OT batch sizes, NFE values, and FID protocol.

\item {\bf Experiment statistical significance}
    \item[] Question: Does the paper report error bars suitably and correctly defined or other appropriate information about the statistical significance of the experiments?
    \item[] Answer: \answerYes{}.
    \item[] Justification: Error bars and bands are defined in the figure captions and in \cref{app:experimental-details}. We report standard errors for Monte Carlo estimates and standard deviations over training seeds for FID.

\item {\bf Experiments compute resources}
    \item[] Question: For each experiment, does the paper provide sufficient information on the computer resources (type of compute workers, memory, time of execution) needed to reproduce the experiments?
    \item[] Answer: \answerYes{}.
    \item[] Justification: Compute resources, hardware, software versions, wall-clock times, and exploratory compute are reported in \cref{app:experimental-details}.
    
\item {\bf Code of ethics}
    \item[] Question: Does the research conducted in the paper conform, in every respect, with the NeurIPS Code of Ethics \url{https://neurips.cc/public/EthicsGuidelines}?
    \item[] Answer: \answerYes{}.
    \item[] Justification: The work conforms to the NeurIPS Code of Ethics: it uses public benchmark datasets and does not involve human subjects, crowdsourcing, private data, surveillance data, or security-sensitive systems.

\item {\bf Broader impacts}
    \item[] Question: Does the paper discuss both potential positive societal impacts and negative societal impacts of the work performed?
    \item[] Answer: \answerYes{}.
    \item[] Justification: Broader impacts are discussed in \cref{app:broader-impact}, including possible benefits for generative modeling and risks related to misuse and unequal access to compute.
    
\item {\bf Safeguards}
    \item[] Question: Does the paper describe safeguards that have been put in place for responsible release of data or models that have a high risk for misuse (e.g., pre-trained language models, image generators, or scraped datasets)?
    \item[] Answer: \answerNA{}.
    \item[] Justification: The paper does not release high-risk pretrained models, scraped datasets, private data, or security-sensitive systems.

\item {\bf Licenses for existing assets}
    \item[] Question: Are the creators or original owners of assets (e.g., code, data, models), used in the paper, properly credited and are the license and terms of use explicitly mentioned and properly respected?
    \item[] Answer: \answerYes{}.
    \item[] Justification: Existing datasets and software assets are cited, and versions and licenses for the main software dependencies are listed in \cref{app:experimental-details}.

\item {\bf New assets}
    \item[] Question: Are new assets introduced in the paper well documented and is the documentation provided alongside the assets?
    \item[] Answer: \answerYes{}.
    \item[] Justification: The accompanying code includes a README with installation and reproduction instructions, a dependency specification, and an MIT license.

\item {\bf Crowdsourcing and research with human subjects}
    \item[] Question: For crowdsourcing experiments and research with human subjects, does the paper include the full text of instructions given to participants and screenshots, if applicable, as well as details about compensation (if any)? 
    \item[] Answer: \answerNA{}.
    \item[] Justification: The paper does not involve crowdsourcing or human-subject research.

\item {\bf Institutional review board (IRB) approvals or equivalent for research with human subjects}
    \item[] Question: Does the paper describe potential risks incurred by study participants, whether such risks were disclosed to the subjects, and whether Institutional Review Board (IRB) approvals (or an equivalent approval/review based on the requirements of your country or institution) were obtained?
    \item[] Answer: \answerNA{}.
    \item[] Justification: The paper does not involve crowdsourcing or human-subject research, so IRB approval is not applicable.

\item {\bf Declaration of LLM usage}
    \item[] Question: Does the paper describe the usage of LLMs if it is an important, original, or non-standard component of the core methods in this research? Note that if the LLM is used only for writing, editing, or formatting purposes and does \emph{not} impact the core methodology, scientific rigor, or originality of the research, declaration is not required.
    \item[] Answer: \answerNA{}.
    \item[] Justification: The core method development in this research does not involve LLMs as any important, original, or non-standard components.

\end{enumerate}

\end{document}